\documentclass[english,12pt]{article}
\usepackage[T1]{fontenc}
\usepackage[latin9]{inputenc}
\usepackage{color}
\usepackage{babel}
\usepackage{float}
\usepackage{booktabs}
\usepackage{mathtools}
\usepackage{url}
\usepackage{amsmath}
\usepackage{amsthm}
\usepackage{amssymb}
\usepackage{graphicx}
\usepackage{setspace}
\usepackage[authoryear]{natbib}
\usepackage[pdfusetitle,
 bookmarks=true,bookmarksnumbered=false,bookmarksopen=false,
 breaklinks=true,pdfborder={0 0 0},pdfborderstyle={},backref=false,colorlinks=true]
 {hyperref}
\hypersetup{
 citecolor=blue}

\makeatletter

\providecommand{\tabularnewline}{\\}
\floatstyle{ruled}
\newfloat{algorithm}{tbp}{loa}
\providecommand{\algorithmname}{Algorithm}
\floatname{algorithm}{\protect\algorithmname}

\theoremstyle{definition}
\newtheorem{defn}{\protect\definitionname}
\theoremstyle{plain}
\newtheorem{assumption}{\protect\assumptionname}
\theoremstyle{plain}
\newtheorem{thm}{\protect\theoremname}
\theoremstyle{plain}
\newtheorem{cor}{\protect\corollaryname}
\theoremstyle{remark}
\newtheorem{rem}{\protect\remarkname}
\theoremstyle{plain}
\newtheorem{lem}{\protect\lemmaname}

\usepackage{algorithmic}

\usepackage{titling}
\usepackage[compact]{titlesec}

\date{}

\ifdefined\showcaptionsetup
 \PassOptionsToPackage{caption=false}{subfig}
\fi
\usepackage{subfig}
\makeatother

\providecommand{\assumptionname}{Assumption}
\providecommand{\corollaryname}{Corollary}
\providecommand{\definitionname}{Definition}
\providecommand{\lemmaname}{Lemma}
\providecommand{\remarkname}{Remark}
\providecommand{\theoremname}{Theorem}

\begin{document}
\onehalfspacing
\title{\textbf{Adaptive Learning of the Latent Space of Wasserstein Generative
Adversarial Networks}}
\author{Yixuan Qiu$^{*}$\\
{\normalsize School of Statistics and Management, Shanghai University
of Finance and Economics}\\
{\normalsize Shanghai 200433, P.R. China, }{\normalsize\texttt{qiuyixuan@sufe.edu.cn}}~\\
{\normalsize\texttt{}}~\\
Qingyi Gao$^{*}$\\
{\normalsize Department of Statistics, Purdue University}\\
{\normalsize West Lafayette, IN 47907, U.S.A., }{\normalsize\texttt{gqystat@gmail.com}}~\\
\\
Xiao Wang\\
{\normalsize Department of Statistics, Purdue University}\\
{\normalsize West Lafayette, IN 47907, U.S.A., }{\normalsize\texttt{wangxiao@purdue.edu}}}

\maketitle
\def\thefootnote{*}\footnotetext{These authors contributed equally to this work.}
\def\thefootnote{\arabic{footnote}}
\begin{abstract}
Generative models based on latent variables, such as generative adversarial
networks (GANs) and variational auto-encoders (VAEs), have gained
lots of interests due to their impressive performance in many fields.
However, many data such as natural images usually do not populate
the ambient Euclidean space but instead reside in a lower-dimensional
manifold. Thus an inappropriate choice of the latent dimension fails
to uncover the structure of the data, possibly resulting in mismatch
of latent representations and poor generative qualities. Towards addressing
these problems, we propose a novel framework called the latent Wasserstein
GAN (LWGAN) that fuses the Wasserstein auto-encoder and the Wasserstein
GAN so that the intrinsic dimension of the data manifold can be adaptively
learned by a modified informative latent distribution. We prove that
there exist an encoder network and a generator network in such a way
that the intrinsic dimension of the learned encoding distribution
is equal to the dimension of the data manifold. We theoretically establish
that our estimated intrinsic dimension is a consistent estimate of
the true dimension of the data manifold. Meanwhile, we provide an
upper bound on the generalization error of LWGAN, implying that we
force the synthetic data distribution to be similar to the real data
distribution from a population perspective. Comprehensive empirical
experiments verify our framework and show that LWGAN is able to identify
the correct intrinsic dimension under several scenarios, and simultaneously
generate high-quality synthetic data by sampling from the learned
latent distribution.
\end{abstract}
\noindent \textit{Keywords}: consistency, generalization error, generative
adversarial networks, latent variable models, manifold learning, minimax
optimization, Wasserstein distance

\setstretch{1.8}
\addtolength{\abovedisplayskip}{-3.5pt}
\addtolength{\belowdisplayskip}{-3.5pt}

\section{Introduction}

Unsupervised generative models receive great attentions in the machine
learning community nowadays due to their impressive performance in
many fields \citep{kingma2014auto,goodfellow2014generative,li2015generative,dinh2016density,gao2020flow,qiu2021almond}.
Given a random sample from a $p$-dimensional random vector $X\in\mathcal{X}\subset\mathbb{R}^{p}$
with an unknown distribution $P_{X}$, the goal is to train a generative
model that can produce synthetic data that look similar to the observed
samples from $X$. While there are several ways of quantifying the
similarity, the most common approach is to directly employ some of
the known divergence measures, such as the Kullback--Leibler (KL)
divergence and the Wasserstein distance, between the real data distribution
and the synthetic data distribution.

There are two influential frameworks for generative models: generative
adversarial networks (GANs, \citealp{goodfellow2014generative}) and
variational auto-encoders (VAEs, \citealp{kingma2014auto}). They
are latent variable models through a latent variable $Z\in\mathcal{Z}\subset\mathbb{R}^{d}$
drawn from a simple and accessible prior distribution $P_{Z}$, such
as the standard multivariate normal distribution $P_{Z}=N(0,I_{d})$.
Then the synthetic data are generated by either a deterministic transformation
$G:\mathcal{Z}\rightarrow\mathcal{X}$ or a conditional distribution
$p(x|z)$ of $X$ given $Z$.

\paragraph{GAN and WGAN.}

Training GANs is like a two-player game, where two networks, a generator
and a discriminator, are simultaneously trained to allow the powerful
discriminator to distinguish between real data and generated samples.
As a result, the generator is trying to maximize its probability of
having its outputs recognized as real. This leads to the following
minimax objective function,
\begin{equation}
\inf_{G\in\mathcal{G}}\sup_{f\in\mathcal{F}}\mathbb{E}_{X}\left[\log(f(X))\right]+\mathbb{E}_{Z}\left[\log\left(1-f(G(Z))\right)\right],\label{eq:GANminmax}
\end{equation}
where $f\in\mathcal{F}$ is a discriminator and $G\in\mathcal{G}$
is a generator. Optimizing (\ref{eq:GANminmax}) is equivalent to
minimizing the Jensen--Shannon divergence between the generation
distribution and real data distribution. GANs can generate visually
realistic images, but suffer from unstable training and mode collapsing.

The Wasserstein GAN (WGAN, \citealp{arjovsky2017wasserstein}) is
an extension to the vanilla GAN that improves the stability of training
by leveraging the 1-Wasserstein distance between two probability measures.
Denote by $P_{G(Z)}$ the generation distribution measure, and then
the 1-Wasserstein distance between $P_{X}$ and $P_{G(Z)}$ is defined
as 
\begin{equation}
W_{1}(P_{X},P_{G(Z)})=\inf_{\pi\in\Pi(P_{X},P_{Z})}\mathbb{E}_{(X,Z)\sim\pi}\left\Vert X-G(Z)\right\Vert ,\label{eq:wgan_primal}
\end{equation}
where $\Vert\cdot\Vert$ represents the $\ell_{2}$-norm and $\Pi(P_{X},P_{Z})$
is the set of all joint distributions of $(X,Z)$ with marginal measures
$P_{X}$ and $P_{Z}$, respectively. It is hard to find the optimal
coupling $\pi$ through this constrained primal problem. However,
thanks to the Kantorovich--Rubinstein duality \citep{villani2008optimal},
WGAN can learn the generator $G$ by minimizing a dual form of (\ref{eq:wgan_primal}),
\begin{equation}
W_{1}(P_{X},P_{G(Z)})=\sup_{f\in\mathcal{F}}\left\{ \mathbb{E}_{X}f(X)-\mathbb{E}_{Z}f(G(Z))\right\} ,\label{eq:wgan_dual}
\end{equation}
where $f$ is called the critic function, and $\mathcal{F}$ is the
set of all bounded 1-Lipschitz functions. Weight clipping \citep{arjovsky2017wasserstein}
and gradient penalty \citep{gulrajani2017improved} are two common
strategies to maintain the Lipschitz continuity of $f$. Weight clipping
utilizes a tuning parameter $c$ to clamp each weight parameter to
a fixed interval $[-c,c]$ after each gradient update, but this method
is very sensitive to the choice of the parameter $c$. Instead, gradient
penalty adds a regularization term, $\mathbb{E}_{\hat{X}}\left\{ (\Vert\nabla_{x}f(\hat{X})\Vert-1)^{2}\right\} $,
to the loss function to enforce the 1-Lipschitz condition, where $\hat{X}$
is sampled uniformly along the segment between pairs of points sampled
from $P_{X}$ and $P_{G(Z)}$. This is motivated by the fact that
the optimal $f$ has unit gradient norm on the segment between optimally
coupled points from $P_{X}$ and $P_{G(Z)}$.

\paragraph{VAE and WAE.}

A VAE defines a ``probabilistic decoder'' $p_{\theta}(x|z)$ with
the unknown parameter $\theta$. Then the marginal distribution of
$X$ is $p_{\theta}(x)=\int p_{\theta}(x|z)p_{Z}(z)\mathrm{d}z$,
where $p_{Z}(\cdot)$ is the density of $P_{Z}$. Due to the intractability
of this integration, the maximum likelihood estimation is prohibited.
Instead, a ``probabilistic encoder'' $q_{\phi}(z|x)$ with the unknown
parameter $\phi$ is defined to approximate the posterior distribution
$p_{\theta}(z|x)=p_{\theta}(x|z)p_{Z}(z)/p_{\theta}(x)$. The objective
of VAE is to maximize a lower bound of the log-likelihood $\log p_{\theta}(x)$,
which is called the evidence lower bound (ELBO):
\[
\mathrm{ELBO}=\mathbb{E}_{q_{\phi}(z|x)}\left[\log p_{\theta}(x|z)\right]-\mathrm{KL}\left(q_{\phi}(z|x)\Vert p_{Z}(z)\right),
\]
where the first term can be efficiently estimated by the Monte Carlo
sampling, and the second term has a closed-form expression when $q_{\phi}$
is Gaussian. VAEs have strong theoretical justifications and typically
can cover all modes of the data distribution. However, they often
produce blurry images due to the normal approximation of the true
posterior.

The Wasserstein auto-encoder (WAE, \citealp{tolstikhin2018wasserstein})
makes two modifications to VAE. It uses a deterministic encoder $Q:\mathcal{X}\rightarrow\mathcal{Z}$
to approximate the conditional distribution of $Z$ given $X$, and
a deterministic generator $G:\mathcal{Z}\rightarrow\mathcal{X}$ to
approximate the conditional distribution of $X$ given $Z$. In addition,
WAE adopts the 1-Wasserstein distance between the real data distribution
$P_{X}$ and the generation distribution $P_{G(Z)}$, rather than
the KL divergence used in VAEs, to train the model. Let $P_{Q(X)}$
denote the aggregated posterior distribution measure, and then WAE
minimizes the following reconstruction error with respect to the generator
$G$,
\[
\inf_{Q\in\mathcal{Q}}\mathbb{E}_{X}\left\Vert X-G(Q(X))\right\Vert +\lambda\mathcal{D}(P_{Q(X)},P_{Z}),
\]
where $\mathcal{D}$ is any divergence measure between two distributions
$P_{Q(X)}$ and $P_{Z}$, and $\lambda>0$ is a regularization coefficient.
The regularization term forces the aggregated posterior $P_{Q(X)}$
to match the prior distribution $P_{Z}$.

There are several limitations for the generative models above. It
is a requirement for current approaches of training generative models
to pre-specify the dimension of the latent distribution $P_{Z}$ and
treat it as fixed during the training process. For example, the latent
dimensions for VAEs and GANs are pre-specified by users. Another type
of generative model called normalizing flows \citep{dinh2016density}
keeps the latent dimension the same as the dimension of the data.
This is because normalizing flows approximate the data distribution
by a deterministic \emph{invertible} mapping $G$ such that $X=G(Z)$.
Since many observed data such as natural images lie on a low-dimensional
manifold embedded in a higher-dimensional Euclidean space, an inappropriate
choice of the latent dimension could cause a wrong latent representation
that does not populate the full ambient space \citep{rubenstein2018latent}.
Hence, the wrongly specified latent dimension fails to uncover the
structure of the data, and the corresponding generative models may
suffer from mode collapsing, under-fitting, mismatch of representation
learning, and poor generation qualities. Furthermore, although there
are many interesting works taking advantages of both VAEs and GANs
\citep{larsen2016autoencoding,dumoulin2017adversarially,donahue2016adversarial,chenyao21},
it remains unclear what principles are underlying the framework combining
the best of WAEs and WGANs when the latent dimension is unknown.

To handle the aforementioned drawbacks, we propose a novel framework,
called the latent Wasserstein GAN (LWGAN), to identify the intrinsic
dimension of a data distribution that lies on a topological manifold,
and then improve the quality of generative modeling as well as representation
learning. We have performed two major modifications to the current
GAN and VAE frameworks. First, we change the latent distribution from
$N(0,I_{d})$ to a generalized normal distribution $N(0,A)$ with
$A$ being a diagonal matrix with entries taking values 0 or 1. Therefore,
the rank of $A$ allows us to characterize the intrinsic dimension
of the latent space. This modification has been adopted for the flow
model to reduce the dimension of the latent space \citep{zhang2021flow},
but it has not been applied to GAN or VAE models. Second, we combine
WGAN and WAE in a principled way motivated by the primal-dual iterative
algorithm. We utilize a deterministic encoder $Q:\mathcal{X}\rightarrow\mathcal{Z}$
to learn an informative prior distribution $P_{Z}\sim N(0,A)$. On
the other hand, a generator $G:\mathcal{Z}\rightarrow\mathcal{X}$
is combined with $Q$ to generate images that look like the real ones
using the latent code $Z$ from $P_{Z}$. We theoretically guarantee
the existence of such a generator $G$ and an encoder $Q$. To get
rid of possible invalid divergences, we focus on the 1-Wasserstein
distance to measure the similarities between two distributions, which
applies to any pair of distributions as long as they can be sufficiently
sampled. Note that the KL divergence is not well-defined when the
supports of two probability measures do not overlap, which is very
common for high-dimensional data.

The rest of the paper is organized as follows. Section \ref{sec:mismatch}
investigates the phenomenon of dimension mismatch between the latent
distribution and data distribution. Section \ref{sec:lwgan} presents
the new LWGAN framework that provides a feasible way to estimate the
encoder, generator, and intrinsic dimension. Theoretical analyses
are given in Section \ref{sec:theory}, including results on generalization
error bounds, estimation consistency, and intrinsic dimension consistency.
Section \ref{sec:experiment} demonstrates extensive numerical experiments
under different settings to verify that the LWGAN is able to detect
the intrinsic dimensions for both simulated examples and real image
data. Finally, Section \ref{sec:conclusion} concludes this article.
Proofs of theorems and additional numerical results are provided in
the supplementary materials.

\section{Issues of Latent Dimension Mismatch}

\label{sec:mismatch}

Throughout this article we use $\mathcal{X}\subset\mathbb{R}^{p}$
and $\mathcal{Z}\subset\mathbb{R}^{d}$ to denote the spaces of observed
data points and latent variables, respectively. To precisely describe
the structure of high-dimensional data with a low latent dimension,
we first make the following definition of a topological manifold.
\begin{defn}[Topological manifold, \citealp{lee2013introduction}]
 \label{def:manifold}Suppose that $\mathcal{M}$ is a topological
space. $\mathcal{M}$ is a topological manifold of dimension $r$
if $\mathcal{M}$ is a second-countable Hausdorff space, and for each
$x\in\mathcal{M}$, there exist an open subset $U\subset\mathcal{M}$
containing $x$, an open subset $V\subset\mathbb{R}^{r}$, and a homeomorphism
$\varphi$ between $U$ and $V$. A homeomorphism $\varphi:U\rightarrow V$
is a continuous bijective mapping with a continuous inverse $\varphi^{-1}$.
\end{defn}
In this article, all manifolds are referred to as topological manifolds
unless otherwise noted. Typically, $\mathcal{M}$ is a subset of some
Euclidean space $\mathbb{R}^{p}$, in which case the Hausdorff and
second-countability properties in Definition \ref{def:manifold} are
automatically inherited from the Euclidean topology. To exclude overly
complicated cases, we moderately strengthen the qualification of the
homeomorphism $\varphi$ in Definition \ref{def:manifold} to make
it a global one:
\begin{assumption}
\label{assu:homeomorphism}$\mathcal{X}$ is an $r$-dimensional manifold,
and there exists a homeomorphism $\varphi$ between $\mathcal{X}$
and $\mathbb{R}^{r}$.
\end{assumption}
In what follows, the symbol $\varphi$ is used to denote one homeomorphism
between $\mathcal{X}$ and $\mathbb{R}^{r}$. Then we can define a
continuous distribution supported on the manifold $\mathcal{X}$ that
satisfies Assumption \ref{assu:homeomorphism}.
\begin{defn}
\label{def:distr_manifold}A random vector $X\in\mathbb{R}^{p}$ is
said to have a continuous distribution $P_{X}$ supported on $\mathcal{X}$,
if its image $\varphi(X)$ follows a continuous distribution on $\mathbb{R}^{r}$.
\end{defn}
Let $X\in\mathcal{X}\subset\mathbb{R}^{p}$ be the observed data with
a continuous distribution $P_{X}$ supported on $\mathcal{X}$, where
$\mathcal{X}$ satisfies Assumption \ref{assu:homeomorphism}. We
define the intrinsic dimension of the data distribution $P_{X}$ as
the dimension of the manifold $\mathcal{X}$, denoted by $\texttt{InDim}(P_{X})=r$,
and its ambient dimension as the dimension of the enclosing Euclidean
space, denoted by $\texttt{AmDim}(P_{X})=p$. By Theorem 1.2 of \citet{lee2013introduction},
$\texttt{InDim}(P_{X})$ must be unique, and it cannot be larger than
$\texttt{AmDim}(P_{X})$.

In most existing deep generative models, the latent variable $Z$
is selected as a $d$-dimensional standard normal distribution $N(0,I_{d})$,
so $\texttt{InDim}(P_{Z})=\texttt{AmDim}(P_{Z})=d$. The dimension
$d$ is typically predetermined to be a number that is smaller than
$p$. In GAN-based models, if the generator $G$ is a continuous function,
then the synthetic sample $G(Z)$ will be supported on a manifold
of dimension at most $\texttt{InDim}(P_{Z})$. When $\texttt{InDim}(P_{Z})<\texttt{InDim}(P_{X})$,
forcing $P_{G(Z)}$ to be close to $P_{X}$ with unmatched intrinsic
dimensions is a challenging task. On the other hand, in auto-encoder-based
models, similar phenomenon of dimension mismatch occurs for the encoded
distribution $P_{Q(X)}$. For example, it is difficult to enforce
$P_{Q(X)}$ to be close to $P_{Z}$ if $\texttt{InDim}(P_{X})<\texttt{InDim}(P_{Z})$,
as filling a plane with a one-dimensional curve is hard.

To highlight this phenomenon and to motivate our proposed model, we
first employ a toy example to provide intuitions for the effects and
consequences of different intrinsic dimensions of the model and data
distributions. Consider a 3D S-curve dataset as shown in Figure \ref{fig:scurve}(a),
where each data point $X=(X_{1},X_{2},X_{3})$ is generated by
\[
X_{1}=\sin(3\pi(U-0.5)),\quad X_{2}=2V,\quad X_{3}=\text{sign}(3\pi(U-0.5))\cos(3\pi(U-0.5)),
\]
for $U\sim Unif(0,1)$ and $V\sim N(0,1)$. This example results in
$\texttt{AmDim}(P_{X})=3$ and $\texttt{InDim}(P_{X})=2$. We first
choose the latent distribution $P_{Z}$ to be a one-dimensional normal
distribution $N(0,1)$, and then the generated sample from WGAN is
plotted in Figure \ref{fig:scurve}(b). To minimize the 1-Wasserstein
distance between the real distribution $P_{X}$ and the generation
distribution $P_{G(Z)}$, WGAN learns an outer contour of the S-curve,
but it cannot fill points on the surface. Instead, if we choose a
three-dimensional standard normal $N(0,I_{3})$ as the latent distribution,
then WAE is forced to reconstruct the images well, but at the same
time it tries to fill the three-dimensional latent space evenly by
a distribution supported on a two-dimensional manifold. The only way
to do this is by curling the manifold up in the latent space as shown
in Figure \ref{fig:scurve}(d). This disparity between $P_{Z}$ and
$P_{Q(X)}$ in the latent space induces a poor generation of $P_{G(Z)}$
in Figure \ref{fig:scurve}(c).

\begin{figure}[h]
\begin{centering}
\subfloat[S-curve data]{\begin{centering}
\includegraphics[width=0.22\textwidth]{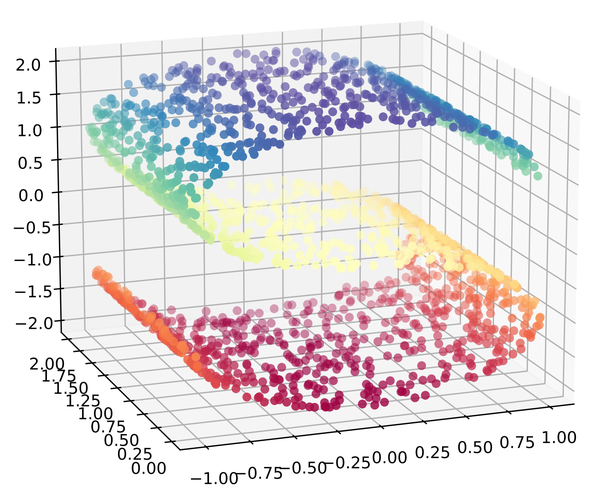}
\par\end{centering}
}\subfloat[WGAN: Generation]{\begin{centering}
\includegraphics[width=0.22\textwidth]{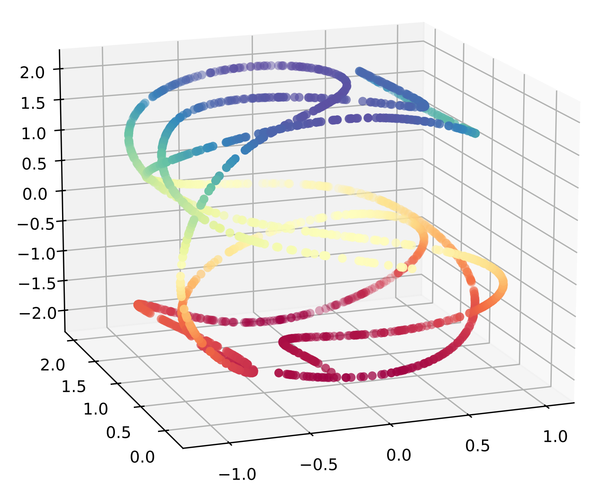}
\par\end{centering}
}\subfloat[WAE: Generation]{\begin{centering}
\includegraphics[width=0.22\textwidth]{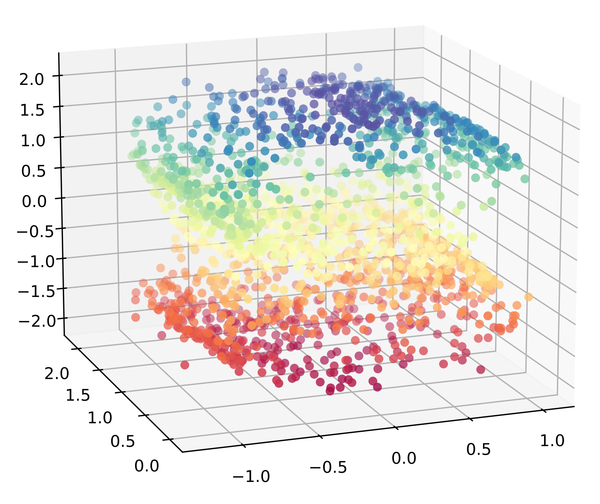}
\par\end{centering}
}\subfloat[WAE: Latent space]{\begin{centering}
\includegraphics[width=0.22\textwidth]{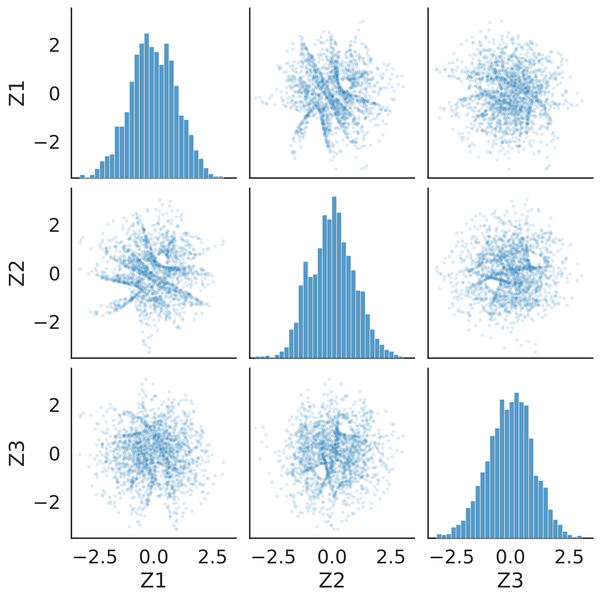}
\par\end{centering}
}
\par\end{centering}
\caption{\protect\label{fig:scurve}Illustrations of data generation with wrong
latent dimensions in WGAN and WAE.}
\end{figure}

\section{The Latent Wasserstein GAN}

\label{sec:lwgan}

A natural solution to the mismatch problem described in Section \ref{sec:mismatch}
is to select a latent distribution $P_{Z}$ whose intrinsic dimension
is the same as that of the data distribution $P_{X}$. However, $\texttt{InDim}(P_{X})$
is typically unknown, so one option is to learn it from the data.
When both the continuous generator $G$ and the continuous encoder
$Q$ are combined in an auto-encoder generative model, $P_{G(Z)}=P_{X}$
and $P_{Q(X)}=P_{Z}$ cannot be satisfied simultaneously unless $\texttt{InDim}(P_{X})=\texttt{InDim}(P_{Z})$
according to our previous discussion. This motivates us to search
for an encoder $Q$ and a corresponding generator $G$, such that
$Q(X)$ reflects the latent space supported on an $r$-dimensional
manifold, and generated samples using the latent variables are of
high quality. To be concrete, we need an auto-encoder generative model
that satisfies the following three goals at the same time: (a) the
latent distribution $P_{Z}$ is supported on an $r$-dimensional manifold;
(b) the distribution of $G(Z)$ is similar to $P_{X}$; (c) the difference
between $X$ and its reconstruction $G(Q(X))$ is small.

\subsection{Existence of optimal encoder-generator pairs}

Unlike conventional generative models that use a fixed standard normal
distribution as the latent distribution, we consider a latent distribution
whose intrinsic dimension could be less than $d$, \emph{i.e.}, the
latent variable $Z\in\mathcal{Z}\subset\mathbb{R}^{d}$ can have a
distribution supported on some manifold $\mathcal{Z}$. This idea
is realized by the generalized definition of the normal distribution
\citep{zhang2021flow}. In particular, let $A_{s}=\mathbf{diag}(1,\ldots,1,0,\ldots,0)$
be a diagonal matrix whose first $s$ diagonal elements are one and
whose remaining $(d-s)$ diagonal elements are zero, and $Z_{0}$
be a random vector following standard multivariate normal distribution
$N(0,I_{d})$. Then clearly, the random vector $Z=A_{s}Z_{0}$ is
supported on an $s$-dimensional manifold $\mathcal{Z}$, and its
distribution $P_{Z}\equiv P_{A_{s}Z_{0}}$ has dimensions $\texttt{InDim}(P_{Z})=s$
and $\texttt{AmDim}(P_{Z})=d$. For convenience, we use the classic
notation $N(0,A_{s})$ to denote this distribution, although $A_{s}$
is a degenerate covariance matrix.

Choosing $P_{Z}=N(0,A_{s})$, where $s$ is a parameter to estimate,
enables us to solve the dimension mismatch problem in Section \ref{sec:mismatch}.
If $s=r$, then the latent variable $Z$ can be mapped to $G(Z)$
supported on an $r$-dimensional manifold, and meanwhile, $P_{Z}$
and the encoded distribution $P_{Q(X)}$ can have matched intrinsic
dimensions. Formally, Theorem \ref{thm:existence_qg} states that
for any data distribution $P_{X}$ defined by Definition \ref{def:distr_manifold},
there always exist a \emph{continuous} encoder $Q^{\diamond}$ that
guarantees meaningful encodings on an $r$-dimensional manifold, and
a \emph{continuous} generator $G^{\diamond}$ that generates samples
with the same distribution as $P_{X}$, using those latent points
encoded by $Q^{\diamond}$.
\begin{thm}
\label{thm:existence_qg}If $d\ge r$, then there exist two continuous
mappings $Q^{\diamond}:\mathcal{X}\rightarrow\mathcal{Z}$ and $G^{\diamond}:\mathcal{Z}\rightarrow\mathcal{X}$
such that $Q^{\diamond}(X)\sim N(0,A_{r})$ and $X=G^{\diamond}(Q^{\diamond}(X))$.
\end{thm}
In such cases, we call $(Q^{\diamond},G^{\diamond})$ an \emph{optimal}
encoder-generator pair for the data distribution $P_{X}$, and note
that $(Q^{\diamond},G^{\diamond})$ may not be unique. On the other
hand, Corollary \ref{cor:insufficient_dim} below shows that if the
ambient dimension of $P_{Z}$ is insufficient, then the auto-encoder
structure is unable to recover the original distribution of $X$,
which justifies the finding in Figure \ref{fig:scurve}(b).
\begin{cor}
\label{cor:insufficient_dim}Suppose that $d<r$. Then for any continuous
mappings $Q:\mathbb{R}^{p}\rightarrow\mathbb{R}^{d}$ and $G:\mathbb{R}^{d}\rightarrow\mathbb{R}^{p}$,
we have $\mathbb{E}_{X}\left\Vert X-G(Q(X))\right\Vert >0$.
\end{cor}

\subsection{The proposed model}

Theorem \ref{thm:existence_qg} shows the possibility to identify
the dimension of the data manifold $\mathcal{X}$ by learning a latent
distribution with the same intrinsic dimension via the encoder $Q$.
In this section, we realize this idea through our new auto-encoder
generative model, LWGAN, which takes advantages of both WGAN and WAE.
LWGAN is capable of learning $Q$, $G$, and $r$ simultaneously to
accomplish all of our three goals. For brevity, we abbreviate the
subscript $s$ in the matrix $A_{s}$ when no confusion is caused.

There are three probability measures involved in our problem: the
real data distribution $P_{X}$, the generation distribution $P_{G(AZ_{0})}$,
and the reconstruction distribution $P_{G(Q(X))}$. Our goal is to
ensure that all three measures are similar to each other in a systematic
way. To this end, we propose the following distance between $P_{X}$
and $P_{G(AZ_{0})}$ with given $G$ and $A$:
\begin{align}
\overline{W}_{1}(P_{X},P_{G(AZ_{0})}) & =\inf\limits_{Q\in{\cal Q}^{\diamond}}\sup\limits_{f\in{\cal F}^{\diamond}}\mathfrak{L}_{A}(G,Q,f),\label{eq:w1distance}\\
\mathfrak{L}_{A}(G,Q,f) & =\mathbb{E}_{X}\left\Vert X-G(Q(X))\right\Vert +\mathbb{E}_{X}\left[f(G(Q(X)))\right]-\mathbb{E}_{Z_{0}}\left[f(G(AZ_{0}))\right],\nonumber 
\end{align}
where ${\cal F}^{\diamond}$ is the set of all bounded 1-Lipschitz
functions, and ${\cal Q}^{\diamond}$ is the set of continuous encoder
mappings. The term $\mathbb{E}_{X}\left\Vert X-G(Q(X))\right\Vert $
can be viewed as the auto-encoder reconstruction error in WAE, and
also a loss to measure the discrepancy between $P_{X}$ and $P_{G(Q(X))}$.
The other term $\mathbb{E}_{X}\left[f(G(Q(X)))\right]-\mathbb{E}_{Z_{0}}\left[f(G(AZ_{0}))\right]$
quantities the difference between $P_{G(Q(X))}$ and $P_{G(AZ_{0})}$.
Theorem \ref{thm:w1bar} below shows that, under some mild conditions,
(\ref{eq:w1distance}) achieves its minimum as the 1-Wasserstein distance
$W_{1}(P_{X},P_{G(AZ_{0})})$.
\begin{thm}
\label{thm:w1bar}The $\overline{W}_{1}$ distance defined in (\ref{eq:w1distance})
has the following representation:
\begin{equation}
\overline{W}_{1}(P_{X},P_{G(AZ_{0})})=\inf_{Q\in{\cal Q}^{\diamond}}\Big\{ W_{1}(P_{X},P_{G(Q(X))})+W_{1}(P_{G(Q(X))},P_{G(AZ_{0})})\Big\}.\label{eq:lwgan_tri_inequality}
\end{equation}
Therefore, $W_{1}(P_{X},P_{G(AZ_{0})})\le\overline{W}_{1}(P_{X},P_{G(AZ_{0})})$,
and the equality holds if there exists an encoder $Q\in{\cal Q}^{\diamond}$
such that $Q(X)$ has the same distribution as $AZ_{0}$.
\end{thm}
\begin{rem}
\label{rem:w1bar_remark1}Theorem \ref{thm:existence_qg} shows that
there exists some optimal encoder-generator pair $(Q^{\diamond},G^{\diamond})$
such that $Q^{\diamond}(X)\overset{d}{=}A_{r}Z_{0}$ and $X=G^{\diamond}(Q^{\diamond}(X))$.
Therefore, $Q^{\diamond}$ is an optimal solution to (\ref{eq:lwgan_tri_inequality})
for $A=A_{r}$, and hence the equality $W_{1}(P_{X},P_{G(A_{r}Z_{0})})=\overline{W}_{1}(P_{X},P_{G(A_{r}Z_{0})})$
holds. This indicates that $\overline{W}_{1}$ is a tight upper bound
for $W_{1}$. Furthermore, with $G=G^{\diamond}$, we have $\overline{W}_{1}(P_{X},P_{G^{\diamond}(A_{r}Z_{0})})=0$,
which reaches its global minimum.
\end{rem}
\begin{rem}
\label{rem:w1bar_remark2}The condition $Q(X)\overset{d}{=}AZ_{0}$
is sufficient but not necessary for $W_{1}=\overline{W}_{1}$ to hold.
For example, using $(Q^{\diamond},G^{\diamond})$ in the proof of
Theorem \ref{thm:existence_qg}, we can show that $Q^{\diamond}(X)\overset{d}{=}A_{r}Z_{0}$
but $W_{1}(P_{X},P_{G^{\diamond}(A_{s}Z_{0})})=\overline{W}_{1}(P_{X},P_{G^{\diamond}(A_{s}Z_{0})})=0$
for any $s$ such that $r\le s\le d$.
\end{rem}
In our framework, we represent the encoder, generator, and critic
using deep neural networks, $G=G(\cdot;\theta_{G})$, $Q=Q(\cdot;\theta_{Q})$,
$f=f(\cdot;\theta_{f})$, where $\theta=(\theta_{G},\theta_{Q},\theta_{f})$
are the network parameters. We restrict the three components of $\theta$
to compact sets $\Theta_{G}$, $\Theta_{Q}$, and $\Theta_{f}$, respectively,
and further define $\bar{\Theta}_{f}=\left\{ \theta_{f}\in\Theta_{f}:\Vert f(\cdot;\theta_{f})\Vert_{L}\le1\right\} $,
where $\Vert g\Vert_{L}$ stands for the Lipschitz constant of a function
$g$. Then we define the parameter space $\Theta=\Theta_{G}\times\Theta_{Q}\times\bar{\Theta}_{f}$
and function spaces $\mathcal{G}=\{G(\cdot;\theta_{G}):\theta_{G}\in\Theta_{G}\}$,
$\mathcal{Q}=\{Q(\cdot;\theta_{Q}):\theta_{Q}\in\Theta_{Q}\}$, $\mathcal{F}=\{f(\cdot;\theta_{f}):\theta_{f}\in\bar{\Theta}_{f}\}$.
Accordingly, hereafter we replace the spaces $\mathcal{Q}^{\diamond}$
and $\mathcal{F}^{\diamond}$ in (\ref{eq:w1distance}) with $\mathcal{Q}$
and $\mathcal{F}$ respectively for the definition of $\overline{W}_{1}(P_{X},P_{G(AZ_{0})})$.

In practice, we only have the empirical versions of $P_{X}$ and $P_{G(AZ_{0})}$.
Suppose we have observed an i.i.d. data sample $X_{1},\ldots,X_{n}$,
and have simulated an i.i.d. sample of $N(0,I_{d})$, $Z_{0,1},\ldots,Z_{0,n}$,
where $X$ and $Z_{0}$ samples are independent. Then we define
\begin{align*}
L(x,z;\theta) & =\Vert x-G(Q(x;\theta_{Q});\theta_{G})\Vert+f(G(Q(x;\theta_{Q});\theta_{G});\theta_{f})-f(G(z;\theta_{G});\theta_{f}),\\
\ell(\theta,A) & =\mathbb{E}_{X\otimes Z_{0}}[L(X,AZ_{0},\theta)],\quad\hat{\ell}_{n}(\theta,A)=\frac{1}{n}\sum_{i=1}^{n}L(X_{i},AZ_{0,i},\theta),
\end{align*}
where $\mathbb{E}_{X\otimes Z_{0}}$ means taking the expectation
of independent $X$ and $Z_{0}$. Clearly,
\[
\overline{W}_{1}(P_{X},P_{G(AZ_{0})})=\inf_{Q\in\mathcal{Q}}\sup_{f\in\mathcal{F}}\mathfrak{L}(G,Q,f,A)=\inf_{\theta_{Q}\in\Theta_{Q}}\sup_{\theta_{f}\in\bar{\Theta}_{f}}\ell(\theta,A),
\]
and we denote its empirical version as $\overline{W}_{1}(\hat{P}_{X},\hat{P}_{G(AZ_{0})})=\inf_{\theta_{Q}\in\Theta_{Q}}\sup_{\theta_{f}\in\bar{\Theta}_{f}}\hat{\ell}_{n}(\theta,A)$.

Remark \ref{rem:w1bar_remark1} of Theorem \ref{thm:w1bar} motivates
us to estimate the generator $G$ and the rank-revealing matrix $A$
based on the $\overline{W}_{1}$ distance, but Remark \ref{rem:w1bar_remark2}
suggests that purely minimizing $\overline{W}_{1}$ is not enough,
since a matrix $A$ with a rank larger than $r$ can still drive $\overline{W}_{1}$
to zero, the global minimum value. Therefore, we also need to introduce
a penalty term to regularize the rank of $A$. Since $A$ is uniquely
determined by its rank $s$, below $A$ and $s$ are used interchangeably
to represent the rank parameter. Define the rank-regularized objective
function as
\[
\hat{\rho}_{n}(\theta_{G},A)=\overline{W}_{1}(\hat{P}_{X},\hat{P}_{G(AZ_{0})})+\lambda_{n}\cdot\mathbf{rank}(A),
\]
where $\lambda_{n}$ is a deterministic sequence satisfying $\lambda_{n}\rightarrow0$
and $n^{1/2}\lambda_{n}\rightarrow\infty$, which will be justified
in Theorem \ref{thm:rank_consistency}. Then the generator $G$ and
the matrix $A$ are estimated by
\begin{equation}
(\hat{\theta}_{G},\hat{r})=\underset{\theta_{G}\in\Theta_{G},1\le s\le d}{\arg\min}\ \hat{\rho}_{n}(\theta_{G},A_{s}).\label{eq:est_g_r}
\end{equation}
When the optimal points are not unique, $\hat{\theta}_{G}$ can be
chosen arbitrarily from the solution set, and $\hat{r}$ is taken
as the smallest one among all the optimal points.

\subsection{Computational algorithm}

The optimization problem (\ref{eq:est_g_r}) can be solved by computing
the ``rank score''
\begin{equation}
\hat{\varrho}_{n}(s)=\underset{\theta_{G},\theta_{Q}}{\min}\max_{\theta_{f}}\ \hat{\ell}_{n}(\theta,A_{s})+\lambda_{n}s\label{eq:rank_score}
\end{equation}
for each $s=1,\ldots,d$, and then we have $\hat{r}=\arg\min_{s}\hat{\varrho}_{n}(s)$.
Equivalently, we need to solve
\begin{equation}
\begin{aligned}\min_{G_{1},Q_{1}}\max_{f_{1}} & \ \frac{1}{n}\sum_{i=1}^{n}\left[\Vert X_{i}-G_{1}(Q_{1}(X_{i}))\Vert+f_{1}(G_{1}(Q_{1}(X_{i})))-f_{1}(G_{1}(A_{1}Z_{0,i}))\right]+\lambda_{n}\cdot1\\
\cdots & \ \cdots\\
\min_{G_{d},Q_{d}}\max_{f_{d}} & \ \frac{1}{n}\sum_{i=1}^{n}\left[\Vert X_{i}-G_{d}(Q_{d}(X_{i}))\Vert+f_{d}(G_{d}(Q_{d}(X_{i})))-f_{d}(G_{d}(A_{d}Z_{0,i}))\right]+\lambda_{n}\cdot d
\end{aligned}
\label{eq:d_problems}
\end{equation}
by fitting $d$ different sets of neural networks $(G_{s},Q_{s},f_{s})$,
$s=1,\ldots,d$, which may be time-consuming. Instead, we propose
a practical and efficient algorithm based on the idea that encoder
and critic functions under different ranks can share network parameters.
We slightly modify the network structures of $Q(x;\theta_{Q})$ and
$f(x;\theta_{f})$ such that they also receive a rank input $e_{s}$,
where the one-hot encoding vector $e_{s}$ is the $s$-th column of
the identity matrix $I_{d}$. As a result, the rank-aware encoder
and critic functions become $Q(x,e_{s};\theta_{Q})$ and $f(x,e_{s};\theta_{f})$,
respectively. We also make the output of $Q(x,e_{s};\theta_{Q})$
to have rank $s$ by setting the last $(d-s)$ components to zero.
The generator $G$ does not need this modification, since its input
$Q(X,e_{s})$ or $A_{s}Z_{0}$ already contains the rank information.

Then problem (\ref{eq:d_problems}) is equivalent to solving
\begin{equation}
\min_{G,Q}\max_{f}\ \frac{1}{nd}\sum_{s=1}^{d}\sum_{i=1}^{n}\left[\Vert X_{i}-G(Q(X_{i},e_{s}))\Vert+f(G(Q(X_{i},e_{s})),e_{s})-f(G(A_{s}Z_{0,i}),e_{s})\right],\label{eq:one_single_problem}
\end{equation}
as long as the rank-aware neural networks $(G,Q,f)$ have sufficient
expressive powers. This would be a reasonable assumption if we recognize
that $(G_{s},Q_{s},f_{s})$ and $(G_{t},Q_{t},f_{t})$ should be similar
if $s\approx t$. In practice, this means that $(G_{s},Q_{s},f_{s})$
and $(G_{t},Q_{t},f_{t})$ can share most of the neural network parameters,
and their difference is reflected by the input rank information $e_{s}$.
Also note that the rank penalty terms in (\ref{eq:d_problems}) are
tentatively dropped, since they only affect the estimation of $s$
but not $(G,Q,f)$. The rank terms will be added back once the optimal
$(G,Q,f)$ are obtained.

Furthermore, the objective function of (\ref{eq:one_single_problem})
can be viewed as an empirical expectation over $(X,Z,S)$, where the
average term $d^{-1}\sum_{s=1}^{d}(\cdot)$ represents an expectation
$\mathbb{E}_{S}(\cdot)$ with $S$ following a discrete uniform distribution
on $\{1,\ldots,d\}$. Therefore, to further save computing time, we
can randomly pick a rank in each iteration, and then update $(G,Q,f)$
accordingly. In our numerical experiments, we have saved various metrics
to monitor the training procecss, and they demonstrate that this computing
algorithm is both stable and efficient (see Section S2.3 of the supplementary
material).

The training details are summarized in Algorithm \ref{alg:lwgan}.
In our algorithm, the 1-Lipschitz constraint on the critic $f$ is
enforced by the gradient penalty technique proposed in \citet{gulrajani2017improved},
where $\hat{X}$ is sampled uniformly along the segment between pairs
of points sampled from $P_{X}$ and $P_{G(AZ_{0})}$, and $\lambda_{\mathrm{GP}}$
is the regularization level of the gradient penalty. The operator
$\mathrm{Adam}(\cdot)$ means applying the Adam optimization method
\citep{kingma2014adam} to update neural network parameters $\theta$.

\begin{algorithm}[h]
\caption{\protect\label{alg:lwgan}The training algorithm of LWGAN.}


\begin{algorithmic}[1]

\REQUIRE Initial parameter value $\theta^{(0)}$, batch size $M$,
critic update frequency $L$, gradient penalty parameter $\lambda_{\mathrm{GP}}$,
rank regularization parameter $\lambda_{n}$.

\ENSURE Neural network parameters $\hat{\theta}$, estimated intrinsic
dimension $\hat{r}$.

\FOR{ $k=1,2,\ldots,T$ }

\STATE Randomly select an integer $s$ from $1,\ldots,d$ with equal
probabilities

\STATE Set $\theta^{(k,0)}\leftarrow\theta^{(k-1)}$

\FOR{ $l=1,2,\ldots,L$ }

\STATE Sample real data $X_{1},\ldots,X_{M}\overset{iid}{\sim}P_{X}$,
latent data $Z_{0,1},\ldots,Z_{0,M}\overset{iid}{\sim}N(0,I_{d})$,
and $\varepsilon_{1},\ldots,\varepsilon_{M}\overset{iid}{\sim}\mathrm{Unif}(0,1)$

\STATE Set $\hat{X}_{i}=\varepsilon_{i}X_{i}+(1-\varepsilon_{i})G(A_{s}Z_{0,i};\theta_{G}^{(k)})$,
$i=1,\ldots,M$

\STATE Define $J(\theta)=\hat{\ell}_{M}(\theta,A_{s})+\lambda_{\mathrm{GP}}\cdot M^{-1}\sum_{i=1}^{M}\left(\Vert\nabla_{x}f(\hat{X}_{i};\theta_{f})\Vert-1\right)^{2}$

\STATE Update $\theta_{f}^{(k,l)}\leftarrow\theta_{f}^{(k,l-1)}+\mathrm{Adam}\left(-\nabla_{\theta_{f}}J(\theta)|_{\theta=\theta^{(k,l-1)}}\right)$

\ENDFOR

\STATE Sample real data $X_{1},\ldots,X_{M}\overset{iid}{\sim}P_{X}$
and latent data $Z_{0,1},\ldots,Z_{0,M}\overset{iid}{\sim}N(0,I_{d})$

\STATE Update $\theta_{G,Q}^{(k)}\leftarrow\theta_{G,Q}^{(k,L)}+\mathrm{Adam}\left(\nabla_{\theta_{G,Q}}\hat{\ell}_{M}(\theta,A_{s})|_{\theta=\theta^{(k,L)}}\right)$

\IF{ $\theta^{(k)}$ converges }

\STATE Compute $\hat{\varrho}_{n}(s)=\hat{\ell}_{n}(\theta^{(k)},A_{s})+\lambda_{n}s$,
$s=1,\ldots,d$

\RETURN $\hat{\theta}=\theta^{(k)}$, $\hat{r}=\arg\min_{s}\hat{\varrho}_{n}(s)$

\ENDIF

\ENDFOR

\end{algorithmic}

\end{algorithm}

\subsection{Tuning parameter selection}

\label{subsec:tuning_parameter}

Another critical issue in applying LWGAN to real-life data is the
selection of the regularzation parameter $\lambda_{n}$ in (\ref{eq:rank_score}).
From a theoretical perspective, in Section \ref{sec:theory} we will
show that $\lambda_{n}$ should be chosen such that $\lambda_{n}\rightarrow0$
and $n^{1/2}\lambda_{n}\rightarrow\infty$, whereas in this section,
we propose a more practical and data-driven scheme for selecting $\lambda_{n}$.
The intuition is to note that without the rank penalty, $\hat{V}_{n}(A_{s})\coloneqq\hat{\varrho}_{n}(s)-\lambda_{n}s$
would all be close to zero for $s\ge r$, and their differences are
mainly attributed to the randomness from estimation. Therefore, if
we can estimate the standard errors of $\hat{V}_{n}(A_{s})$ for $s\ge r$,
then $\lambda_{n}$ should be chosen slightly larger than the estimated
standard error, so as to encourage the selection of the simplest model,
namely, the model with the smallest rank $s$.

Concretely, we use the following method to determine the data-driven
$\lambda_{n}$. First, train the model to optimum according to Algorithm
\ref{alg:lwgan}, using the whole training dataset. Second, continue
to train the model for $\tilde{T}$ iterations, using a subset of
the training data, denoted as $\tilde{X}_{1}$. This can be viewed
as fitting a model on $\tilde{X}_{1}$ based on a warm start. Third,
based on this model, compute the metric $\hat{V}_{n}(A_{s})$ for
each $s$, and we use the symbol $\hat{V}_{1s}$ to denote its value.
Then repeat this process on different training data subsets $\tilde{X}_{k}$,
$k=2,\ldots,\tilde{K}$, and similarly compute the scores $\hat{V}_{ks}$,
$k=2,\ldots,\tilde{K}$, $s=1,\ldots,d$. Let
\[
\tilde{r}=\underset{s}{\arg\min}\ \hat{V}_{\cdot s}\coloneqq\frac{1}{\tilde{K}}\sum_{k=1}^{\tilde{K}}\hat{V}_{ks},\quad\widehat{\mathrm{SE}}=\sqrt{\frac{1}{\tilde{K}-1}\sum_{k=1}^{\tilde{K}}\left(\hat{V}_{k\tilde{r}}-\hat{V}_{\cdot\tilde{r}}\right)^{2}}.
\]
In other words, we first find the rank $s$ that has the smallest
mean value $\hat{V}_{\cdot s}$, and then estimate the standard error
of the mean on this rank. Finally, we set $\lambda_{n}=\widehat{\mathrm{SE}}^{0.8}$.
In a typical setting, $\widehat{\mathrm{SE}}=O(n^{-1/2})$, so $\lambda_{n}=O(n^{-0.4})$
satisfies the theoretical rate. Our numerical experiments use $\tilde{T}=20$
and $\tilde{K}=50$, so this method essentially trains the model for
additional 1000 iterations, which is relatively small compared to
the main training cost for real-life datasets.

\section{Theoretical Results}

\label{sec:theory}

\subsection{Generalization error bound}

\label{subsec:generalization_bound}

Since the LWGAN model highly relies on the $\overline{W}_{1}$ distance,
and the estimators are based on its empirical version, a natural question
is how well the empirical quantity $\overline{W}_{1}(\hat{P}_{X},\hat{P}_{G(AZ_{0})})$
approximates the population quantity $\overline{W}_{1}(P_{X},P_{G(AZ_{0})})$.
This problem can be characterized by the generalization error. In
the context of supervised learning, the generalization error is defined
as the gap between the empirical risk (\emph{i.e.}, the training error)
and the the expected risk (\emph{i.e.}, the testing error). Similarly,
in the framework of LWGAN, we make the following definition derived
from \citet{arora2017generalization}.
\begin{defn}
Given $\hat{P}_{X}$, an empirical version of the true data distribution
with $n$ observations, a generation distribution $P_{G(AZ_{0})}$
\emph{generalizes} under the $\overline{W}_{1}(\cdot,\cdot)$ distance
with generalization error $\varepsilon$, if
\[
\left|\overline{W}_{1}(P_{X},P_{G(AZ_{0})})-\overline{W}_{1}(\hat{P}_{X},\hat{P}_{G(AZ_{0})})\right|\le\varepsilon
\]
holds with a high probability, where $\hat{P}_{G(AZ_{0})}$ is an
empirical version of the generation distribution $P_{G(AZ_{0})}$
with polynomial number of observations drawn after $P_{G(AZ_{0})}$
is fixed.
\end{defn}
Since the empirical version is what we have access to in practice,
a small generalization error implies that after we minimize the empirical
$\overline{W}_{1}$ distance, we can expect a small distance between
the true data distribution and the generation distribution. To present
the theorem below, we define the function sets $\mathcal{F}\circ G\circ\mathcal{Q}=\{f\circ G\circ Q:f\in\mathcal{F},Q\in\mathcal{Q}\}$
and $\mathcal{F}\circ G\circ\mathcal{A}=\{h:h(z)=f(G(A_{s}z)),f\in\mathcal{F},1\le s\le d\}$.
\begin{thm}
\label{thm:generalization}Assume that $\Vert x\Vert\le B$ for all
$x\in\mathcal{X}$, and every function in $\mathcal{Q}$ is $L_{Q}$-Lipschitz
with respect to the input and $L_{\theta_{Q}}$-Lipschitz with respect
to the parameter. For a fixed $L_{G}$-Lipschitz generator $G$, let
$\hat{\Theta}_{Q}$ be an $\varepsilon/(8L_{G}L_{\theta_{Q}})$-net
of the encoder parameter space $\Theta_{Q}$. Then with a probability
at least
\[
1-e^{-d}-2d|\hat{\Theta}_{Q}|\exp\left\{ -\frac{n\varepsilon^{2}}{8[(1+2L_{G}L_{Q})B+L_{G}t_{n,d}]^{2}}\right\} ,
\]
where $t_{n,d}=\sqrt{3d+2\log n+2\sqrt{d^{2}+d\log n}}$, the following
inequality holds:
\begin{equation}
\max_{1\le s\le d}\left|\overline{W}_{1}(P_{X},P_{G(A_{s}Z_{0})})-\overline{W}_{1}(\hat{P}_{X},\hat{P}_{G(A_{s}Z_{0})})\right|\le2\mathfrak{R}_{n}(\mathcal{F}\circ G\circ\mathcal{Q})+2\mathfrak{R}_{n}(\mathcal{F}\circ G\circ\mathcal{A})+\varepsilon,\label{eq:error_bound}
\end{equation}
where $\mathfrak{R}_{n}(\mathcal{F}\circ G\circ\mathcal{Q})=\mathbb{E}_{\delta}\left\{ \sup_{f\in\mathcal{F},Q\in\mathcal{Q}}n^{-1}\sum_{i=1}^{n}\delta_{i}f(G(Q(X_{i})))\right\} $
and $\mathfrak{R}_{n}(\mathcal{F}\circ G\circ\mathcal{A})=\allowbreak\mathbb{E}_{\delta}\left\{ \sup_{f\in\mathcal{F},1\le s\le d}n^{-1}\sum_{i=1}^{n}\delta_{i}f(G(A_{s}Z_{0,i}))\right\} $
are Rademacher complexities of the function sets $\mathcal{F}\circ G\circ\mathcal{Q}$
and $\mathcal{F}\circ G\circ\mathcal{A}$, respectively, $\delta=(\delta_{1},\ldots,\delta_{n})$
are independent Rademacher variables, i.e., $P(\delta_{i}=1)=P(\delta_{i}=-1)=1/2$,
and $\mathbb{E}_{\delta}$ stands for expectations with respect to
$\delta$ while fixing $X$ and $Z_{0}$.
\end{thm}
Theorem \ref{thm:generalization} describes how the function classes
$\mathcal{F}$ and $\mathcal{Q}$ contribute to the generalization
error bound in our framework. Given a fixed generator $G$, there
exists a uniform upper bound for any critic $f\in\mathcal{F}$, encoder
$Q\in\mathcal{Q}$, and low-rank matrix $A$ with appropriate numbers
of observations from $P_{X}$ and $P_{Z_{0}}$. More concretely, if
$|\hat{\Theta}_{Q}|$ is small and the sample size is large, then
the generalization error is consequently guaranteed to hold with a
high probability. In \citet{gao2021theoretical}, it has been proved
that $\log(|\hat{\Theta}_{Q}|)\le\mathcal{O}(K_{Q}^{2}D_{Q}\log(D_{Q}L_{Q}L_{G}L_{\theta_{Q}}/\varepsilon))$,
where $K_{Q}$ and $D_{Q}$ denote the width and depth of $Q$, respectively.
Additionally, the Lipschitz constants of $Q$ and $G$ are under the
control of the spectral normalization of their weights.

The Rademacher complexities in (\ref{eq:error_bound}) measure the
richness of a class of real-valued functions with respect to a probability
distribution. There are several existing results on the Rademacher
complexity of neural networks. For example, under some mild conditions,
$\mathfrak{R}_{n}(\mathcal{F}\circ G\circ\mathcal{Q})$ is upper bounded
by an order scaling as $\mathcal{O}(L_{G}L_{Q}\sqrt{(K_{Q}^{2}D_{Q}+K_{f}^{2}D_{f})/n})$,
where $K_{f}$ and $D_{f}$ denote the width and depth of $f$, respectively.
Similarly, an upper bound on $\mathfrak{R}_{n}(\mathcal{F}\circ G\circ\mathcal{A})$
scales as $\mathcal{O}(L_{G}\sqrt{(d^{2}+K_{f}^{2}D_{f})/n})$ \citep{gao2021theoretical}.

Finally, since $\overline{W}_{1}(P_{X},P_{G(AZ_{0})})$ is a tight
upper bound for the 1-Wasserstein distance between $P_{X}$ and $P_{G(AZ_{0})}$
from Theorem \ref{thm:w1bar}, we further have
\[
W_{1}(P_{X},P_{G(A_{s}Z_{0})})\le\overline{W}_{1}(\hat{P}_{X},\hat{P}_{G(A_{s}Z_{0})})+2\mathfrak{R}_{n}(\mathcal{F}\circ G\circ\mathcal{Q})+2\mathfrak{R}_{n}(\mathcal{F}\circ G\circ\mathcal{A})+\varepsilon
\]
with a high probability. This implies that from the population perspective,
the real data distribution is close to the generation distribution
with respective to the 1-Wasserstein distance when we minimize the
empirical loss function $\overline{W}_{1}(\hat{P}_{X},\hat{P}_{G(A_{s}Z_{0})})$.

\subsection{Estimation consistency}

Theorem \ref{thm:existence_qg} has shown that an optimal encoder-generator
pair globally minimizes the $\overline{W}_{1}(P_{X},P_{G(AZ_{0})})$
distance under a suitable rank of $A$, and equation (\ref{eq:est_g_r})
indicates that the encoder and generator are estimated by minimizing
the empirical version $\overline{W}_{1}(\hat{P}_{X},\hat{P}_{G(AZ_{0})})$.
Therefore, the question of interest here is how the estimated quantities
relate to the population ones.

However, unlike regular parameter estimation problems, an important
property of the encoder-generator structure in LWGAN is that the encoder-generator
pair may not be unique even with the same objective function value.
For example, when $Q$ and $G$ simultaneously permute the first $s$
output and input variables, respectively, the corresponding value
of $\mathfrak{L}_{A}(G,Q,f)$ does not change. Therefore, the optimal
solutions to (\ref{eq:est_g_r}) are not singletons but set-valued.
In this section, we first fix the rank of $A$, and consider the estimation
consistency through a distance between sets called Hausdorff distance
\citep{rockafellar2009variational}. We defer the estimation of the
optimal rank of $A$, or equivalently, $\texttt{InDim}(P_{X})$, to
Section \ref{subsec:rank_consistency}.

For any two non-empty bounded subsets $S_{1}$ and $S_{2}$ of some
Euclidean space, the Hausdorff distance between $S_{1}$ and $S_{2}$
is defined as
\[
d_{H}(S_{1},S_{2})=\max\left\{ \sup_{a\in S_{1}}d(a,S_{2}),\sup_{b\in S_{2}}d(b,S_{1})\right\} ,
\]
where $d(x,S)=\inf_{y\in S}\Vert x-y\Vert$ is the shortest distance
from a point $x$ to a set $S$. The Hausdorff distance $d_{H}$ is
a metric for non-empty compact sets, and $d_{H}(S_{1},S_{2})=0$ if
and only if $S_{1}=S_{2}$.

Recall that we represent $G$, $Q$, and $f$ using deep neural networks,
and we pre-specify the network structures for these mappings, such
as the widths and depths. In this section we only consider functions
within the space $\mathcal{G}\times\mathcal{Q}\times\mathcal{F}$.
Introduce the function $\phi_{A}(\theta_{G},\theta_{Q})=\sup_{\theta_{f}}\ell(\theta,A)$,
and then an optimal solution $\theta^{*}$ solves
\[
\inf_{\theta_{G}}\,\overline{W}_{1}(P_{X},P_{G(AZ_{0})})=\inf_{\theta_{G},\theta_{Q}}\sup_{\theta_{f}}\,\ell(\theta,A)=\inf_{\theta_{G},\theta_{Q}}\phi_{A}(\theta_{G},\theta_{Q})
\]
when it is a solution to both the outer minimization problem and the
inner maximization problem. Therefore, the optimal solution set $\Theta_{A}^{*}$
is defined as
\[
\Theta_{A}^{*}=\left\{ \theta^{*}\in\Theta:\phi_{A}(\theta_{G}^{*},\theta_{Q}^{*})=\inf_{\theta_{G},\theta_{Q}}\phi_{A}(\theta_{G},\theta_{Q}),\ \ell(\theta^{*},A)=\phi_{A}(\theta_{G}^{*},\theta_{Q}^{*})\right\} .
\]
For the empirical minimax problem $\inf_{\theta_{G},\theta_{Q}}\sup_{\theta_{f}}\hat{\ell}_{n}(\theta,A)$,
algorithms typically search for approximate solutions rather than
exact ones. Therefore, we define the empirical solution set with slackness
level $\tau_{n}$ as
\[
\hat{\Theta}_{n,A}^{*}(\tau_{n})=\left\{ \theta^{*}\in\Theta:\hat{\phi}_{A}(\theta_{G}^{*},\theta_{Q}^{*})\le\inf_{\theta_{G},\theta_{Q}}\phi_{A}(\theta_{G},\theta_{Q})+\tau_{n},\ \hat{\ell}_{n}(\theta^{*},A)\ge\hat{\phi}_{A}(\theta_{G}^{*},\theta_{Q}^{*})-\tau_{n}\right\} ,
\]
where $\hat{\phi}_{A}(\theta_{G},\theta_{Q})=\sup_{\theta_{f}}\hat{\ell}_{n}(\theta,A)$,
and $\tau_{n}$ is a sequence of non-negative random variables such
that $\tau_{n}\overset{P}{\rightarrow}0$. We further make some assumptions
on the LWGAN model:
\begin{assumption}
\label{assu:deriv}(a) $\Theta$ is a compact set. (b) The function
$L(x,z;\theta)$ is continuously differentiable on $\Theta$ for all
$(x,z)$ with
\[
\mathbb{E}_{X\otimes Z_{0}}\left[\sup_{\theta\in\Theta}\left\Vert \frac{\partial}{\partial\theta}L(X,A_{s}Z_{0};\theta)\right\Vert ^{2}\right]<\infty,\ s=1,\ldots,d.
\]
\end{assumption}
The compact parameter space assumption simplifies the asymptotic analysis.
The moment condition rules out degenerate cases, and the differentiability
is a common requirement for GAN training as various gradient descent-ascent
algorithms are used. Then we adopt the ideas from \citet{meitz2021statistical}
to prove the estimation consistency of LWGAN.
\begin{thm}
\label{thm:hausdorff_consistency}Suppose that $\tau_{n}$ is a sequence
of non-negative random variables such that $\tau_{n}\overset{P}{\rightarrow}0$
and $n^{-1/2}/\tau_{n}\overset{P}{\rightarrow}0$. Then for a fixed
$A$, under Assumption \ref{assu:deriv}, $d_{H}(\hat{\Theta}_{n,A}^{*}(\tau_{n}),\Theta_{A}^{*})\overset{P}{\rightarrow}0$
as $n\rightarrow\infty$.
\end{thm}
Theorem \ref{thm:hausdorff_consistency} assures that the encoder,
generator, and critic estimators of LWGAN are consistent under the
Hausdorff distance for a fixed latent dimension.

\subsection{Intrinsic dimension consistency}

\label{subsec:rank_consistency}

Finally, we show that the estimator $\hat{r}$ computed from (\ref{eq:est_g_r})
is capable of recovering the intrinsic dimension of $P_{X}$. To this
end, we need to further assume that the neural network function space
$\mathcal{G}\times\mathcal{Q}\times\mathcal{F}$ is large enough to
cover some optimal points of interest. Define $\mathfrak{F}_{A}(G,Q)=\sup_{f\in{\cal F}^{\diamond}}\mathfrak{L}_{A}(G,Q,f)$,
and let $\mathcal{G}^{\diamond}$ denote the set of continuous generators.
Then the optimal solution set of minimizing $\overline{W}_{1}(P_{X},P_{G(AZ_{0})})$
can be characterized as
\[
\mathcal{S}_{A}=\left\{ (G^{*},Q^{*},f^{*}):\mathfrak{F}_{A}(G^{*},Q^{*})=\inf_{\substack{Q\in{\cal Q}^{\diamond}\\
G\in\mathcal{G}^{\diamond}
}
}\mathfrak{F}_{A}(G,Q),\ \mathfrak{L}_{A}(G^{*},Q^{*},f^{*})=\sup_{f\in{\cal F}^{\diamond}}\mathfrak{L}_{A}(G^{*},Q^{*},f)\right\} .
\]
Clearly, coupled with some $f^{\diamond}\in\mathcal{F}^{\diamond}$,
we have $(G^{\diamond},Q^{\diamond},f^{\diamond})\in\mathcal{S}_{A_{r}}$.
We then make the following assumption.
\begin{assumption}
\label{assu:theta}(a) $\mathcal{S}_{A_{r}}\cap(\mathcal{G}\times\mathcal{Q}\times\mathcal{F})\neq\varnothing$.
(b) For each $s<r$, there exists a triplet $(G_{s}^{*},Q_{s}^{*},f_{s}^{*})\in\mathcal{S}_{A_{s}}$
such that $f_{s}^{*}\in\mathcal{F}$ and
\[
\sup_{f\in\mathcal{F}}\mathfrak{L}_{A_{s}}(G_{s}^{*},Q_{s}^{*},f)=\inf_{\substack{Q\in{\cal Q}\\
G\in\mathcal{G}
}
}\sup\limits_{f\in{\cal F}}\mathfrak{L}_{A_{s}}(G,Q,f).
\]
\end{assumption}
Now we are ready to show that the rank estimated from (\ref{eq:est_g_r})
approaches the intrinsic dimension of $\mathcal{X}$ as the sample
size grows.
\begin{thm}
\label{thm:rank_consistency}Assume that Assumptions \ref{assu:deriv}
and \ref{assu:theta} hold. Then with $\lambda_{n}\rightarrow0$ and
$n^{1/2}\lambda_{n}\rightarrow\infty$, we have $P(\hat{r}=r)\rightarrow1$,
where $r=\texttt{InDim}(P_{X})$ stands for the intrinsic dimension
of $\mathcal{X}$.
\end{thm}
Theorem \ref{thm:rank_consistency} can be compared to the well-known
Bayesian information criterion (BIC) for model selection of the following
form:
\begin{equation}
n^{-1}\mathrm{BIC}=-\frac{2}{n}L(\hat{\theta};X_{1},\ldots,X_{n})+\frac{\log(n)}{n}\cdot s,\label{eq:BIC}
\end{equation}
where $L(\hat{\theta};X_{1},\ldots,X_{n})=\sum_{i=1}^{n}\log p(X_{i};\hat{\theta})$
is the maximized likelihood function of the model $p(x;\theta)$,
$\hat{\theta}$ is the maximum likelihood estimator, and $s$ is the
number of parameters. We normalize BIC by $n$ in (\ref{eq:BIC})
to make the first term comparable to an expectation.

To some extent, LWGAN and BIC share perceptible similarities. For
example, if we interpret the rank $s$ as the complexity of the model,
then both LWGAN and BIC construct a penalty term $\lambda_{n}\cdot s$
with $\lambda_{n}\rightarrow0$. More importantly, they both promise
some type of model selection consistency. However, there are some
fundamental differences between LWGAN and BIC. First, the theoretical
rates are different. BIC has $\lambda_{n}=\log(n)/n$, whereas in
LWGAN we require $\lambda_{n}\rightarrow0$ and $n^{1/2}\lambda_{n}\rightarrow\infty$.
Second, BIC is mostly a likelihood-based criterion, whereas in LWGAN,
the main part is based on the $\overline{W}_{1}$ distance given in
(\ref{eq:w1distance}). Third, in the BIC framework, $s$ always represents
the number of parameters, but in LWGAN, this quantity is not meaningful,
as neural networks are known to be highly overparameterized.

\section{Experimental Results}

\label{sec:experiment}

In this section, we conduct comprehensive numerical experiments to
validate that LWGAN is able to achieve our three goals simultaneously:
detecting the correct intrinsic dimension, generating high-quality
samples, and obtaining small reconstruction errors. The programming
code to reproduce the experiment results is available at \url{https://github.com/yixuan/LWGAN}.

\subsection{Simulated experiments}

We first verify our method using three toy examples supported on manifolds
with increasing dimensions. Besides the S-curve data introduced in
Section \ref{sec:mismatch}, the other two datasets are generated
as:
\begin{enumerate}
\item Swiss roll: $X_{1}=V\cos(V)$, $X_{2}=V\sin(V)$, where $V=3\pi(1+2U)/2$,
$U\sim N(0,1)$.
\item Hyperplane: $X_{1},X_{2},X_{3},X_{4}\overset{iid}{\sim}N(0,1)$, $X_{5}=X_{1}+X_{2}+X_{3}+X_{4}^{2}$.
\end{enumerate}
The scatterplots for the three datasets are shown in the first column
of Figure \ref{fig:toy_examples}. It is straightforward to find that
the intrinsic dimensions of the Swiss roll, S-curve, and Hyperplane
datasets are one, two, and four, respectively.

\begin{figure}[h]
\begin{centering}
\subfloat[Swiss roll]{\begin{centering}
\includegraphics[width=0.25\textwidth]{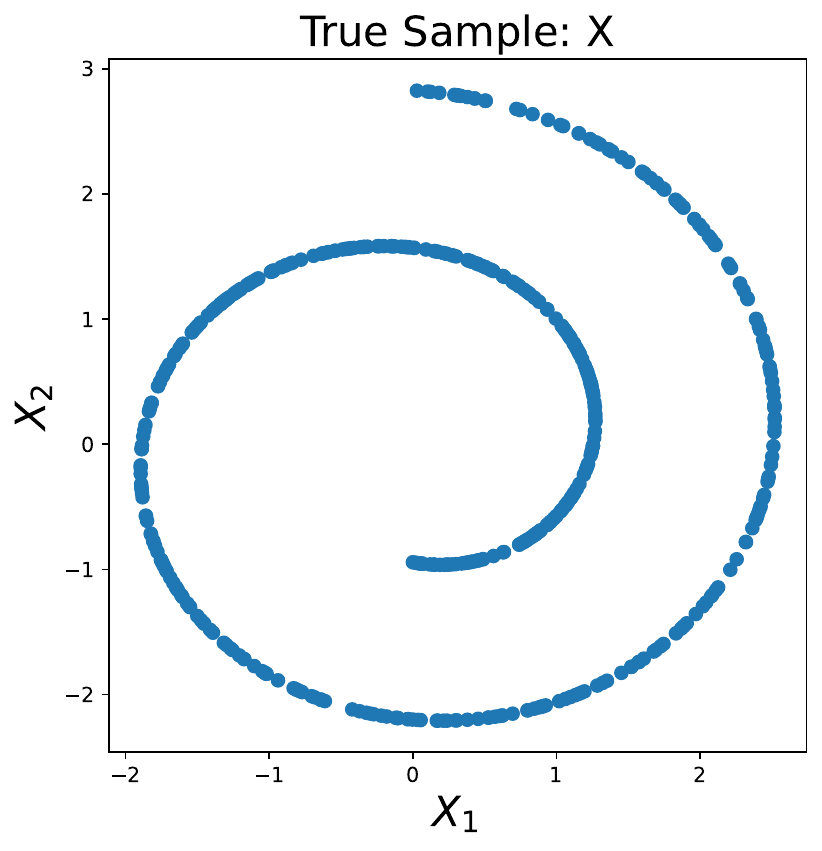}\includegraphics[width=0.24\textwidth]{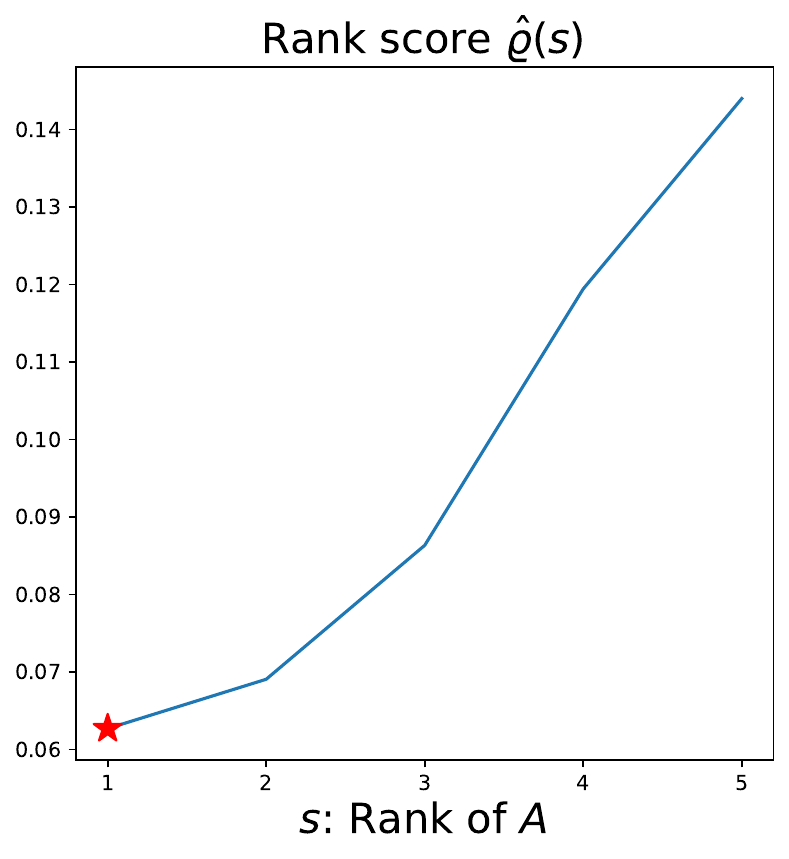}\includegraphics[width=0.25\textwidth]{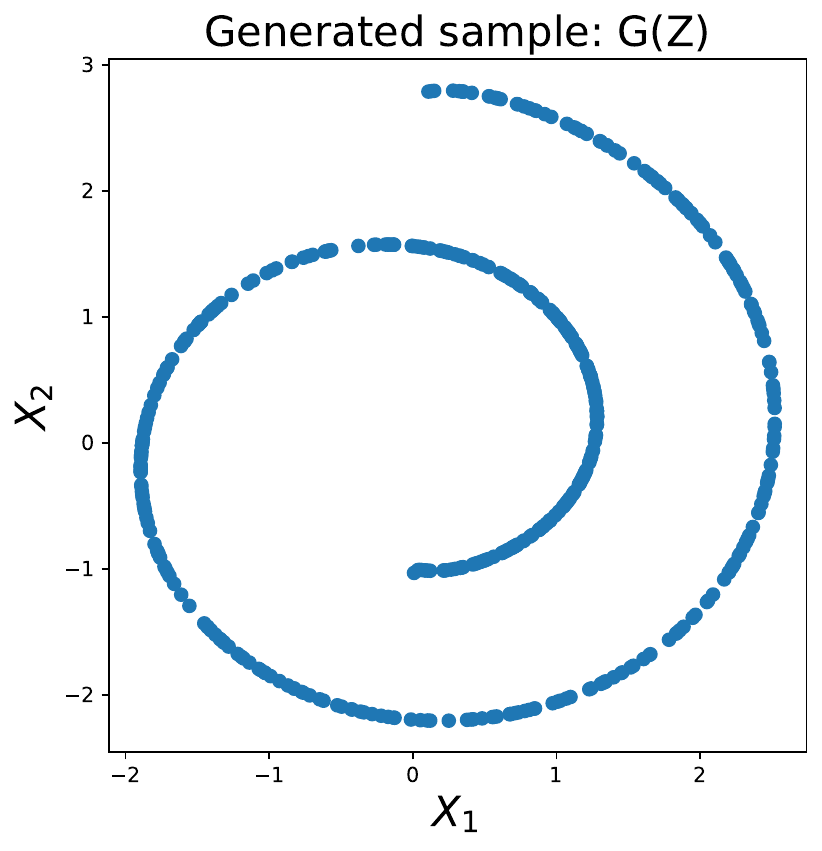}\includegraphics[width=0.25\textwidth]{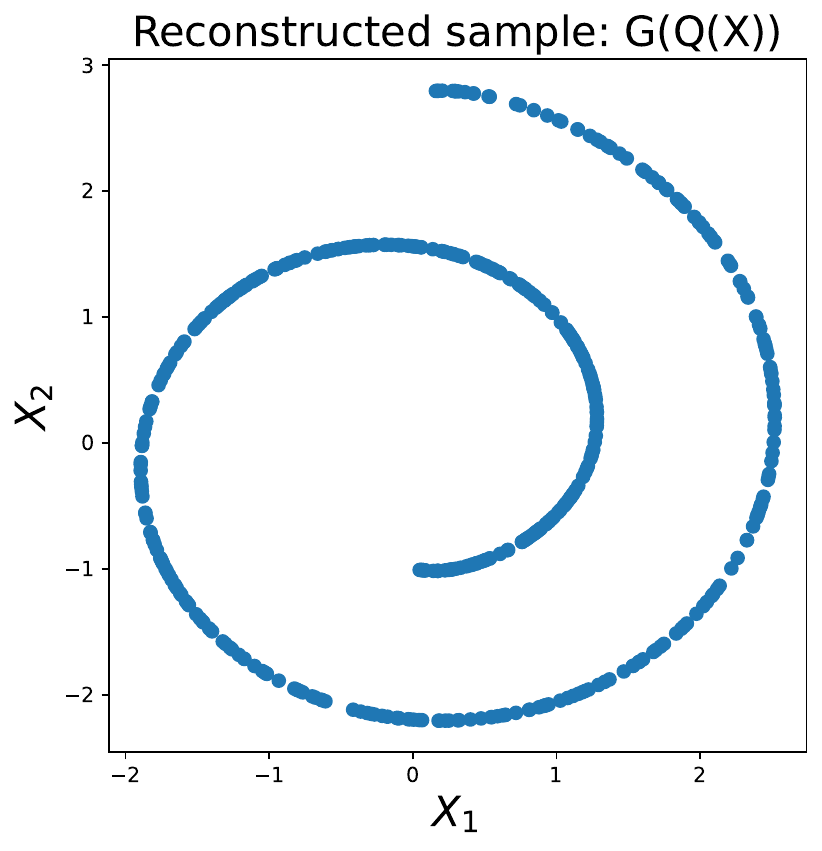}
\par\end{centering}
}
\par\end{centering}
\begin{centering}
\subfloat[S-curve]{\begin{centering}
\includegraphics[width=0.25\textwidth]{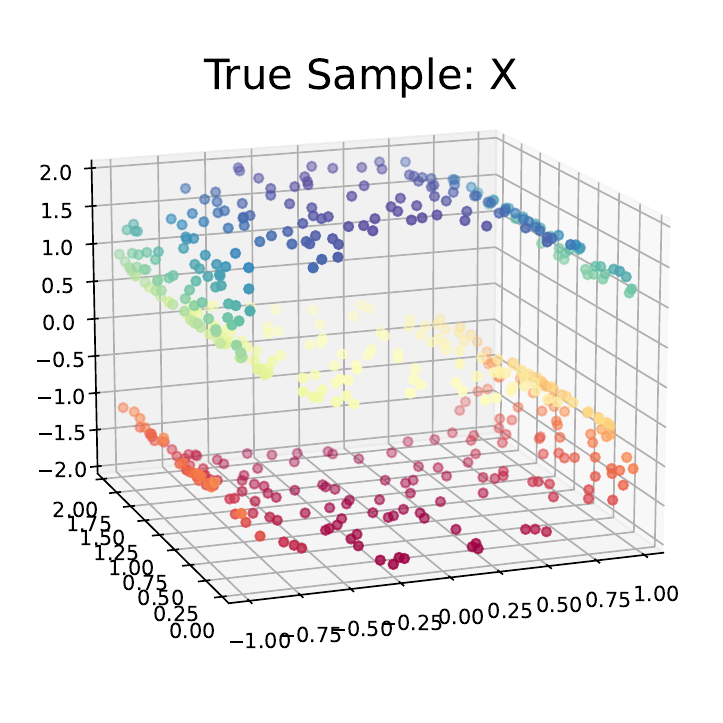} \includegraphics[width=0.232\textwidth]{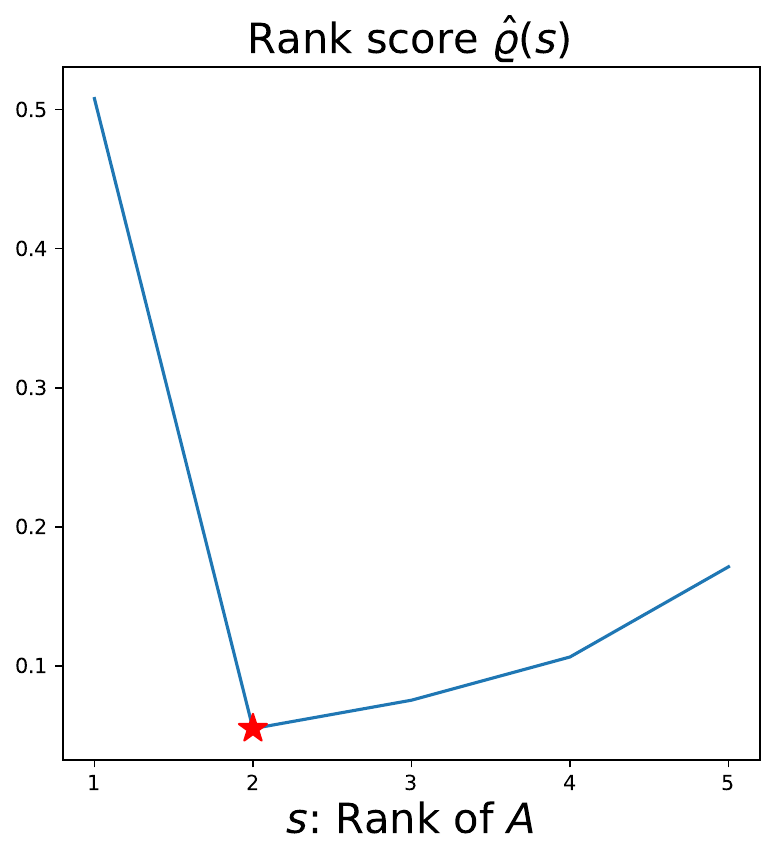}
\includegraphics[width=0.25\textwidth]{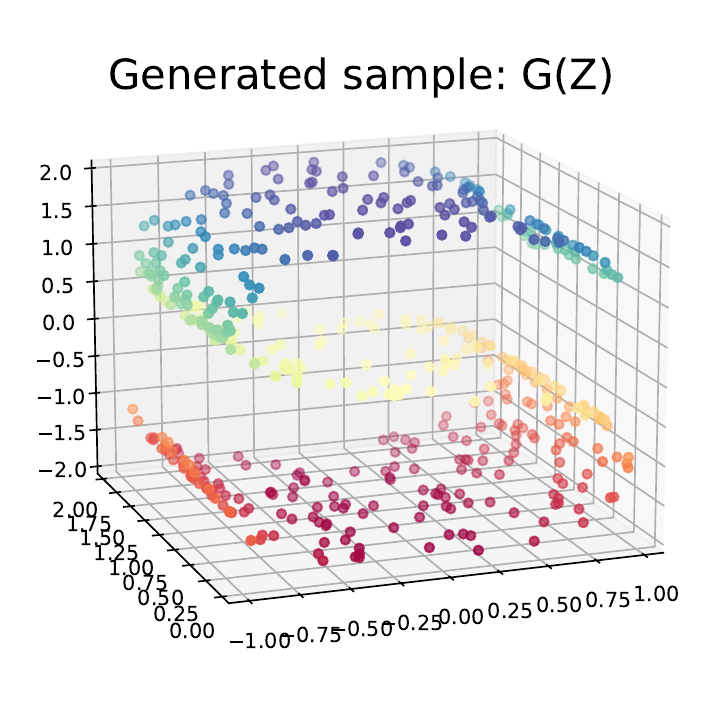}\includegraphics[width=0.25\textwidth]{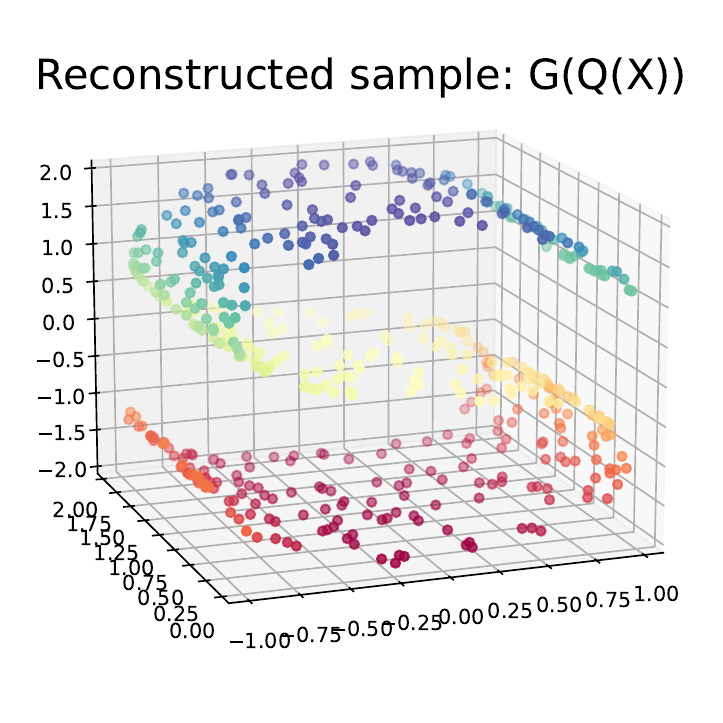}
\par\end{centering}
}
\par\end{centering}
\begin{centering}
\subfloat[Hyperplane]{\begin{centering}
\includegraphics[width=0.25\textwidth]{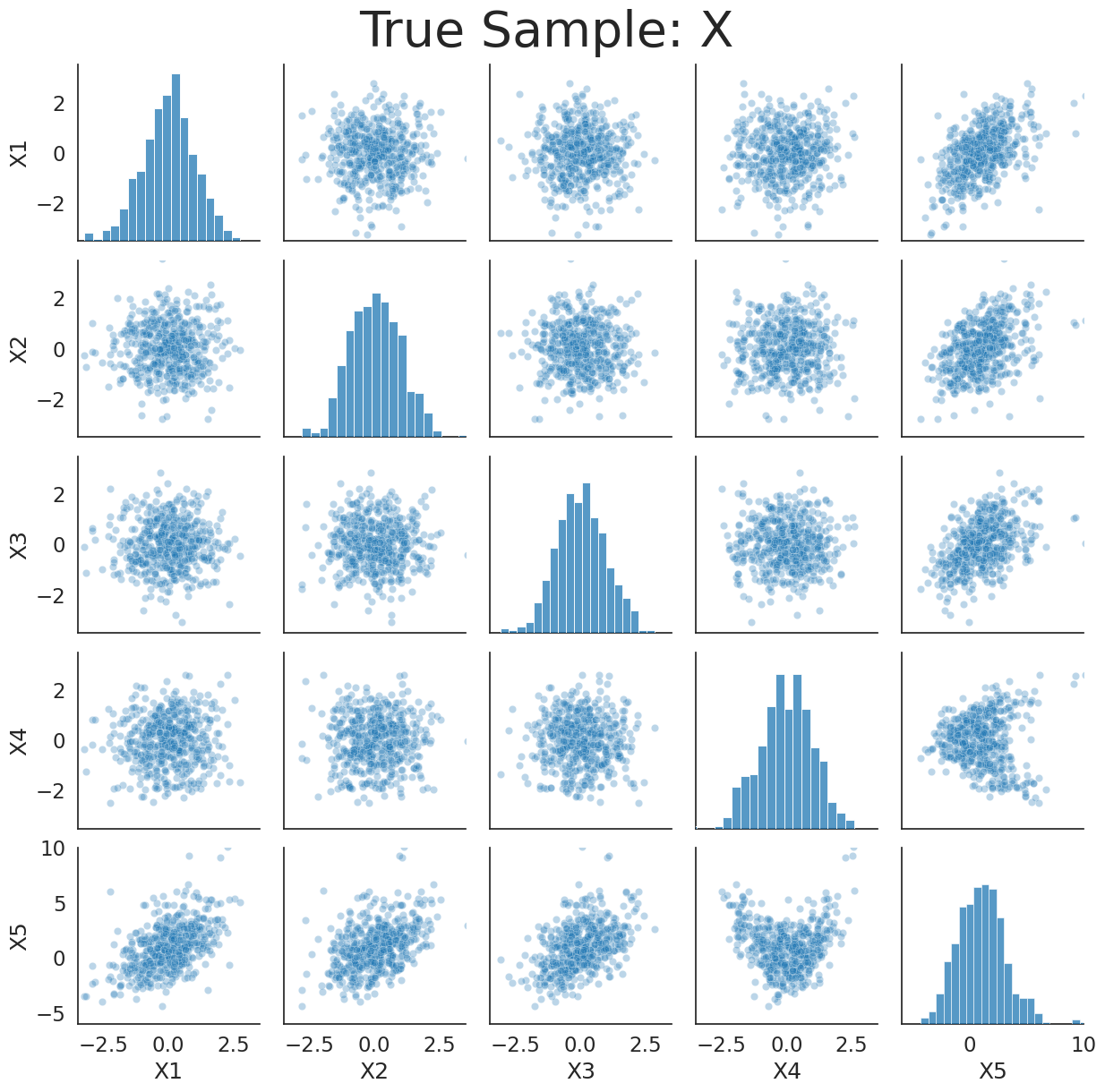} \includegraphics[width=0.232\textwidth]{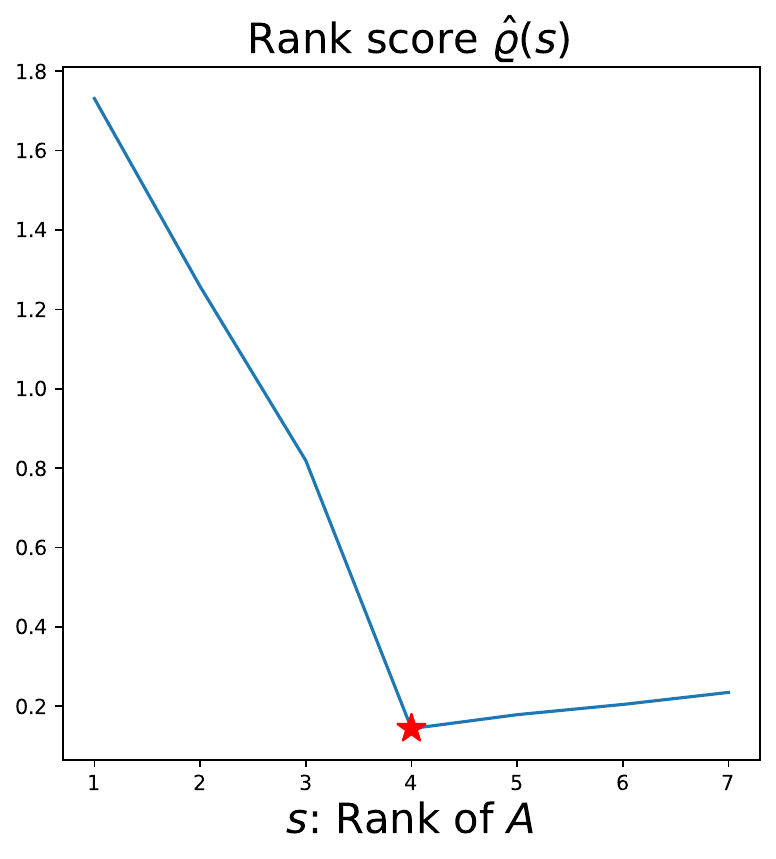}
\includegraphics[width=0.25\textwidth]{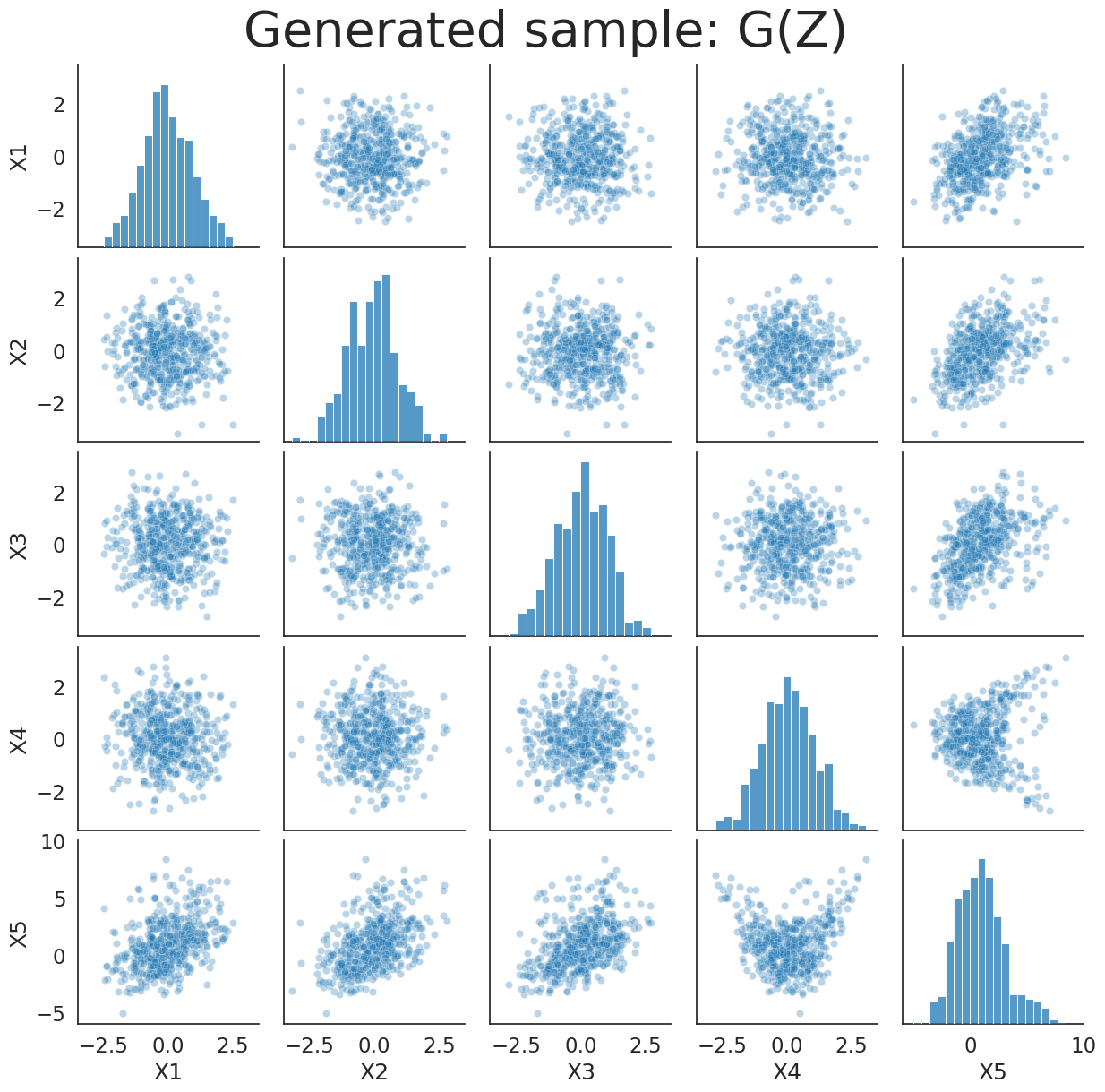}\includegraphics[width=0.25\textwidth]{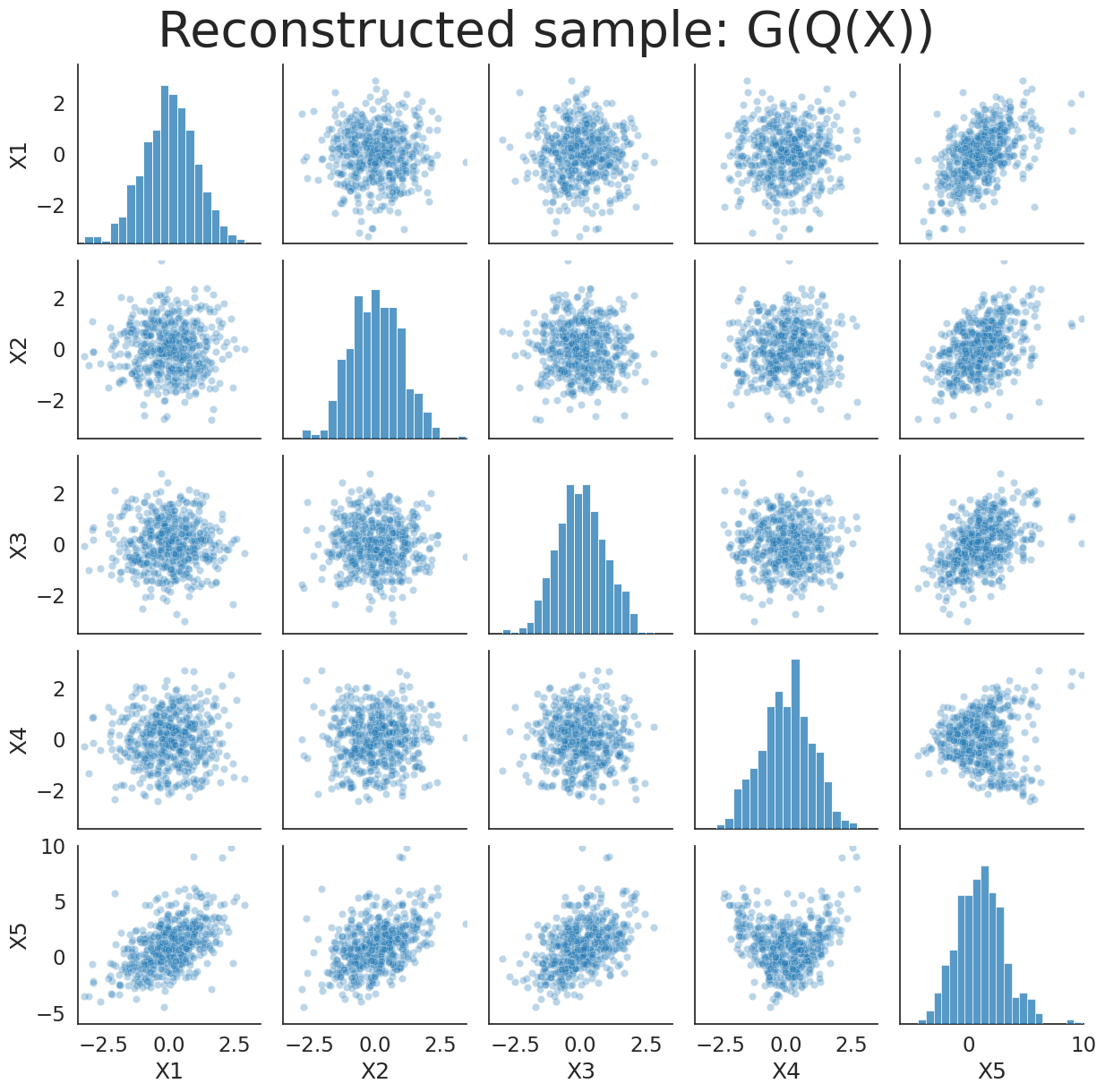}
\par\end{centering}
}
\par\end{centering}
\begin{centering}
\par\end{centering}
\caption{\protect\label{fig:toy_examples}Simulated data supported on manifolds
and the demonstrations of the fitted LWAGN models.}
\end{figure}

We then use Algorithm \ref{alg:lwgan} to estimate the encoder $Q$
and generator $G$ for each dataset. The gradient penalty parameter
is fixed to $\lambda_{\mathrm{GP}}=5$, and the rank regularization
parameter is chosen using the method introduced in Section \ref{subsec:tuning_parameter}.
After each model is trained to convergence, we compute the rank scores
$\hat{\varrho}_{n}(s)$ defined in (\ref{eq:rank_score}) for each
$s$, and their values are plotted in the second column of Figure
\ref{fig:toy_examples}. From the plots we can find that the minimizers
of $\hat{\varrho}_{n}(s)$ are consistent with the corresponding true
intrinsic dimensions, which validate that LWGAN can detect the manifold
dimensions of the data distributions. In Section S2.4 of the supplementary
material, we also design a bootstrap-type experiment to quantify the
uncertainty of the estimation results.

In addition, the third and fourth columns of Figure \ref{fig:toy_examples}
demonstrate the model-generated points $G(Z)\equiv G(AZ_{0})$ and
auto-encoder-reconstructed data $G(Q(X))$, respectively. Clearly,
all of the plots show a high quality of the generated distribution
$P_{G(Z)}$ and a small reconstruction error $\Vert X-G(Q(X))\Vert$.

\subsection{MNIST}

MNIST \citep{lecun1998gradient} is a large dataset of handwritten
0-9 digits commonly used for training various image processing systems.
The training set of MNIST contains 60,000 images, each consisting
of $28\times28$ grey-scale pixels. It was shown that different digits
have different intrinsic dimensions \citep{costa2006determining},
so the distribution of MNIST data may be supported on several disconnected
manifolds with various intrinsic dimensions.

We first train models on digits 1 and 2 separately using a 16-dimensional
latent variable, and the gradient penalty parameter is fixed to $\lambda_{\mathrm{GP}}=5$.
The true sample, estimated rank scores, generated sample, and reconstructed
sample for each digit are given in Figure \ref{fig:mnist_12}. The
rank score plots show that our estimation of the intrinsic dimension
of digit 1 is 8, whereas the estimation of digit 2 is 12. These estimates
are consistent with those of \citet{costa2006determining}, which
states that digit 1 exhibits a dimension estimate between 9 and 10,
and digit 2 has a dimension estimate between 12 and 14.

\begin{figure}[h]
\begin{centering}
\includegraphics[width=0.23\textwidth]{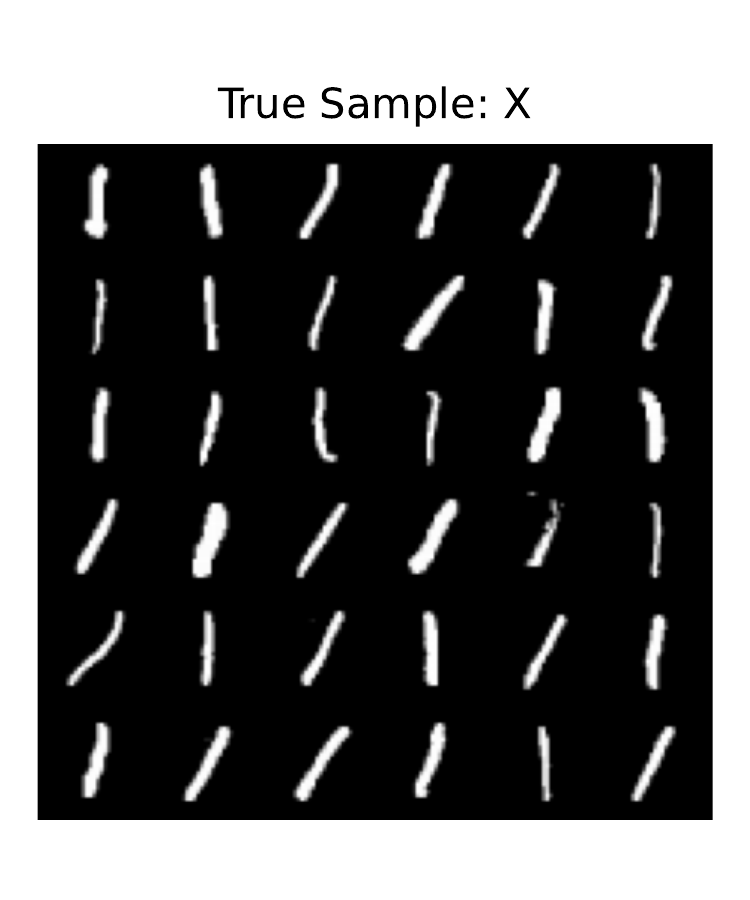} \includegraphics[width=0.23\textwidth]{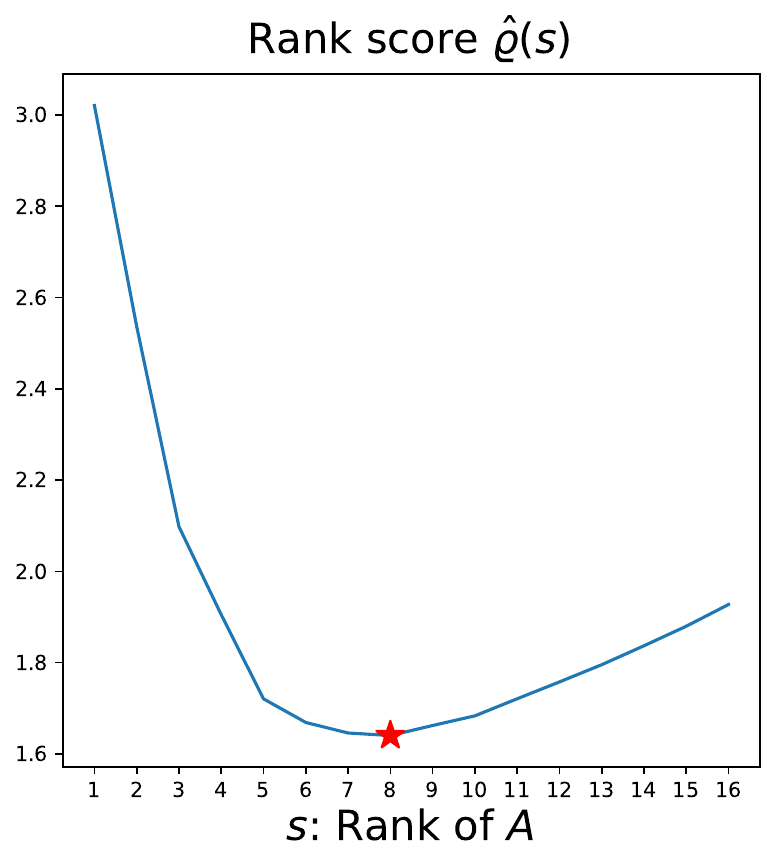}
\includegraphics[width=0.23\textwidth]{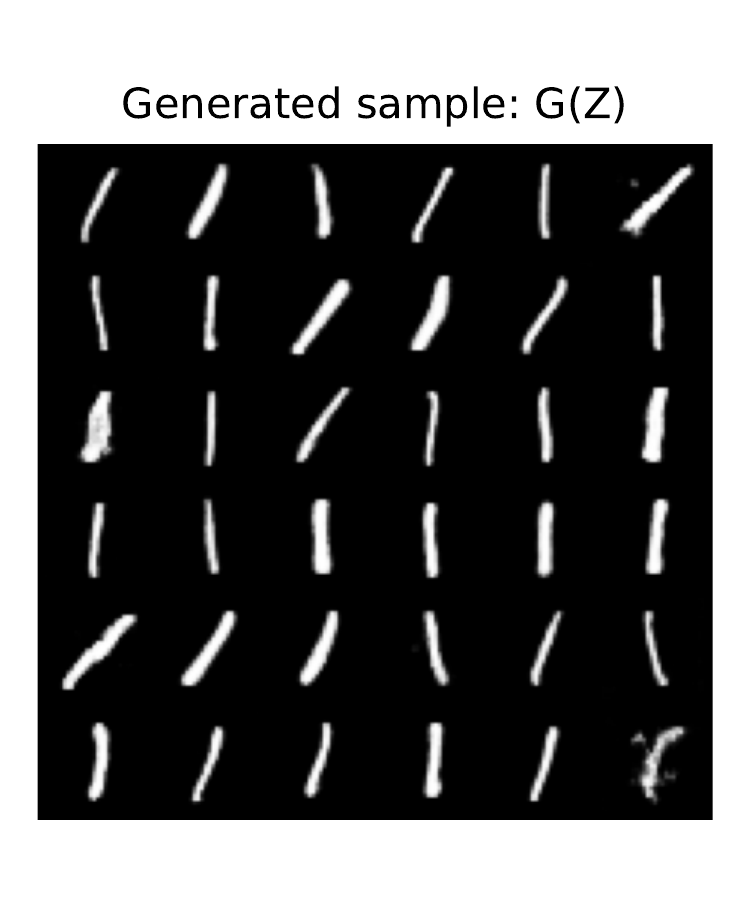} \includegraphics[width=0.23\textwidth]{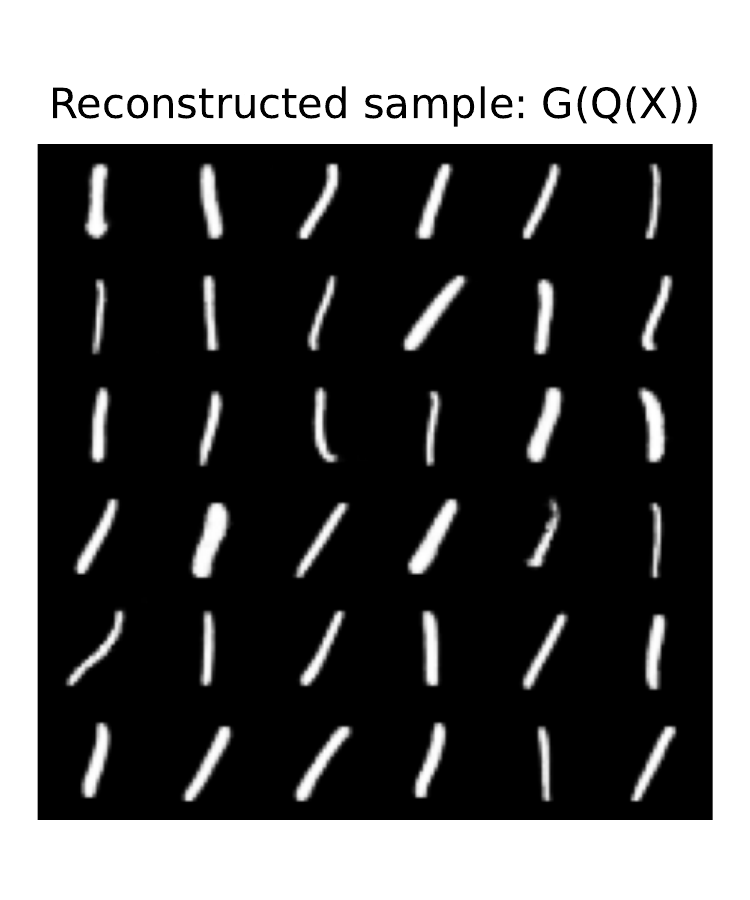}
\par\end{centering}
\begin{centering}
\includegraphics[width=0.23\textwidth]{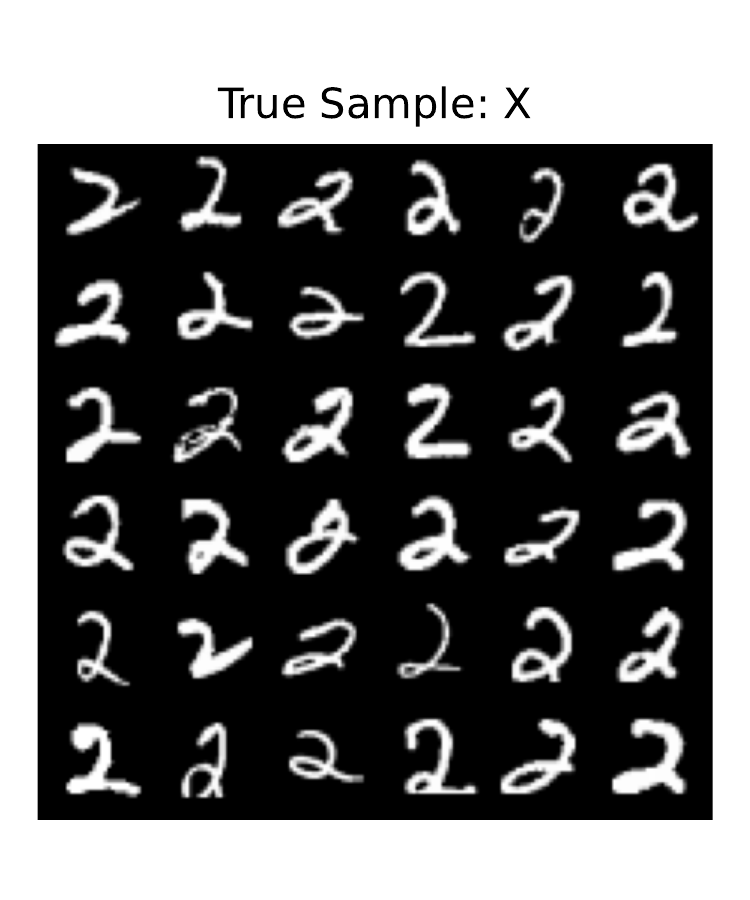} \includegraphics[width=0.23\textwidth]{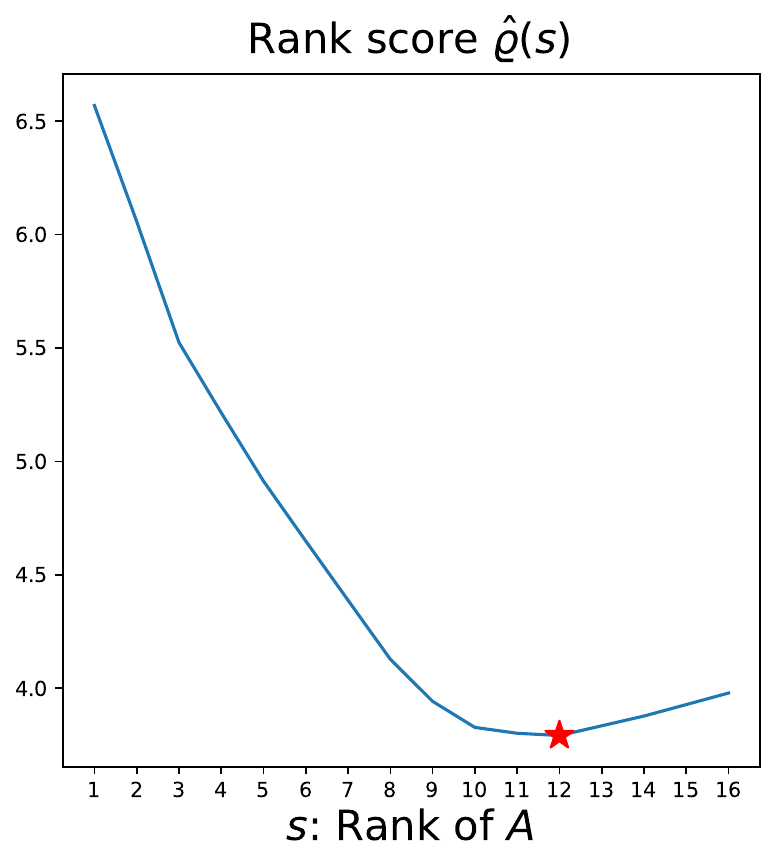}
\includegraphics[width=0.23\textwidth]{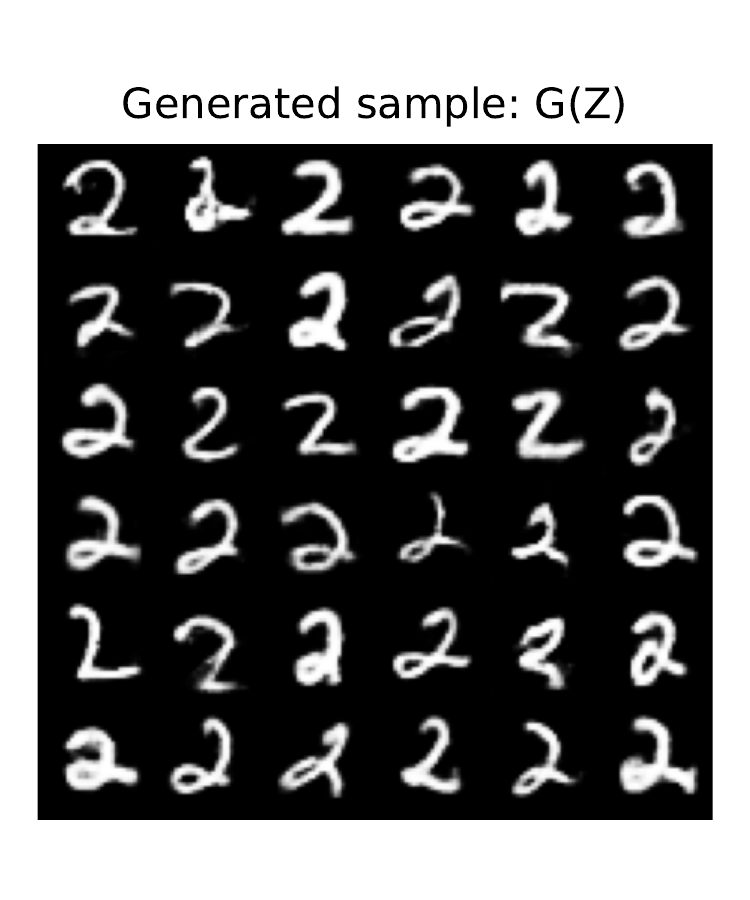} \includegraphics[width=0.23\textwidth]{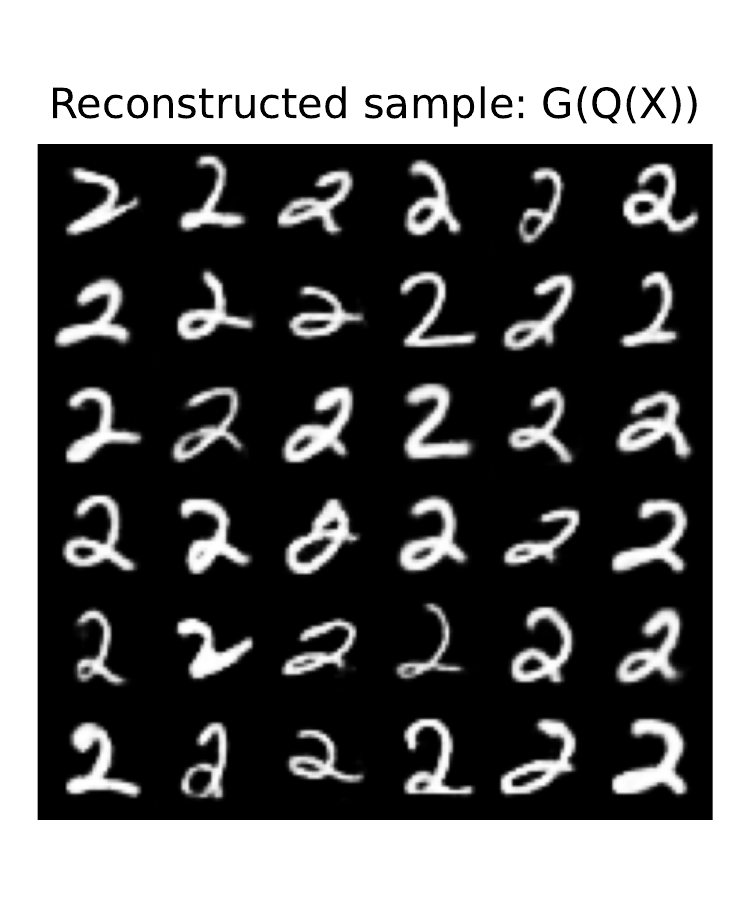}
\par\end{centering}
\caption{\protect\label{fig:mnist_12}Digits 1 (top row) and 2 (bottom row)
of the MNIST data, and the demonstrations of the fitted LWAGN models.}
\end{figure}

\begin{figure}[h]
\begin{centering}
\includegraphics[width=0.235\textwidth]{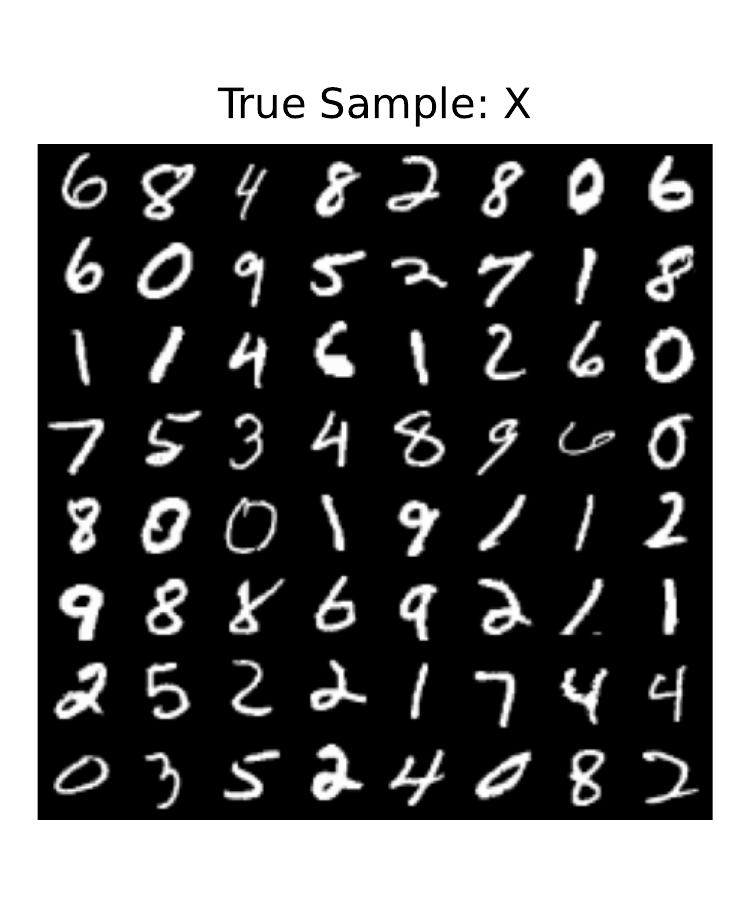}\hspace{0.75em}\includegraphics[width=0.227\textwidth]{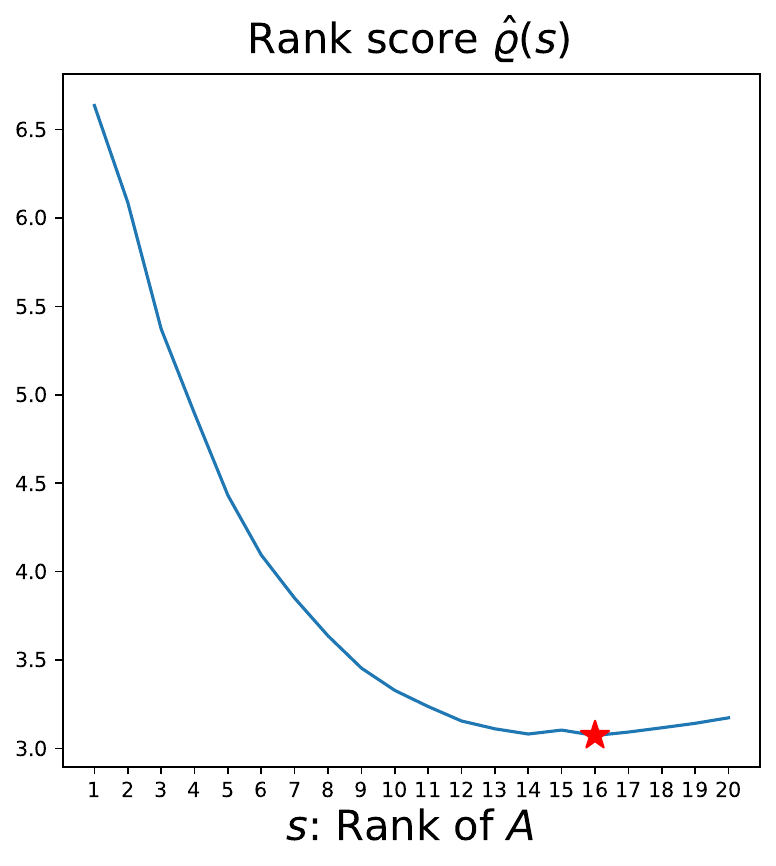}\hspace{0.75em}\includegraphics[width=0.235\textwidth]{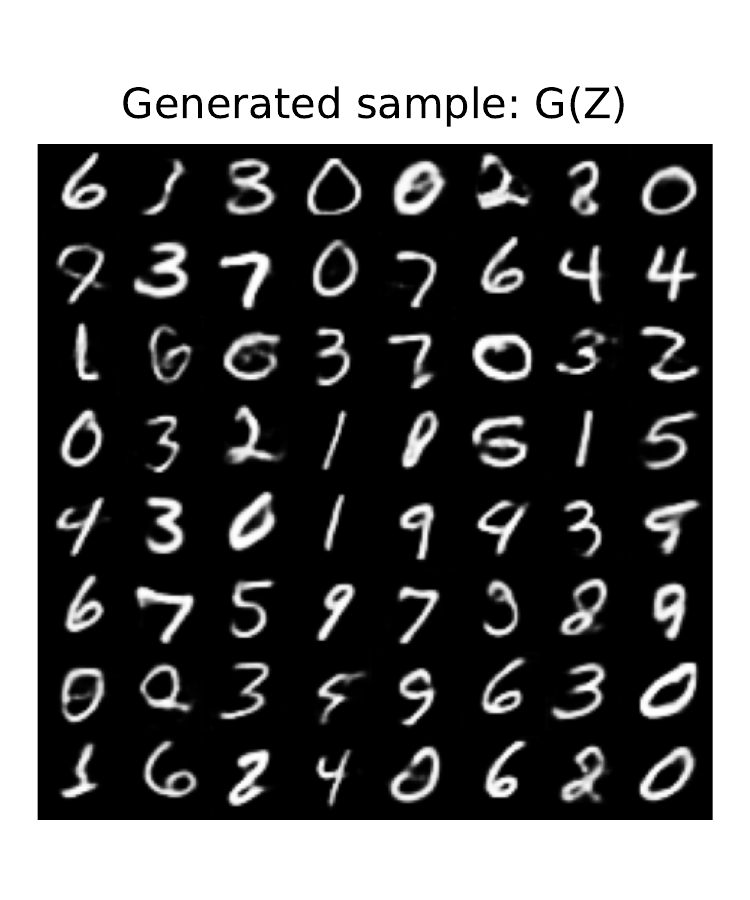}\enskip{}
\par\end{centering}
\begin{centering}
\includegraphics[width=0.235\textwidth]{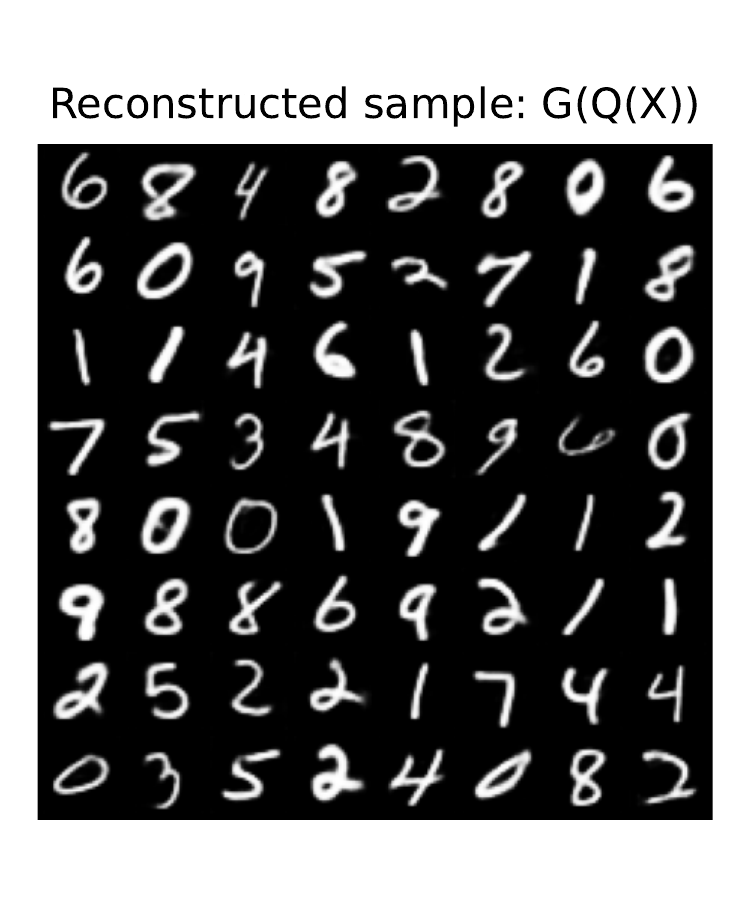}\includegraphics[width=0.51\textwidth]{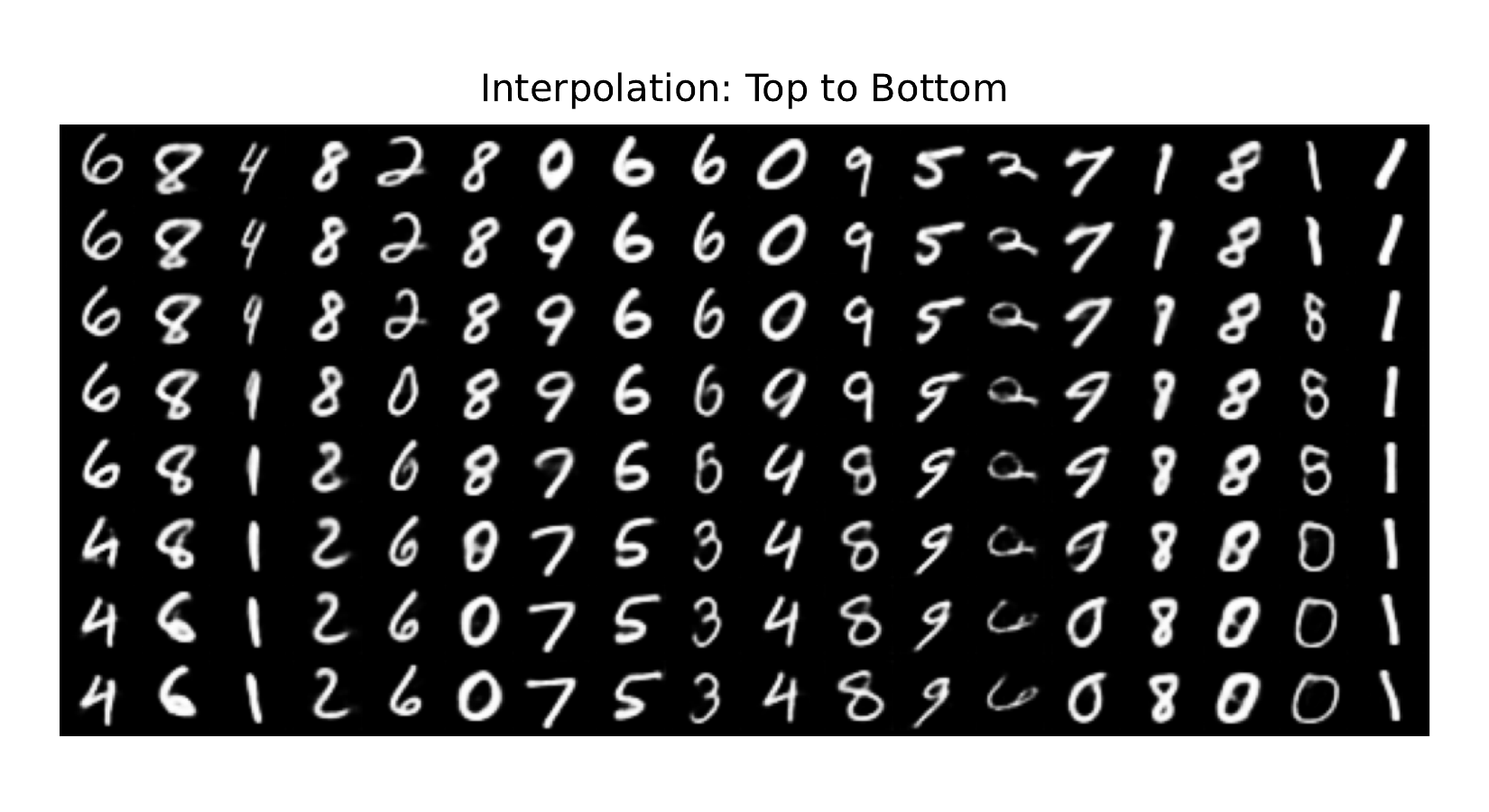}
\par\end{centering}
\caption{\protect\label{fig:mnist_all}MNIST data with all digits and the demonstrations
of the fitted LWAGN model.}
\end{figure}

We further estimate the intrinsic dimension of all digits from MNIST,
using a similar training scheme and parameter setting, except that
the maximum latent dimension is set to 20. The results for the common
tasks same as above are shown in Figure \ref{fig:mnist_all}, which
suggest that the intrinsic dimension of all digits is around 16. Moreover,
we also test the interpolation between two digits in the latent space.
In particular, we sample pairs of testing images $x_{1}$ and $x_{2}$,
and project them onto the latent space using the encoder $Q$, obtaining
latent representations $z_{1}=Q(x_{1})$ and $z_{2}=Q(x_{2})$. We
then linearly interpolate between $z_{1}$ and $z_{2}$, and pass
the intermediary points through the generator $G$ to visualize the
observation-space interpolations. The results are also displayed in
Figure \ref{fig:mnist_all}, which suggest that our model can get
rid of mode collapsing issues.

\subsection{CelebA}

CelebA \citep{liu2015deep} is another benchmark dataset for training
models to generate synthetic images. It is a large-scale face attributes
dataset with 202,599 color celebrity face images, which cover large
pose variations. We preprocess the data by detecting the bounding
box of face region in each image, cropping images to the bounding
boxes, and resizing each image to $64\times64$ pixels. The preprocessing
step has the effect of aligning the face region of each image, after
which we obtain a sample of 16,055 aligned face images. A demonstration
of the preprocessed CelebA images is shown in Figure \ref{fig:celeba_rank}(a).

\begin{figure}
\begin{centering}
\subfloat[True sample]{\begin{centering}
\includegraphics[width=0.405\textwidth]{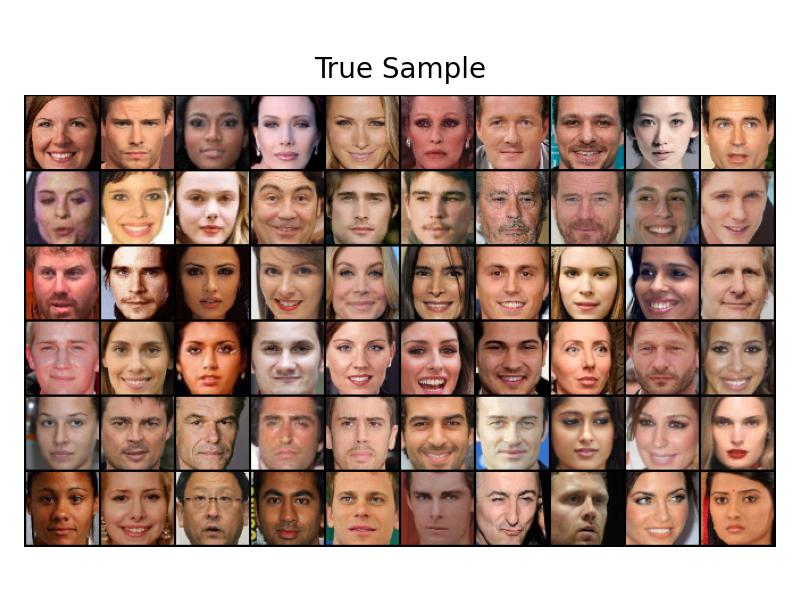}
\par\end{centering}
} \subfloat[Rank scores]{\begin{centering}
\includegraphics[width=0.25\textwidth]{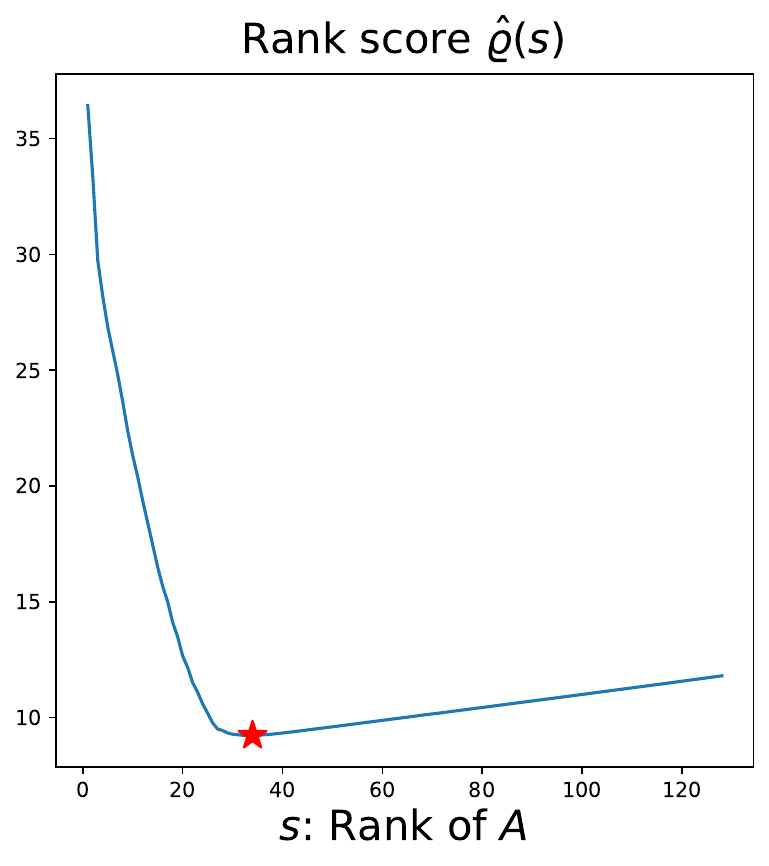}
\par\end{centering}
}
\par\end{centering}
\caption{\protect\label{fig:celeba_rank}True sample of the preprocessed CelebA
dataset and the rank score plot to estimate the intrinsic dimension.}
\end{figure}

We train CelebA using a latent dimension $d=128$, and the rank score
plot in Figure \ref{fig:celeba_rank}(b) shows that the estimated
intrinsic dimension is 34. We then compare LWGAN with other generative
models including WGAN, WAE, and CycleGAN \citep{zhu2017unpaired}
both visually and numerically. In particular, the CycleGAN model introduces
a cycle consistency loss based on the $\ell_{1}$-norm to push $G(Q(X))\approx X$
and $Q(G(Z))\approx Z$.

The generated images from the four models are demonstrated in Figure
\ref{fig:celeba_gen}. For LWGAN, the images are generated as $G(A_{s}Z_{0})$,
$Z_{0}\sim N(0,I_{d})$, where we consider different ranks $s=16,34,128$.
The other three methods generate images as $G(Z)$, $Z\sim N(0,I_{d})$.
We show the reconstructed images $G(Q(X))$ in Figure \ref{fig:celeba_recon},
and demonstrate the interpolation results in Figure \ref{fig:celeba_interp}.
For these two tasks we exclude WGAN, since it does not have an encoder.

\begin{figure}
\begin{centering}
\includegraphics[width=0.25\textwidth]{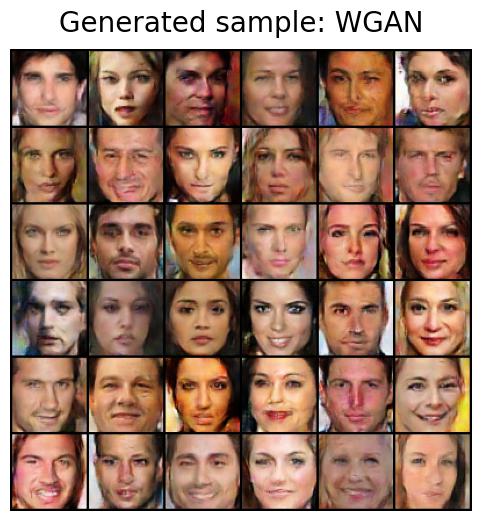} \includegraphics[width=0.25\textwidth]{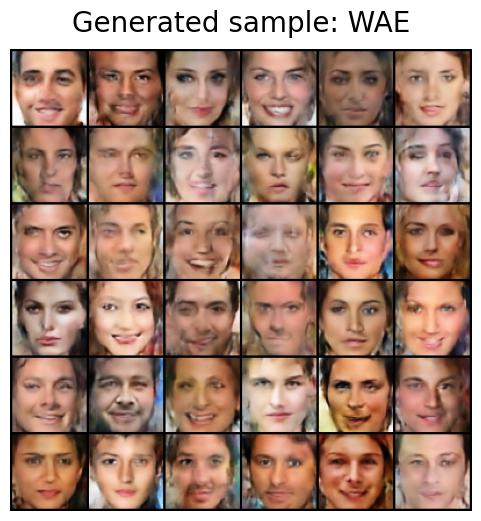}
\includegraphics[width=0.25\textwidth]{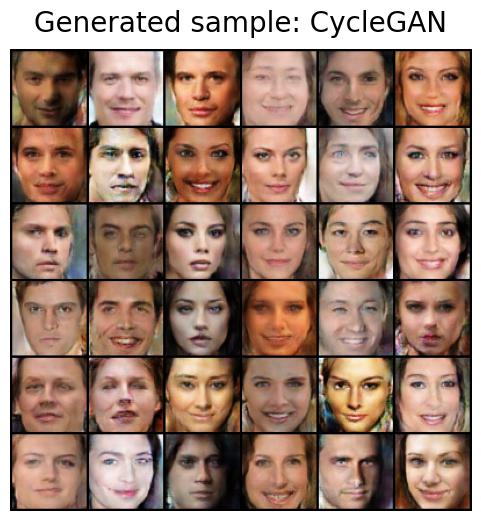}
\par\end{centering}
\begin{centering}
\includegraphics[width=0.25\textwidth]{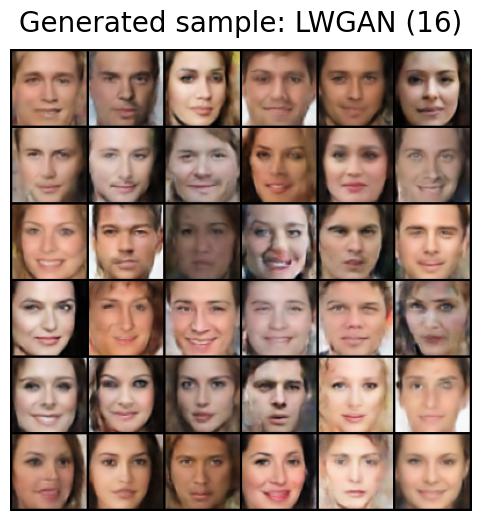}
\includegraphics[width=0.25\textwidth]{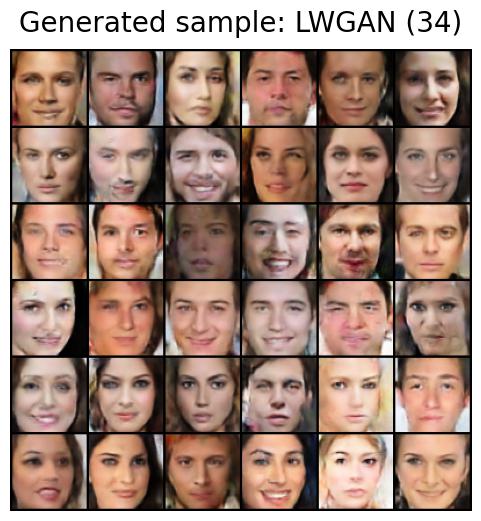}
\includegraphics[width=0.25\textwidth]{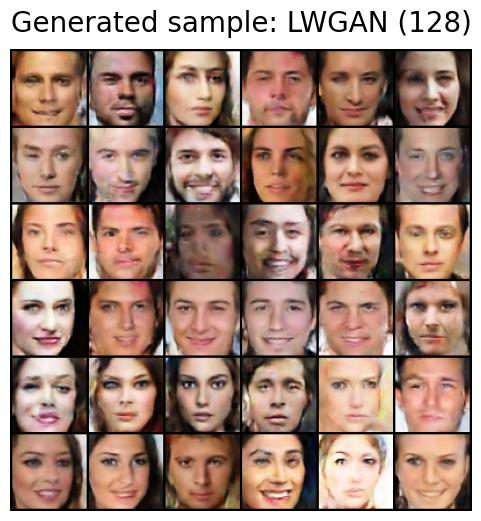}
\par\end{centering}
\caption{\protect\label{fig:celeba_gen}Generated images of WGAN, WAE, CycleGAN,
and LWGAN trained from the CelebA dataset.}
\end{figure}

\begin{figure}
\begin{centering}
\includegraphics[width=0.24\textwidth]{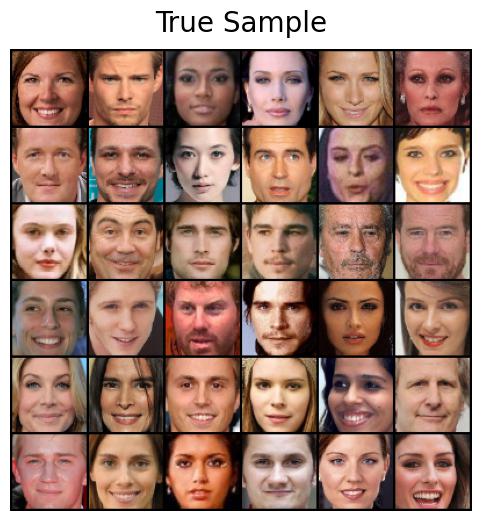} \includegraphics[width=0.24\textwidth]{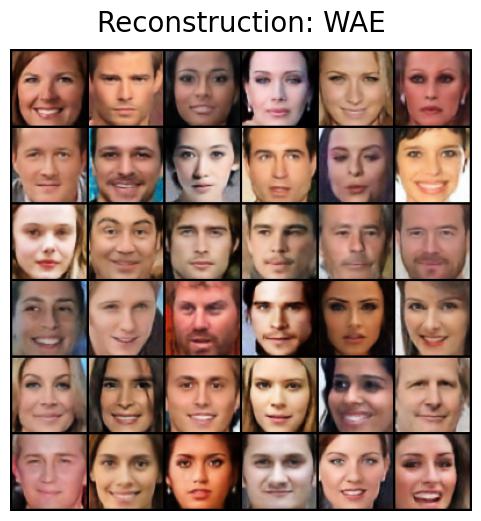}
\includegraphics[width=0.24\textwidth]{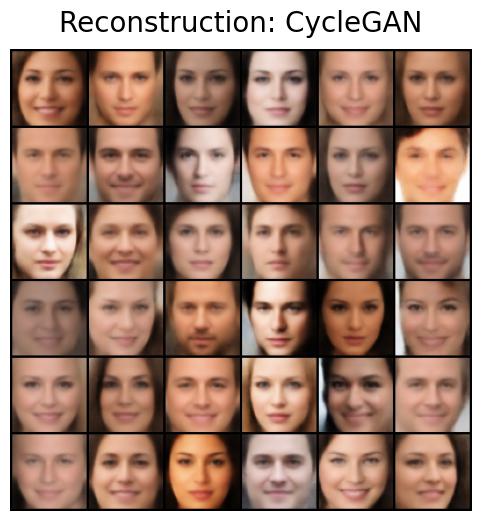}
\includegraphics[width=0.24\textwidth]{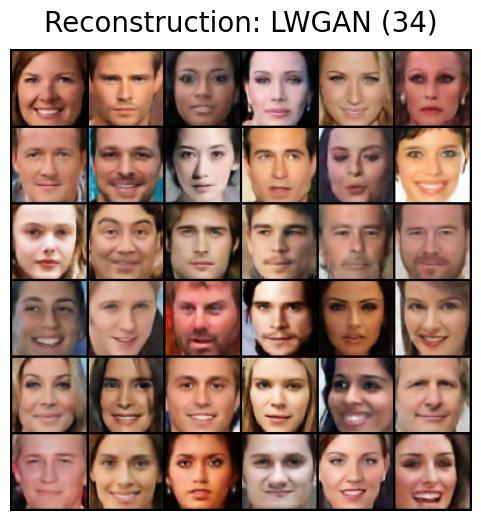}
\par\end{centering}
\caption{\protect\label{fig:celeba_recon}Reconstructed images of CelebA dataset.}
\end{figure}

\begin{figure}
\begin{centering}
\includegraphics[width=0.32\textwidth]{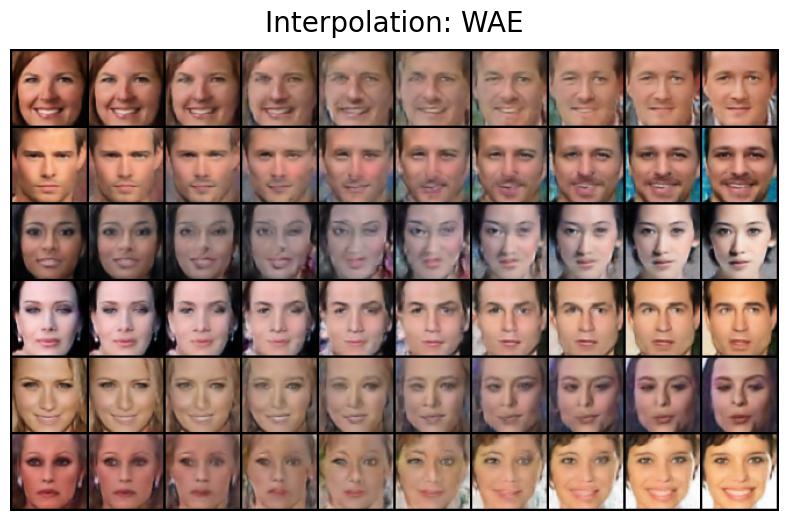}
\includegraphics[width=0.32\textwidth]{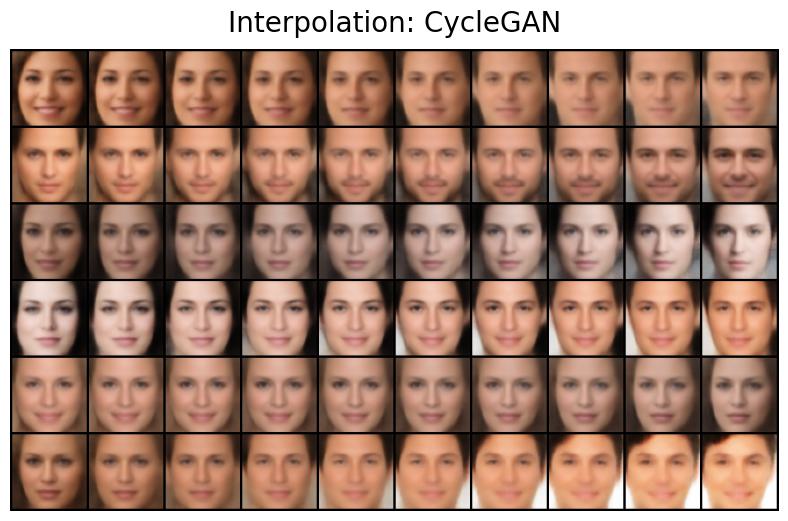}
\includegraphics[width=0.32\textwidth]{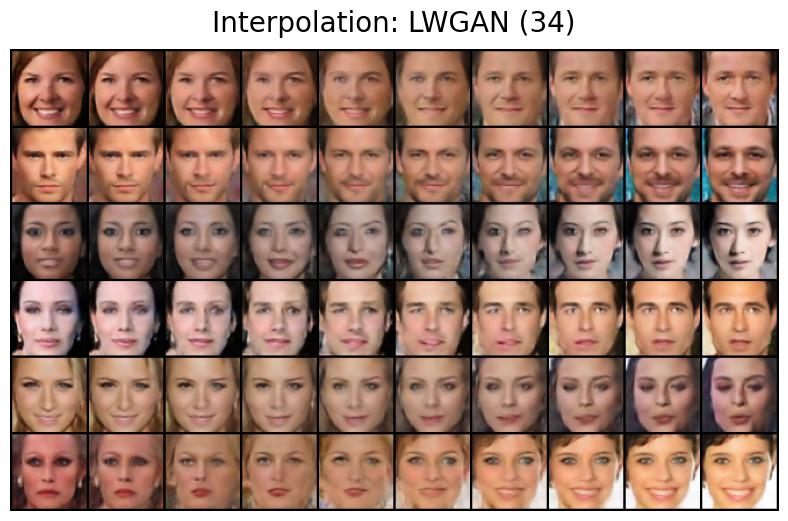}
\par\end{centering}
\caption{\protect\label{fig:celeba_interp}Interpolation of CelebA dataset.}
\end{figure}

Figures \ref{fig:celeba_gen} and \ref{fig:celeba_recon} show that
LWGAN is able to generate high-quality images as long as the rank
of $A_{s}$ is larger than or equal to the intrinsic dimension, and
an insufficient rank results in a low quality. This validates our
claims in Theorem \ref{thm:existence_qg} and Corollary \ref{cor:insufficient_dim}.
The generated images from the other three models have different levels
of blur and distortion, especially for WAE. In Figure \ref{fig:celeba_recon},
we find that WAE has a good reconstruction quality, so its low generation
quality may be due to the dimension mismatch between $P_{Q(X)}$ and
$P_{Z}$. On the other hand, CycleGAN has a better generation quality
than WAE, but it has a large reconstruction error. As a result, its
reconstructed images are blurry, and it also loses many details in
the interpolated images.

Finally, we numerically compare these methods with respect to three
metrics: the inception scores (IS, \citealp{salimans2016improved}),
the Fr\'echet inception distances (FID, \citealp{heusel2017gans}),
and the reconstruction errors. IS uses a pre-trained Inception-v3
model to predict the class probabilities for each generated image,
and FID improves IS by directly comparing the statistics of generated
samples to real samples. For IS, higher scores are better, and for
FID, lower is better. The reconstruction error is used to evaluate
whether the model generates meaningful latent codes and has the capacity
to recover the original information. The detailed descriptions of
these three metrics are provided in Section S2.2 of the supplementary
material.

Table \ref{tab:compare_metric} shows the values of these metrics
on each trained model. The numerical results are consistent with our
qualitative findings in Figure \ref{fig:celeba_gen} to Figure \ref{fig:celeba_interp}.
Specifically, WGAN and LWGAN have relatively higher generation quality
than the other two models, measured by IS and FID. WAE has a small
reconstruction error, but its generation quality is low. On the contrary,
CycleGAN has moderate generation quality but large reconstruction
errors. For LWGAN, an insufficient rank $s$ results in poor generation
and reconstruction quality, but models with ranks larger than $\hat{r}=34$
have good overall performance. We can also find that with the estimated
rank $s=\hat{r}=34$, LWGAN can achieve similar performance as the
case of $s=d=128$, but choosing $s$ to be the intrinsic dimension
can greatly reduce the model complexity without sacrificing the model
accuracy. Overall, the proposed LWGAN is able to produce meaningful
latent code and generate high-quality images at the same time, and
it is the only one among all the methods compared that is capable
of detecting the intrinsic dimension of data distributions.

\begin{table}
\caption{\protect\label{tab:compare_metric}Numerical comparison of LWGAN,
CycleGAN, WAE, and WGAN. The values in the parentheses are standard
deviations.}

\begin{centering}
\begin{tabular}{cccc}
\toprule 
Methods & IS $\uparrow$ & FID $\downarrow$ & Reconstruction error $\downarrow$\tabularnewline
\midrule
\midrule 
True & 2.07 (0.04) & 2.77 & --\tabularnewline
\midrule 
LWGAN, $s=16$ & 1.62 (0.02) & 40.98 & 14.95 (3.59)\tabularnewline
\midrule 
LWGAN, $s=\hat{r}=34$ & 1.66 (0.03) & 32.79 & 8.19 (1.54)\tabularnewline
\midrule 
LWGAN, $s=64$ & 1.70 (0.03) & \textbf{31.21} & 8.15 (1.54)\tabularnewline
\midrule 
LWGAN, $s=128$ & \textbf{1.71 (0.03)} & 31.56 & 8.15 (1.54)\tabularnewline
\midrule 
CycleGAN & 1.54 (0.02) & 42.76 & 20.73 (4.40)\tabularnewline
\midrule 
WAE & 1.59 (0.04) & 51.10 & \textbf{7.53 (1.35)}\tabularnewline
\midrule 
WGAN & 1.50 (0.03) & 31.60 & --\tabularnewline
\bottomrule
\end{tabular}
\par\end{centering}
\end{table}

\section{Conclusion}

\label{sec:conclusion}

We have developed a novel LWGAN framework that enables us to adaptively
learn the intrinsic dimension of data distributions supported on manifolds.
This framework fuses WAE and WGAN in a principled way, so that the
model learns a latent normal distribution whose rank is consistent
with the dimension of the data manifold. We have provided theoretical
guarantees on the generalization error bound, estimation consistency,
and dimension consistency of LWGAN. Numerical experiments have shown
that the intrinsic dimension of the data can be successfully detected
under several settings on both synthetic datasets and benchmark datasets,
and the model-generated samples are of high quality.

A potential future direction of LWGAN is to investigate a more general
scenario where the generator $G$ is stochastic. This can be achieved
by adding an extra noise vector to the input of $G$. In addition,
it is interesting to incorporate the stochastic LWGAN into some more
recent GAN modules such as BigGAN \citep{brock2018large}, so that
high-resolution and high-fidelity images can be produced along with
the estimation of the intrinsic dimension.

The new LWGAN framework has many potential applications in other fields.
For example, LWGAN can be used for structural estimation, which is
a useful tool to quantify economic mechanisms and learn about the
effects of policies that are yet to be implemented \citep{wei2022estimating}.
An economic structural model specifies some outcome $g(x,\varepsilon;\theta)$
that depends on a set of observables $x$, unobservables $\varepsilon$,
and structural parameters $\theta$. The function $g$ can represent
a utility maximization problem or other observed outcomes. Under many
scenarios, the likelihood function and moment functions are not easy
to obtain. This makes the maximum likelihood estimator and generalized
method of moments infeasible, and other simulation-based methods can
cause additional computational burden. By training LWGAN on the data
from $(x,y)$, we are able to adaptively learn the data representation
by the encoder $Q$, instead of using moments. At the same time, we
are able to boost the sample size by the generator $G$. By comparing
the generated data $(x,g(x,\varepsilon;\theta))$ and the observed
data $(x,y)$ in the latent space, we can estimate $\theta$ efficiently.

\setstretch{1.2}

\bibliographystyle{apalike}
\bibliography{ref}

\begin{thebibliography}{}

\bibitem[Arjovsky et~al., 2017]{arjovsky2017wasserstein}
Arjovsky, M., Chintala, S., and Bottou, L. (2017).
\newblock Wasserstein generative adversarial networks.
\newblock In {\em International Conference on Machine Learning}, pages
  214--223.

\bibitem[Arora et~al., 2017]{arora2017generalization}
Arora, S., Ge, R., Liang, Y., Ma, T., and Zhang, Y. (2017).
\newblock Generalization and equilibrium in generative adversarial nets
  ({GANs}).
\newblock In {\em International Conference on Machine Learning}, pages
  224--232.

\bibitem[Brock et~al., 2019]{brock2018large}
Brock, A., Donahue, J., and Simonyan, K. (2019).
\newblock Large scale {GAN} training for high fidelity natural image synthesis.
\newblock In {\em International Conference on Learning Representations}.

\bibitem[Chen et~al., 2021]{chenyao21}
Chen, Y., Gao, Q., and Wang, X. (2021).
\newblock Inferential {Wasserstein} generative adversarial networks.
\newblock {\em Journal of the Royal Statistical Society, Series B}.

\bibitem[Costa and Hero, 2006]{costa2006determining}
Costa, J.~A. and Hero, A.~O. (2006).
\newblock Determining intrinsic dimension and entropy of high-dimensional shape
  spaces.
\newblock In {\em Statistics and Analysis of Shapes}, pages 231--252. Springer.

\bibitem[Dinh et~al., 2016]{dinh2016density}
Dinh, L., Sohl-Dickstein, J., and Bengio, S. (2016).
\newblock Density estimation using real {NVP}.
\newblock {\em arXiv preprint arXiv:1605.08803}.

\bibitem[Donahue et~al., 2017]{donahue2016adversarial}
Donahue, J., Kr{\"a}henb{\"u}hl, P., and Darrell, T. (2017).
\newblock Adversarial feature learning.
\newblock In {\em International Conference on Learning Representations}.

\bibitem[Dumoulin et~al., 2017]{dumoulin2017adversarially}
Dumoulin, V., Belghazi, I., Poole, B., Mastropietro, O., Lamb, A., Arjovsky,
  M., and Courville, A. (2017).
\newblock Adversarially learned inference.
\newblock In {\em International Conference on Learning Representations}.

\bibitem[Gao and Wang, 2021]{gao2021theoretical}
Gao, Q. and Wang, X. (2021).
\newblock Theoretical investigation of generalization bounds for adversarial
  learning of deep neural networks.
\newblock {\em Journal of Statistical Theory and Practice}, 15(2):1--28.

\bibitem[Gao et~al., 2020]{gao2020flow}
Gao, R., Nijkamp, E., Kingma, D.~P., Xu, Z., Dai, A.~M., and Wu, Y.~N. (2020).
\newblock Flow contrastive estimation of energy-based models.
\newblock In {\em IEEE/CVF Conference on Computer Vision and Pattern
  Recognition}, pages 7518--7528.

\bibitem[Goodfellow et~al., 2014]{goodfellow2014generative}
Goodfellow, I., Pouget-Abadie, J., Mirza, M., Xu, B., Warde-Farley, D., Ozair,
  S., Courville, A., and Bengio, Y. (2014).
\newblock Generative adversarial nets.
\newblock In {\em Advances in Neural Information Processing Systems}, pages
  2672--2680.

\bibitem[Gulrajani et~al., 2017]{gulrajani2017improved}
Gulrajani, I., Ahmed, F., Arjovsky, M., Dumoulin, V., and Courville, A.~C.
  (2017).
\newblock Improved training of {Wasserstein} {GANs}.
\newblock In {\em Advances in Neural Information Processing Systems}, pages
  5767--5777.

\bibitem[Heusel et~al., 2017]{heusel2017gans}
Heusel, M., Ramsauer, H., Unterthiner, T., Nessler, B., and Hochreiter, S.
  (2017).
\newblock {GANs} trained by a two time-scale update rule converge to a local
  nash equilibrium.
\newblock In {\em Advances in Neural Information Processing Systems}, pages
  6626--6637.

\bibitem[Kingma and Ba, 2014]{kingma2014adam}
Kingma, D.~P. and Ba, J. (2014).
\newblock Adam: A method for stochastic optimization.
\newblock {\em arXiv preprint arXiv:1412.6980}.

\bibitem[Kingma and Welling, 2014]{kingma2014auto}
Kingma, D.~P. and Welling, M. (2014).
\newblock Auto-encoding variational {Bayes}.
\newblock In {\em International Conference on Learning Representations}.

\bibitem[Larsen et~al., 2016]{larsen2016autoencoding}
Larsen, A. B.~L., S{\o}nderby, S.~K., Larochelle, H., and Winther, O. (2016).
\newblock Autoencoding beyond pixels using a learned similarity metric.
\newblock In {\em International Conference on Machine Learning}, pages
  1558--1566.

\bibitem[Laurent and Massart, 2000]{laurent2000adaptive}
Laurent, B. and Massart, P. (2000).
\newblock Adaptive estimation of a quadratic functional by model selection.
\newblock {\em Annals of Statistics}, pages 1302--1338.

\bibitem[LeCun et~al., 1998]{lecun1998gradient}
LeCun, Y., Bottou, L., Bengio, Y., and Haffner, P. (1998).
\newblock Gradient-based learning applied to document recognition.
\newblock {\em Proceedings of the IEEE}, 86(11):2278--2324.

\bibitem[Lee, 2013]{lee2013introduction}
Lee, J.~M. (2013).
\newblock Introduction to smooth manifolds.

\bibitem[Li et~al., 2015]{li2015generative}
Li, Y., Swersky, K., and Zemel, R. (2015).
\newblock Generative moment matching networks.
\newblock In {\em International Conference on Machine Learning}, pages
  1718--1727.

\bibitem[Liu et~al., 2015]{liu2015deep}
Liu, Z., Luo, P., Wang, X., and Tang, X. (2015).
\newblock Deep learning face attributes in the wild.
\newblock In {\em Proceedings of the IEEE international conference on computer
  vision}, pages 3730--3738.

\bibitem[Meitz, 2024]{meitz2021statistical}
Meitz, M. (2024).
\newblock Statistical inference for generative adversarial networks and other
  minimax problems.
\newblock {\em Scandinavian Journal of Statistics}.

\bibitem[Mohri et~al., 2018]{mohri2018foundations}
Mohri, M., Rostamizadeh, A., and Talwalkar, A. (2018).
\newblock {\em Foundations of machine learning}.
\newblock MIT press.

\bibitem[Qiu and Wang, 2021]{qiu2021almond}
Qiu, Y. and Wang, X. (2021).
\newblock {ALMOND}: Adaptive latent modeling and optimization via neural
  networks and {Langevin} diffusion.
\newblock {\em Journal of the American Statistical Association},
  116(535):1224--1236.

\bibitem[Rockafellar and Wets, 2009]{rockafellar2009variational}
Rockafellar, R.~T. and Wets, R. J.-B. (2009).
\newblock {\em Variational analysis}.
\newblock Springer Science \& Business Media.

\bibitem[Rubenstein et~al., 2018]{rubenstein2018latent}
Rubenstein, P.~K., Schoelkopf, B., and Tolstikhin, I. (2018).
\newblock On the latent space of {Wasserstein} auto-encoders.
\newblock {\em arXiv preprint arXiv:1802.03761}.

\bibitem[Salimans et~al., 2016]{salimans2016improved}
Salimans, T., Goodfellow, I., Zaremba, W., Cheung, V., Radford, A., and Chen,
  X. (2016).
\newblock Improved techniques for training {GANs}.
\newblock In {\em Advances in neural information processing systems}, pages
  2234--2242.

\bibitem[Tolstikhin et~al., 2018]{tolstikhin2018wasserstein}
Tolstikhin, I., Bousquet, O., Gelly, S., and Schoelkopf, B. (2018).
\newblock Wasserstein auto-encoders.
\newblock In {\em International Conference on Learning Representations}.

\bibitem[van~der Vaart, 1998]{vaart_1998}
van~der Vaart, A.~W. (1998).
\newblock {\em Asymptotic Statistics}.
\newblock Cambridge University Press.

\bibitem[Villani, 2008]{villani2008optimal}
Villani, C. (2008).
\newblock {\em Optimal transport: old and new}.
\newblock Springer Science \& Business Media.

\bibitem[Wei and Jiang, 2022]{wei2022estimating}
Wei, Y. and Jiang, Z. (2022).
\newblock Estimating parameters of structural models using neural networks.
\newblock {\em USC Marshall School of Business Research Paper}.

\bibitem[Zhang et~al., 2023]{zhang2021flow}
Zhang, M., Sun, Y., Zhang, C., and Mcdonagh, S. (2023).
\newblock Spread flows for manifold modelling.
\newblock In {\em International Conference on Artificial Intelligence and
  Statistics}, pages 11435--11456.

\bibitem[Zhu et~al., 2017]{zhu2017unpaired}
Zhu, J.~Y., Park, T., Isola, P., and Efros, A.~A. (2017).
\newblock Unpaired image-to-image translation using cycle-consistent
  adversarial networks.
\newblock In {\em IEEE International Conference on Computer Vision}, pages
  2223--2232.

\end{thebibliography}

\setstretch{1.8}

\appendix

\section{Proof of Theorems}

\subsection{Proof of Theorem \ref{thm:existence_qg}}

Let $\tilde{X}=\varphi(X)=(\tilde{X}_{1},\ldots,\tilde{X}_{r})^{T}$,
and then by Definition \ref{def:distr_manifold}, $\tilde{X}$ is
a continuous random vector on $\mathbb{R}^{r}$. We then seek a mapping
$Q$ such that the transformed variable $Q(\tilde{X})$ follows the
standard multivariate normal distribution $N(0,I_{r})$.

Denote the marginal c.d.f.'s of $\tilde{X}$ as $F_{i}(x)=\mathbb{P}(\tilde{X}_{i}\le x)$,
$i=1,\ldots,r$. By applying the probability integral transformation
to each component, the random vector
\[
Q_{1}(\tilde{X})=\left(F_{1}(\tilde{X}_{1}),\ldots,F_{r}(\tilde{X}_{r})\right)\coloneqq(U_{1},\ldots,U_{r})
\]
has uniformly distributed marginals. Clearly, $Q_{1}$ has a continuous
inverse:
\[
Q_{1}^{-1}(U_{1},\ldots,U_{r})=\left(F_{1}^{-1}(U_{1}),\ldots,F_{r}^{-1}(U_{r})\right),
\]
indicating that $Q_{1}:\mathbb{R}^{r}\rightarrow\mathbb{R}^{r}$ is
a homeomorphism.

Let $C:[0,1]^{r}\rightarrow[0,1]$ be the copula of $\tilde{X}$,
which is defined as the joint c.d.f. of $\ensuremath{(U_{1},\ldots,U_{r})}$:
\[
C(u_{1},\ldots,u_{r})=\mathbb{P}\left(U_{1}\le u_{1},\ldots,U_{r}\le u_{r}\right).
\]
Accordingly, let $c(u_{1},\ldots,u_{r})=\partial^{r}C(u_{1},\ldots,u_{r})/\partial u_{1}\cdots\partial u_{r}$
be the copula density. The copula $C$ contains all information of
the dependence structure among the components of $\tilde{X}$, and
the joint c.d.f. of $\tilde{X}$ is $C\left(F_{1}(\tilde{x}_{1}),\ldots,F_{r}(\tilde{x}_{r})\right)$.
Denote the conditional c.d.f. of $U_{k}$ given $U_{1},\ldots,U_{k-1}$
by
\[
C_{k}(u_{k}|u_{<k})\coloneqq C_{k}(u_{k}|u_{1},\ldots,u_{k-1})=\mathbb{P}(U_{k}\le u_{k}|U_{1}=u_{1},\ldots,U_{k-1}=u_{k-1}),\quad k=2,\ldots,r,
\]
as well as the conditional density $c_{k}(u_{k}|u_{<k})=\partial C_{k}(u_{k}|u_{<k})/\partial u_{k}$.
Then clearly,
\[
c(u_{1},\ldots,u_{r})=c_{1}(u_{1})c_{2}(u_{2}|u_{<2})\cdots c_{r}(u_{r}|u_{<r}).
\]

Define the mapping $Q_{2}:\mathbb{R}^{r}\rightarrow\mathbb{R}^{r}$
as $Q_{2}(U_{1},\ldots,U_{r})=(\tilde{U}_{1},\ldots,\tilde{U}_{r})$,
where
\[
\begin{cases}
\tilde{U}_{1}=U_{1}\coloneqq C_{1}(U_{1}),\\
\tilde{U}_{k}=C_{k}(U_{k}|U_{<k}), & k=2,\ldots,r.
\end{cases}
\]
We can readily show that $\tilde{U}_{1},\ldots,\tilde{U}_{r}$ are
independent uniform random variables, since
\begin{align*}
\mathbb{P}\left(\tilde{U}_{1}\le\tilde{u}_{1},\ldots,\tilde{U}_{r}\le\tilde{u}_{r}\right) & =\int_{C_{1}(u_{1})\le\tilde{u}_{1}}\cdots\int_{C_{r}(u_{r}|u_{<r})\le\tilde{u}_{r}}c(u_{1},\ldots,u_{r})\mathrm{d}u_{1}\cdots\mathrm{d}u_{r}\\
 & =\int_{C_{1}(u_{1})\le\tilde{u}_{1}}\cdots\int_{C_{r}(u_{r}|u_{<r})\le\tilde{u}_{r}}\mathrm{d}C_{1}(u_{1})\cdots\mathrm{d}C_{r}(u_{r}|u_{<r})\\
 & =\int_{0}^{\tilde{u}_{1}}\cdots\int_{0}^{\tilde{u}_{r}}\mathrm{d}z_{1}\cdots\mathrm{d}z_{r}=\prod_{k=1}^{r}\tilde{u}_{k}.
\end{align*}
It is easy to verify that $Q_{2}$ is also a homeomorphism.

Next, let $Z=Q_{3}(\tilde{U}_{1},\ldots,\tilde{U}_{r})=(\Phi^{-1}(\tilde{U}_{1}),\ldots,\Phi^{-1}(\tilde{U}_{r}))$,
where $\Phi^{-1}$ is the inverse c.d.f. of the standard normal distribution,
and then $Z\sim N(0,I_{r})$. So by defining $Q=Q_{3}\circ Q_{2}\circ Q_{1}$,
we have $Z=Q(\tilde{X})\sim N(0,I_{r})$, and $Q$ is a homeomorphism.
Further let
\[
Z^{\diamond}=L_{r\rightarrow d}(Z)\coloneqq\begin{pmatrix}I_{r}\\
\mathbf{0}_{(d-r)\times r}
\end{pmatrix}Z
\]
and define $Q^{\diamond}:\mathcal{X}\rightarrow\mathcal{Z}$ as $Q^{\diamond}=L_{r\rightarrow d}\circ Q\circ\varphi$,
and then $Z^{\diamond}=Q^{\diamond}(X)\sim N(0,A_{r})$.

We can get $G^{\diamond}:\mathcal{Z}\rightarrow\mathcal{X}$ by reversing
the transformations above. First define $L_{d\rightarrow r}:\mathcal{Z}\rightarrow\mathbb{R}^{r}$
as
\[
L_{d\rightarrow r}(z)=\left(I_{r}\quad\mathbf{0}_{r\times(d-r)}\right)z,
\]
and then $Z=L_{d\rightarrow r}(Z^{\diamond})$. Since $Q$ is a homeomorphism,
$G=Q^{-1}$ must exist and is continuous, which implies that $\tilde{X}=G(Z)$.
Similarly, $\varphi$ is a homeomorphism by Assumption \ref{assu:homeomorphism},
so $\varphi^{-1}$ exists and is continuous, with $X=\varphi^{-1}(\tilde{X})$.
By defining $G^{\diamond}=\varphi^{-1}\circ G\circ L_{d\rightarrow r}$,
we have $X=G^{\diamond}(Z^{\diamond})=G^{\diamond}(Q^{\diamond}(X))$.

\subsection{Proof of Corollary \ref{cor:insufficient_dim}}

We first present the following useful lemma.
\begin{lem}
\label{lem:invariance_domain}Let $\mathcal{D}$ be an open subset
of $\mathbb{R}^{n}$, and $f:\mathcal{D}\rightarrow\mathbb{R}^{m}$
be a continuous mapping with $m<n$. Then $f$ cannot be injective,
i.e., there exist two points $x,y\in\mathcal{D}$, $x\neq y$, such
that $f(x)=f(y)$.
\end{lem}
\begin{proof}
Suppose that $f$ is injective, and then take $g:\mathcal{D}\rightarrow\mathbb{R}^{n}$
with $g(x)=(f(x),\mathbf{0}_{n-m})$. Clearly, $g$ is continuous
and injective, so by the invariance of domain theorem, we have that
$g(\mathcal{D})$ is open in $\mathbb{R}^{n}$, and $g$ is a homeomorphism
between $\mathcal{D}$ and $g(\mathcal{D})$. However, we have $g(\mathcal{D})=f(\mathcal{D})\times\{\mathbf{0}_{n-m}\}$,
so $g(\mathcal{D})$ cannot be open, which leads to a contradiction.
\end{proof}
We then prove this corollary by contradiction. Suppose that there
exist continuous mappings $Q$ and $G$ such that $\mathbb{E}_{X}\left\Vert X-G(Q(X))\right\Vert =0$.

As in the proof of Theorem \ref{thm:existence_qg}, let $\tilde{X}=\varphi(X)$,
and then by Definition \ref{def:distr_manifold}, we have
\[
\mathbb{E}_{X}\left\Vert X-G(Q(X))\right\Vert =\mathbb{E}_{\tilde{X}}\left\Vert \varphi^{-1}(\tilde{X})-(G\circ Q\circ\varphi^{-1})(\tilde{X})\right\Vert =0.
\]
Define $Q_{\varphi}=Q\circ\varphi^{-1}$ and $G_{\varphi}=\varphi\circ G$,
and then $Q_{\varphi}:\mathbb{R}^{r}\rightarrow\mathbb{R}^{d}$ and
$G_{\varphi}:\mathbb{R}^{d}\rightarrow\mathbb{R}^{r}$ are continuous
mappings, with
\begin{equation}
\mathbb{E}_{X}\left\Vert X-G(Q(X))\right\Vert =\mathbb{E}_{\tilde{X}}\left\Vert \varphi^{-1}(\tilde{X})-(\varphi^{-1}\circ G_{\varphi}\circ Q_{\varphi})(\tilde{X})\right\Vert =0.\label{eq:gq_assump}
\end{equation}
Let $\mathcal{D}$ be an open subset of $\mathbb{R}^{n}$ such that
$\tilde{X}$ has a positive density on $\mathcal{D}$. Then (\ref{eq:gq_assump})
indicates that $\varphi^{-1}=\varphi^{-1}\circ G_{\varphi}\circ Q_{\varphi}$
almost everywhere on $\mathcal{D}$. Since the mappings on both sides
are continuous, the identity in fact holds everywhere. Moreover, $\varphi$
is a homeomorphism, so we also have $G_{\varphi}(Q_{\varphi}(x))=x$
on $\mathcal{D}$.

However, when $d<r$, Lemma \ref{lem:invariance_domain} shows that
$Q_{\varphi}$ cannot be injective. Therefore, there exist $y,z\in\mathcal{D}$,
$y\neq z$, such that $Q_{\varphi}(y)=Q_{\varphi}(z)$. As a result,
$G_{\varphi}(Q_{\varphi}(y))=G_{\varphi}(Q_{\varphi}(z))$, which
contradicts with the previous claim that $G_{\varphi}(Q_{\varphi}(y))=y\neq z=G_{\varphi}(Q_{\varphi}(z))$.

\subsection{Proof of Theorem \ref{thm:w1bar}}

By the primal form (\ref{eq:wgan_primal}) of the 1-Wasserstein distance,
\[
W_{1}(P_{X},P_{G(Q(X))})=\inf_{\pi\in\Pi(P_{X},P_{W})}\mathbb{E}_{(X,W)\sim\pi}\left\Vert X-G(W)\right\Vert ,
\]
where $W=Q(X)$. Since $W$ is a deterministic function of $X$, we
immediately get
\begin{equation}
\mathbb{E}_{(X,W)\sim\pi}\left\Vert X-G(W)\right\Vert =\mathbb{E}_{(X,W)\sim\pi}\left\Vert X-G(Q(X))\right\Vert =\mathbb{E}_{X\sim P_{X}}\left\Vert X-G(Q(X))\right\Vert .\label{eq:w1bar_first}
\end{equation}
Moreover, by the dual form (\ref{eq:wgan_dual}) of the 1-Wasserstein
distance,
\begin{equation}
W_{1}(P_{G(Q(X))},P_{G(AZ_{0})})=\sup_{f\in\mathcal{F}}\left\{ \mathbb{E}_{X}f(G(Q(X)))-\mathbb{E}_{Z_{0}}f(G(AZ_{0}))\right\} .\label{eq:w1bar_second}
\end{equation}
Combining (\ref{eq:w1bar_first}) and (\ref{eq:w1bar_second}), we
have
\begin{align*}
 & W_{1}(P_{X},P_{G(Q(X))})+W_{1}(P_{G(Q(X))},P_{G(AZ_{0})})\\
=\  & \mathbb{E}_{X}\left\Vert X-G(Q(X))\right\Vert +\sup_{f\in\mathcal{F}}\left\{ \mathbb{E}_{X}f(G(Q(X)))-\mathbb{E}_{Z_{0}}f(G(AZ_{0}))\right\} .
\end{align*}
Then by taking the infimum of $Q$ and combining with (\ref{eq:w1distance}),
we get
\[
\overline{W}_{1}(P_{X},P_{G(AZ_{0})})=\inf_{Q\in\mathcal{Q}}\Big\{ W_{1}(P_{X},P_{G(Q(X))})+W_{1}(P_{G(Q(X))},P_{G(AZ_{0})})\Big\}.
\]

Since $W_{1}$ is a distance between probability measures, by the
triangle inequality we have $W_{1}(P_{X},P_{G(AZ_{0})})\le\overline{W}_{1}(P_{X},P_{G(AZ_{0})})$.
If there exists a $Q^{*}\in\mathcal{Q}$ such that $Q^{*}(X)$ has
the same distribution as $AZ_{0}$, then $W_{1}(P_{G(Q^{*}(X))},P_{G(AZ_{0})})=0$
and $W_{1}(P_{X},P_{G(Q^{*}(X))})=W_{1}(P_{X},P_{G(AZ_{0})})$, so
\[
\overline{W}_{1}(P_{X},P_{G(AZ_{0})})\le W_{1}(P_{X},P_{G(Q^{*}(X))})+W_{1}(P_{G(Q^{*}(X))},P_{G(AZ_{0})})=W_{1}(P_{X},P_{G(AZ_{0})}),
\]
which implies that $W_{1}(P_{X},P_{G(AZ_{0})})=\overline{W}_{1}(P_{X},P_{G(AZ_{0})})$.

\subsection{Proof of Theorem \ref{thm:generalization}}

\begin{lem}
\label{lem:max_norm}Let $Z_{0,1},\ldots,Z_{0,n}\overset{iid}{\sim}N(0,I_{d})$
and define $t_{n,d}=\sqrt{3d+2\log n+2\sqrt{d^{2}+d\log n}}$. Then
\[
\mathbb{P}\left(\max_{1\le i\le n}\Vert Z_{0,i}\Vert\le t_{n,d}\right)\ge1-e^{-d}.
\]
\end{lem}
\begin{proof}
Let $\xi_{i}=\Vert Z_{0,i}\Vert^{2}$, so $\xi_{i}\overset{iid}{\sim}\chi_{d}^{2}$.
By Lemma 1 of \citet{laurent2000adaptive}, for any $x>0$, we have
\[
\mathbb{P}(\xi_{i}>d+2\sqrt{dx}+2x)\le e^{-x}.
\]
As a result,
\[
\mathbb{P}(\xi_{1}\le d+2\sqrt{dx}+2x,\ldots,\xi_{n}\le d+2\sqrt{dx}+2x)\ge(1-e^{-x})^{n}.
\]
Bernoulli's inequality states that $(1+x)^{r}\ge1+rx$ for every integer
$r\ge1$ and real number $x\ge-1$. Therefore,
\[
\mathbb{P}(\xi_{1}\le d+2\sqrt{dx}+2x,\ldots,\xi_{n}\le d+2\sqrt{dx}+2x)\ge1-ne^{-x}=1-e^{-x+\log n}.
\]

Let $x=d+\log n$, and then
\[
t_{n,d}^{2}=3d+2\log n+2\sqrt{d^{2}+d\log n}=d+2\sqrt{dx}+2x.
\]
Therefore,
\begin{align*}
\mathbb{P}\left(\max_{1\le i\le n}\Vert Z_{0,i}\Vert\le t_{n,d}\right) & =\mathbb{P}\left(\max_{1\le i\le n}\Vert Z_{0,i}\Vert^{2}\le t_{n,d}^{2}\right)\\
 & =\mathbb{P}\left(\xi_{1}\le t_{n,d}^{2},\ldots,\xi_{n}\le t_{n,d}^{2}\right)\\
 & =\mathbb{P}\left(\xi_{1}\le d+2\sqrt{dx}+2x,\ldots,\xi_{n}\le d+2\sqrt{dx}+2x\right)\\
 & \ge1-e^{-x+\log n}=1-e^{-d}.
\end{align*}
\end{proof}
Let $\hat{\mathbb{E}}_{S_{X}}$ and $\hat{\mathbb{E}}_{S_{Z_{0}}}$
denote the empirical expectations over $n$ observations from $P_{X}$
and $N(0,I_{d})$, respectively, \emph{i.e.}, for some functions $g$
and $\tilde{g}$,
\[
\hat{\mathbb{E}}_{S_{X}}[g]=\frac{1}{n}\sum_{i=1}^{n}g(X_{i}),\quad\hat{\mathbb{E}}_{S_{Z_{0}}}[\tilde{g}]=\frac{1}{n}\sum_{i=1}^{n}\tilde{g}(Z_{0,i}).
\]
Define $\mathcal{A}=\{A_{s}:1\le s\le d\}$, and then it is easy to
find that $\Vert A\Vert=1$ for any $A\in\mathcal{A}$, where $\Vert A\Vert$
is the operator norm of $A$. For convenience, given a fixed $Q$,
let $\iota(x)=x$, $h(x)=G(Q(x))$, and $\tilde{h}(z)=G(Az)$, so
$h$ and $\tilde{h}$ implicitly depend on $G$, $Q$, and $A$. Without
loss of generality, we combine the two sets $S_{X}$ and $S_{Z_{0}}$
together, and write $S=\{(X_{1},Z_{0,1}),\ldots,(X_{n},Z_{0,n})\}$.
Then define
\[
\Psi_{1}(S)=\hat{\mathbb{E}}_{S_{X}}\Vert\iota-h\Vert,\quad\Psi_{2}(S)=\sup_{f\in\mathcal{F}}\left\{ \hat{\mathbb{E}}_{S_{X}}[f\circ h]-\hat{\mathbb{E}}_{S_{Z_{0}}}[f\circ\tilde{h}]\right\} ,\quad\Psi(S)=\Psi_{1}(S)+\Psi_{2}(S).
\]
Consider the events
\[
\mathcal{E}=\left\{ \sup_{A\in\mathcal{A},Q\in\mathcal{Q}}\left|\mathbb{E}\Psi(S)-\Psi(S)\right|\le\varepsilon\right\} ,\quad\mathcal{T}=\left\{ \max_{1\le i\le n}\Vert Z_{0,i}\Vert\le t_{n,d}\right\} ,
\]
and then we have
\[
\mathbb{P}(\mathcal{E}^{c})=\mathbb{P}(\mathcal{E}^{c}\cap\mathcal{T})+\mathbb{P}(\mathcal{E}^{c}\cap\mathcal{T}^{c})\le\mathbb{P}(\mathcal{E}^{c}\mid\mathcal{T})\mathbb{P}(\mathcal{T})+\mathbb{P}(\mathcal{T}^{c})\le\mathbb{P}(\mathcal{E}^{c}\mid\mathcal{T})+e^{-d},
\]
where the last inequality is due to Lemma \ref{lem:max_norm}.

The analysis below is conditioned on event $\mathcal{T}$, which implies
that $\Vert Z_{0,i}\Vert\le t_{n,d}$ for $i=1,\ldots,n$. Suppose
that there is another sample $S'=\{(X_{1},Z_{0,1}),\ldots,(X_{i}',Z_{0,i}')\ldots,(X_{n},Z_{0,n})\}$
that differs from $S$ by exactly one element. Then it is clear that
\begin{align*}
\left|\Psi_{1}(S)-\Psi_{1}(S')\right| & =\left|\hat{\mathbb{E}}_{S_{X}}\Vert\iota-h\Vert-\hat{\mathbb{E}}_{S_{X}'}\Vert\iota-h\Vert\right|=\left|\frac{1}{n}\Vert X_{i}-h(X_{i})\Vert-\frac{1}{n}\Vert X_{i}'-h(X_{i}')\Vert\right|\\
 & \le\frac{\Vert X_{i}-X_{i}'\Vert+\Vert h(X_{i})-h(X_{i}')\Vert}{n}\\
 & \le\frac{(1+L_{G}L_{Q})\Vert X_{i}-X_{i}'\Vert}{n}\le\frac{2(1+L_{G}L_{Q})B}{n},
\end{align*}
where the last inequality is due to the Lipschitz continuity of $G$
and $Q$. Moreover,
\begin{align*}
\left|\Psi_{2}(S)-\Psi_{2}(S')\right| & \le\sup_{f\in\mathcal{F}}\left|\hat{\mathbb{E}}_{S_{X}}[f\circ h]-\hat{\mathbb{E}}_{S_{X}'}[f\circ h]\right|+\sup_{f\in\mathcal{F}}\left|\hat{\mathbb{E}}_{S_{Z_{0}}}[f\circ\tilde{h}]-\hat{\mathbb{E}}_{S_{Z_{0}}'}[f\circ\tilde{h}]\right|\\
 & =\frac{1}{n}\sup_{f\in\mathcal{F}}\left|(f\circ h)(X_{i})-(f\circ h)(X_{i}')\right|+\frac{1}{n}\sup_{f\in\mathcal{F}}\left|(f\circ\tilde{h})(Z_{0,i})-(f\circ\tilde{h})(Z_{0,i}')\right|\\
 & \le\frac{L_{G}L_{Q}\Vert X_{i}-X_{i}'\Vert}{n}+\frac{L_{G}\Vert A\Vert\cdot\Vert Z_{0,i}-Z_{0,i}'\Vert}{n}\le\frac{2L_{G}(L_{Q}B+t_{n,d})}{n}.
\end{align*}
Combining the results together, we get
\[
\left|\Psi(S)-\Psi(S')\right|\le\frac{2(1+2L_{G}L_{Q})B+2L_{G}t_{n,d}}{n}.
\]
Applying McDiarmid's inequality, it holds that
\[
\mathbb{P}\left[\left.|\Psi(S)-\mathbb{E}\Psi(S)|\ge\frac{\varepsilon}{2}\right|\mathcal{T}\right]\le2\exp\left\{ -\frac{n\varepsilon^{2}}{8[(1+2L_{G}L_{Q})B+L_{G}t_{n,d}]^{2}}\right\} .
\]
Then by a union bound over all $\mathcal{A}$ and a set of encoders
$\mathcal{Q}_{\hat{\Theta}_{Q}}$ parameterized by $\hat{\Theta}_{Q}$,
we have
\[
\mathbb{P}\left[\left.\sup_{A\in\mathcal{A},Q\in\mathcal{Q}_{\hat{\Theta}_{Q}}}\left|\Psi(S)-\mathbb{E}\Psi(S)\right|\ge\frac{\varepsilon}{2}\right|\mathcal{T}\right]\le2d|\hat{\Theta}_{Q}|\exp\left\{ -\frac{n\varepsilon^{2}}{8[(1+2L_{G}L_{Q})B+L_{G}t_{n,d}]^{2}}\right\} .
\]

Now consider another $Q'\in\mathcal{Q}$, and we define the corresponding
notations $h'(x)=G(Q'(x))$, $\Psi_{1}'(S)=\hat{\mathbb{E}}_{S_{X}}\Vert\iota-h'\Vert$,
$\Psi_{2}'(S)=\sup_{f\in\mathcal{F}}\left\{ \hat{\mathbb{E}}_{S_{X}}[f\circ h']-\hat{\mathbb{E}}_{S_{Z_{0}}}[f\circ\tilde{h}]\right\} $,
and $\Psi'(S)=\Psi_{1}'(S)+\Psi_{2}'(S)$. Since $\hat{\Theta}_{Q}$
is an $\varepsilon/(8L_{G}L_{\theta_{Q}})$-net of the parameter space
$\Theta_{Q}$ of $\mathcal{Q}$, every point in $\Theta_{Q}$ is within
the distance $\varepsilon/(8L_{G}L_{\theta_{Q}})$ of a point in $\hat{\Theta}_{Q}$.
For any $Q'\in\mathcal{Q}$, there exists a $Q\in\mathcal{Q}_{\hat{\Theta}_{Q}}$
such that
\begin{align*}
\left|\Psi_{1}(S)-\Psi_{1}'(S)\right| & =\left|\frac{1}{n}\sum_{i=1}^{n}\Vert X_{i}-h(X_{i})\Vert-\frac{1}{n}\sum_{i=1}^{n}\Vert X_{i}-h'(X_{i})\Vert\right|\\
 & \le\frac{1}{n}\sum_{i=1}^{n}\left|\Vert X_{i}-h(X_{i})\Vert-\Vert X_{i}-h'(X_{i})\Vert\right|\\
 & \le\frac{1}{n}\sum_{i=1}^{n}\Vert h(X_{i})-h'(X_{i})\Vert\le L_{G}L_{\theta_{Q}}\cdot\frac{\varepsilon}{8L_{G}L_{\theta_{Q}}}=\frac{\varepsilon}{8},
\end{align*}
and
\begin{align*}
\left|\Psi_{2}(S)-\Psi_{2}'(S)\right| & =\left|\sup_{f\in\mathcal{F}}\left\{ \hat{\mathbb{E}}_{S_{X}}[f\circ h]-\hat{\mathbb{E}}_{S_{Z_{0}}}[f\circ\tilde{h}]\right\} -\sup_{f\in\mathcal{F}}\left\{ \hat{\mathbb{E}}_{S_{X}}[f\circ h']-\hat{\mathbb{E}}_{S_{Z_{0}}}[f\circ\tilde{h}]\right\} \right|\\
 & \le\sup_{f\in\mathcal{F}}\left|\hat{\mathbb{E}}_{S_{X}}[f\circ h]-\hat{\mathbb{E}}_{S_{X}}[f\circ h']\right|\\
 & \le\sup_{f\in\mathcal{F}}L_{G}L_{\theta_{Q}}\cdot\frac{\varepsilon}{8L_{G}L_{\theta_{Q}}}=\frac{\varepsilon}{8}.
\end{align*}
As a result, $\left|\Psi(S)-\Psi'(S)\right|\le\left|\Psi_{1}(S)-\Psi_{1}'(S)\right|+\left|\Psi_{2}(S)-\Psi_{2}'(S)\right|\le\varepsilon/4$,
which also implies that
\[
\left|\mathbb{E}\Psi(S)-\mathbb{E}\Psi'(S)\right|\le\mathbb{E}\left|\Psi(S)-\Psi'(S)\right|\le\frac{\varepsilon}{4}.
\]
Therefore, with a high probability,
\begin{align*}
 & \sup_{A\in\mathcal{A},Q'\in\mathcal{Q}}\left|\Psi'(S)-\mathbb{E}\Psi'(S)\right|\\
\le\  & \sup_{A\in\mathcal{A},Q'\in\mathcal{Q}}\left\{ \inf_{Q\in\mathcal{Q}_{\hat{\Theta}_{Q}}}\left(\left|\Psi(S)-\Psi'(S)\right|+\left|\mathbb{E}\Psi(S)-\mathbb{E}\Psi'(S)\right|\right)+\sup_{Q\in\mathcal{Q}_{\hat{\Theta}_{Q}}}\left|\Psi(S)-\mathbb{E}\Psi(S)\right|\right\} \\
\le\  & \frac{\varepsilon}{4}+\frac{\varepsilon}{4}+\frac{\varepsilon}{2}=\varepsilon.
\end{align*}

Next, we can show that
\begin{align*}
 & \sup_{A\in\mathcal{A},Q\in\mathcal{Q}}\left|\sup_{f\in\mathcal{F}}\left\{ \mathbb{E}f(h(X))-\mathbb{E}f(\tilde{h}(Z_{0}))\right\} -\mathbb{E}\Psi_{2}(S)\right|\\
\le\  & \sup_{A\in\mathcal{A},Q\in\mathcal{Q}}\mathbb{E}\left|\sup_{f\in\mathcal{F}}\left\{ \mathbb{E}f(h(X))-\mathbb{E}f(\tilde{h}(Z_{0}))\right\} -\Psi_{2}(S)\right|\\
=\  & \sup_{A\in\mathcal{A},Q\in\mathcal{Q}}\mathbb{E}\left|\sup_{f\in\mathcal{F}}\left\{ \mathbb{E}f(h(X))-\mathbb{E}f(\tilde{h}(Z_{0}))\right\} -\sup_{f\in\mathcal{F}}\left\{ \hat{\mathbb{E}}_{S_{X}}[f\circ h]-\hat{\mathbb{E}}_{S_{Z_{0}}}[f\circ\tilde{h}]\right\} \right|\\
\le\  & \sup_{A\in\mathcal{A},Q\in\mathcal{Q}}\mathbb{E}\sup_{f\in\mathcal{F}}\left|\left\{ \mathbb{E}f(h(X))-\mathbb{E}f(\tilde{h}(Z_{0}))\right\} -\left\{ \hat{\mathbb{E}}_{S_{X}}[f\circ h]-\hat{\mathbb{E}}_{S_{Z_{0}}}[f\circ\tilde{h}]\right\} \right|\\
=\  & \sup_{A\in\mathcal{A},Q\in\mathcal{Q}}\mathbb{E}\sup_{f\in\mathcal{F}}\left|\left\{ \mathbb{E}f(h(X))-\hat{\mathbb{E}}_{S_{X}}[f\circ h]\right\} +\left\{ \hat{\mathbb{E}}_{S_{Z_{0}}}[f\circ\tilde{h}]-\mathbb{E}f(\tilde{h}(Z_{0}))\right\} \right|\\
\le\  & \mathbb{E}\sup_{A\in\mathcal{A},Q\in\mathcal{Q},f\in\mathcal{F}}\left|\left\{ \mathbb{E}f(h(X))-\hat{\mathbb{E}}_{S_{X}}[f\circ h]\right\} +\left\{ \hat{\mathbb{E}}_{S_{Z_{0}}}[f\circ\tilde{h}]-\mathbb{E}f(\tilde{h}(Z_{0}))\right\} \right|\\
\le\  & \mathbb{E}\sup_{Q\in\mathcal{Q},f\in\mathcal{F}}\left|\mathbb{E}f(h(X))-\hat{\mathbb{E}}_{S_{X}}[f\circ h]\right|+\sup_{A\in\mathcal{A},f\in\mathcal{F}}\left|\hat{\mathbb{E}}_{S_{Z_{0}}}[f\circ\tilde{h}]-\mathbb{E}f(\tilde{h}(Z_{0}))\right|\\
\le\  & 2\mathfrak{R}_{n}(\mathcal{F}\circ G\circ\mathcal{Q})+2\mathfrak{R}_{n}(\mathcal{F}\circ G\circ\mathcal{A}).
\end{align*}
The last inequality is obtained by the standard technique of symmetrization
in \citet{mohri2018foundations}.

Finally, note that $\mathbb{E}\Vert X-h(X)\Vert=\mathbb{E}\Psi_{1}(S)$,
and then
\begin{align*}
 & \sup_{A\in\mathcal{A}}\left|\overline{W}_{1}(P_{X},P_{G(AZ_{0})})-\overline{W}_{1}(\hat{P}_{X},\hat{P}_{G(AZ_{0})})\right|\\
=\  & \sup_{A\in\mathcal{A}}\left|\inf_{Q\in\mathcal{Q}}\sup_{f\in\mathcal{F}}\left\{ \mathbb{E}\Vert X-h(X)\Vert+\mathbb{E}f(h(X))-\mathbb{E}f(\tilde{h}(Z_{0}))\right\} -\inf_{Q\in\mathcal{Q}}\Psi(S)\right|\\
\le\  & \sup_{A\in\mathcal{A},Q\in\mathcal{Q}}\left|\sup_{f\in\mathcal{F}}\left\{ \mathbb{E}\Vert X-h(X)\Vert+\mathbb{E}f(h(X))-\mathbb{E}f(\tilde{h}(Z_{0}))\right\} -\Psi(S)\right|\\
=\  & \sup_{A\in\mathcal{A},Q\in\mathcal{Q}}\left|\mathbb{E}\Psi_{1}(S)-\Psi_{1}(S)+\sup_{f\in\mathcal{F}}\left\{ \mathbb{E}f(h(X))-\mathbb{E}f(\tilde{h}(Z_{0}))\right\} -\Psi_{2}(S)\right|\\
\le\  & \sup_{A\in\mathcal{A},Q\in\mathcal{Q}}\left|\mathbb{E}\Psi_{1}(S)+\mathbb{E}\Psi_{2}(S)-\Psi_{1}(S)-\Psi_{2}(S)\right|\\
 & \quad+\sup_{A\in\mathcal{A},Q\in\mathcal{Q}}\left|\sup_{f\in\mathcal{F}}\left\{ \mathbb{E}f(h(X))-\mathbb{E}f(\tilde{h}(Z_{0}))\right\} -\mathbb{E}\Psi_{2}(S)\right|\\
=\  & \sup_{A\in\mathcal{A},Q\in\mathcal{Q}}\left|\Psi(S)-\mathbb{E}\Psi(S)\right|+\sup_{A\in\mathcal{A},Q\in\mathcal{Q}}\left|\sup_{f\in\mathcal{F}}\left\{ \mathbb{E}f(h(X))-\mathbb{E}f(\tilde{h}(Z_{0}))\right\} -\mathbb{E}\Psi_{2}(S)\right|.
\end{align*}
We have shown that the first term is smaller than or equal to $\varepsilon$
with a high probability, and the second term is bounded by $2\mathfrak{R}_{n}(\mathcal{F}\circ G\circ\mathcal{Q})+2\mathfrak{R}_{n}(\mathcal{F}\circ G\circ\mathcal{A})$.
Then the stated result holds.

\subsection{Proof of Theorem \ref{thm:hausdorff_consistency}}

The proof is mostly adapted from \citet{meitz2021statistical}. By
Assumption \ref{assu:deriv} and the mean value theorem, we have for
any fixed $(x,z)$,
\[
|L(x,z;\theta)-L(x,z;\theta')|\le m(x,z)\cdot\Vert\theta-\theta'\Vert,\quad m(x,z)\coloneqq\sup_{\theta\in\Theta}\left\Vert \frac{\partial L(x,z;\theta)}{\partial\theta}\right\Vert 
\]
holds for all $\theta,\theta'\in\Theta$. Assumption \ref{assu:deriv}
also assumes that $\mathbb{E}_{X\otimes Z_{0}}[m(X,AZ_{0})]^{2}<\infty$,
and then Theorem 19.5 and Example 19.7 of \citet{vaart_1998} imply
that $n^{1/2}(\hat{\ell}_{n}(\theta,A)-\ell(\theta,A))\overset{d}{\rightarrow}\mathbb{G}$
for some tight limit process $\mathbb{G}$ in $\ell^{\infty}(\Theta)$.
Since $\Theta$ is compact, we have
\begin{equation}
\sup_{\theta\in\Theta}n^{1/2}\left|\hat{\ell}_{n}(\theta,A)-\ell(\theta,A)\right|=O_{P}(1).\label{eq:ln_converge}
\end{equation}

Recall that $\phi_{A}(\theta_{G},\theta_{Q})=\sup_{\theta_{f}}\ell(\theta,A)$
and $\hat{\phi}_{A}(\theta_{G},\theta_{Q})=\sup_{\theta_{f}}\hat{\ell}_{n}(\theta,A)$.
For convenience, define
\begin{align*}
V(A) & =\inf_{G\in\mathcal{G}}\overline{W}_{1}(P_{X},P_{G(AZ_{0})})=\inf_{\substack{\theta_{Q}\in\Theta_{Q}\\
\theta_{G}\in\Theta_{G}
}
}\phi_{A}(\theta_{Q},\theta_{G}),\\
\hat{V}_{n}(A) & =\inf_{G\in\mathcal{G}}\overline{W}_{1}(\hat{P}_{X},\hat{P}_{G(AZ_{0})})=\inf_{\substack{\theta_{Q}\in\Theta_{Q}\\
\theta_{G}\in\Theta_{G}
}
}\hat{\phi}_{A}(\theta_{Q},\theta_{G}).
\end{align*}
Also introduce the functions $\Delta_{A}(\theta)$ and $\hat{\Delta}_{n,A}(\theta)$
as follows:
\begin{align*}
\Delta_{A}(\theta) & =\max\left\{ \phi_{A}(\theta_{G},\theta_{Q})-\ell(\theta,A),\phi_{A}(\theta_{G},\theta_{Q})-V(A)\right\} ,\\
\hat{\Delta}_{n,A}(\theta) & =\max\left\{ \hat{\phi}_{A}(\theta_{G},\theta_{Q})-\hat{\ell}_{n}(\theta,A),\hat{\phi}_{A}(\theta_{G},\theta_{Q})-\hat{V}_{n}(A)\right\} .
\end{align*}
The function $\Delta_{A}(\theta)$ is non-negative for all $\theta\in\Theta$,
and $\theta^{*}\in\Theta_{A}^{*}$ if and only if $\Delta_{A}(\theta^{*})=0$,
implying that
\[
\Theta_{A}^{*}=\left\{ \theta\in\Theta:\Delta_{A}(\theta)=0\right\} .
\]
Similarly, we have
\[
\hat{\Theta}_{n,A}^{*}(\tau_{n})=\left\{ \theta\in\Theta:\hat{\Delta}_{n,A}(\theta)\le\tau_{n}\right\} .
\]
First note that
\begin{align}
 & \left|\hat{V}_{n}(A)-V(A)\right|=\left|\inf_{\theta_{Q},\theta_{G}}\hat{\phi}_{A}(\theta_{Q},\theta_{G})-\inf_{\theta_{Q},\theta_{G}}\phi_{A}(\theta_{Q},\theta_{G})\right|\nonumber \\
\le & \sup_{\theta_{Q},\theta_{G}}\left|\hat{\phi}_{A}(\theta_{Q},\theta_{G})-\phi_{A}(\theta_{Q},\theta_{G},A)\right|=\sup_{\theta_{Q},\theta_{G}}\left|\sup_{\theta_{f}}\hat{\ell}_{n}(\theta,A)-\sup_{\theta_{f}}\ell(\theta,A)\right|\nonumber \\
\le & \sup_{\theta_{Q},\theta_{G}}\sup_{\theta_{f}}\left|\hat{\ell}_{n}(\theta,A)-\ell(\theta,A)\right|=\sup_{\theta}\left|\hat{\ell}_{n}(\theta,A)-\ell(\theta,A)\right|,\label{eq:va_difference}
\end{align}
and then
\begin{align}
 & \left|\hat{\Delta}_{n,A}(\theta)-\Delta_{A}(\theta)\right|\nonumber \\
= & \left|\hat{\phi}_{A}(\theta_{G},\theta_{Q})-\phi_{A}(\theta_{G},\theta_{Q})-\min\left\{ \hat{\ell}_{n}(\theta,A),\hat{V}_{n}(A)\right\} +\min\left\{ \ell(\theta,A),V(A)\right\} \right|\nonumber \\
\le & \left|\hat{\phi}_{A}(\theta_{G},\theta_{Q})-\phi_{A}(\theta_{G},\theta_{Q})\right|+\left|\min\left\{ \hat{\ell}_{n}(\theta,A),\hat{V}_{n}(A)\right\} -\min\left\{ \ell(\theta,A),V(A)\right\} \right|\nonumber \\
\le & \left|\hat{\phi}_{A}(\theta_{G},\theta_{Q})-\phi_{A}(\theta_{G},\theta_{Q})\right|+\max\left\{ \left|\hat{\ell}_{n}(\theta,A)-\ell(\theta,A)\right|,\left|\hat{V}_{n}(A)-V(A)\right|\right\} \nonumber \\
= & \left|\sup_{\theta_{f}}\hat{\ell}_{n}(\theta,A)-\sup_{\theta_{f}}\ell(\theta,A)\right|+\max\left\{ \left|\hat{\ell}_{n}(\theta,A)-\ell(\theta,A)\right|,\left|\hat{V}_{n}(A)-V(A)\right|\right\} \nonumber \\
\le & \sup_{\theta_{f}}\left|\hat{\ell}_{n}(\theta,A)-\ell(\theta,A)\right|+\sup_{\theta}\left|\hat{\ell}_{n}(\theta,A)-\ell(\theta,A)\right|\le2\sup_{\theta}\left|\hat{\ell}_{n}(\theta,A)-\ell(\theta,A)\right|.\label{eq:delta_difference}
\end{align}
Combined with (\ref{eq:ln_converge}), it holds that $\sup_{\theta\in\Theta}\left|\hat{\Delta}_{n,A}(\theta)-\Delta_{A}(\theta)\right|\overset{P}{\rightarrow}0$.

Since $\ell(\theta,A)$ is continuous in the compact set $\Theta$,
by Berge's maximum theorem, the function $\phi_{A}(\theta_{G},\theta_{Q})=\sup_{\theta_{f}}\ell(\theta,A)$
is continuous in $(\theta_{G},\theta_{Q})$, and further we have that
$\Delta_{A}(\theta)$ is continuous in $\theta$. By the continuity
of $\Delta_{A}(\theta)$ and the definition of $\Theta_{A}^{*}$,
we have that for any $\varepsilon>0$, there exists an $\eta(\varepsilon)>0$
such that
\[
\inf_{\theta\in\Theta\backslash\Theta_{A,\varepsilon}^{*}}\Delta_{A}(\theta)\ge\eta(\varepsilon),
\]
where $\Theta_{A,\varepsilon}^{*}=\{\theta\in\Theta:d(\theta,\Theta_{A}^{*})\le\varepsilon\}$
denotes the $\varepsilon$-net of the set $\Theta_{A}^{*}$.

Now we are ready to show that $\sup_{\theta\in\hat{\Theta}_{n,A}^{*}(\tau_{n})}d(\theta,\Theta_{A}^{*})\overset{P}{\rightarrow}0$.
Let small $\varepsilon_{p},\varepsilon_{d}>0$ be arbitrary, choose
an $\eta=\eta(\varepsilon_{d})$ such that $\inf_{\theta\in\Theta\backslash\Theta_{A,\varepsilon_{d}}^{*}}\Delta_{A}(\theta)\ge\eta$
holds, and choose $n_{\varepsilon_{p}}$ such that for all $n\ge n_{\varepsilon_{p}}$,
both $\sup_{\theta\in\Theta}\left|\hat{\Delta}_{n,A}(\theta)-\Delta_{A}(\theta)\right|\le\eta/4$
and $\tau_{n}\le\eta/4$ hold with probability larger than $1-\varepsilon_{p}$.
Then
\begin{align*}
\sup_{\theta\in\hat{\Theta}_{n,A}^{*}(\tau_{n})}\Delta_{A}(\theta) & \le\sup_{\theta\in\hat{\Theta}_{n,A}^{*}(\tau_{n})}\left|\hat{\Delta}_{n,A}(\theta)-\Delta_{A}(\theta)\right|+\sup_{\theta\in\hat{\Theta}_{n,A}^{*}(\tau_{n})}\hat{\Delta}_{n,A}(\theta)\\
 & \le\sup_{\theta\in\Theta}\left|\hat{\Delta}_{n,A}(\theta)-\Delta_{A}(\theta)\right|+\tau_{n}\le\eta/2<\inf_{\theta\in\Theta\backslash\Theta_{A,\varepsilon_{d}}^{*}}\Delta_{A}(\theta),
\end{align*}
which implies that $\hat{\Theta}_{n,A}^{*}(\tau_{n})\cap(\Theta\backslash\Theta_{A,\varepsilon_{d}}^{*})=\varnothing$,
and hence $\hat{\Theta}_{n,A}^{*}(\tau_{n})\subset\Theta_{A,\varepsilon_{d}}^{*}$,
and
\[
\sup_{\theta\in\hat{\Theta}_{n,A}^{*}(\tau_{n})}d(\theta,\Theta_{A}^{*})\le\sup_{\theta\in\Theta_{A,\varepsilon_{d}}^{*}}d(\theta,\Theta_{A}^{*})\le\varepsilon_{d}.
\]
This holds for all $n\ge n_{\varepsilon_{p}}$ with probability larger
than $1-\varepsilon_{p}$. Since $\varepsilon_{p}$ is chosen arbitrarily,
we have $\sup_{\theta\in\hat{\Theta}_{n,A}^{*}(\tau_{n})}d(\theta,\Theta_{A}^{*})\overset{P}{\rightarrow}0$.

Finally, we are going to prove that $\sup_{\theta\in\Theta_{A}^{*}}d(\theta,\hat{\Theta}_{n,A}^{*}(\tau_{n}))\overset{P}{\rightarrow}0$.
First, (\ref{eq:ln_converge}) and (\ref{eq:delta_difference}) show
that
\[
\sup_{\theta\in\Theta}\left|\hat{\Delta}_{n,A}(\theta)-\Delta_{A}(\theta)\right|=O_{P}(n^{-1/2}).
\]
Then by definition, $\sup_{\theta\in\Theta_{A}^{*}}\Delta_{A}(\theta)=0$,
so
\begin{align*}
\sup_{\theta\in\Theta_{A}^{*}}\hat{\Delta}_{n,A}(\theta) & \le\sup_{\theta\in\Theta_{A}^{*}}\left|\hat{\Delta}_{n,A}(\theta)-\Delta_{A}(\theta)\right|+\sup_{\theta\in\Theta_{A}^{*}}\Delta_{A}(\theta)\\
 & \le\sup_{\theta\in\Theta}\left|\hat{\Delta}_{n,A}(\theta)-\Delta_{A}(\theta)\right|=O_{P}(n^{-1/2}).
\end{align*}
By assumption, $n^{-1/2}/\tau_{n}\overset{P}{\rightarrow}0$, so for
any $\varepsilon_{p}>0$, there exists an $n_{\varepsilon_{p}}$ such
that for all $n\ge n_{\varepsilon_{p}}$,
\begin{equation}
\sup_{\theta\in\Theta_{A}^{*}}\hat{\Delta}_{n,A}(\theta)=O_{P}(n^{-1/2})=O_{P}(n^{-1/2}/\tau_{n})\cdot\tau_{n}\le\tau_{n}\label{eq:smaller_than_tau_n}
\end{equation}
holds with probability larger than $1-\varepsilon_{p}$. Under the
event (\ref{eq:smaller_than_tau_n}), we have $\Theta_{A}^{*}\subset\hat{\Theta}_{n,A}^{*}(\tau_{n})$,
and hence $\sup_{\theta\in\Theta_{A}^{*}}d(\theta,\hat{\Theta}_{n,A}^{*}(\tau_{n}))=0$.
Since $\varepsilon_{p}$ is arbitrary, we have $\sup_{\theta\in\Theta_{A}^{*}}d(\theta,\hat{\Theta}_{n,A}^{*}(\tau_{n}))\overset{P}{\rightarrow}0$.

Combining both $\sup_{\theta\in\hat{\Theta}_{n,A}^{*}(\tau_{n})}d(\theta,\Theta_{A}^{*})\overset{P}{\rightarrow}0$
and $\sup_{\theta\in\Theta_{A}^{*}}d(\theta,\hat{\Theta}_{n,A}^{*}(\tau_{n}))\overset{P}{\rightarrow}0$,
we immediately obtain $d_{H}(\hat{\Theta}_{n,A}^{*}(\tau_{n}),\Theta_{A}^{*})\overset{P}{\rightarrow}0$.

\subsection{Proof of Theorem \ref{thm:rank_consistency}}

In (\ref{eq:va_difference}) we have shown that $\left|\hat{V}_{n}(A)-V(A)\right|\le\sup_{\theta\in\Theta}\left|\hat{\ell}_{n}(\theta,A)-\ell(\theta,A)\right|$.
Combined with (\ref{eq:ln_converge}), it holds that $\hat{V}_{n}(A)=V(A)+O_{P}(n^{-1/2})$.

By Assumption \ref{assu:theta}(a), there exists $(G^{*},Q^{*},f^{*})\in(\mathcal{G}\times\mathcal{Q}\times\mathcal{F})\cap\mathcal{S}_{A_{r}}$
such that
\begin{equation}
\mathfrak{L}_{A_{r}}(G^{*},Q^{*},f^{*})=\inf_{\substack{Q\in{\cal Q}^{\diamond}\\
G\in\mathcal{G}^{\diamond}
}
}\sup\limits_{f\in{\cal F}^{\diamond}}\mathfrak{L}_{A_{r}}(G,Q,f),\label{eq:optim_val}
\end{equation}
and Theorem \ref{thm:existence_qg} indicates that the right hand
side of (\ref{eq:optim_val}) is in fact zero. Therefore,
\begin{align*}
V(A_{r}) & =\inf_{G\in\mathcal{G}}\overline{W}_{1}(P_{X},P_{G(A_{r}Z_{0})})=\inf_{\substack{Q\in{\cal Q}\\
G\in\mathcal{G}
}
}\sup\limits_{f\in{\cal F}}\mathfrak{L}_{A_{r}}(G,Q,f)\le\sup\limits_{f\in{\cal F}}\mathfrak{L}_{A_{r}}(G^{*},Q^{*},f)\\
 & \le_{(i)}\sup\limits_{f\in{\cal F}^{\diamond}}\mathfrak{L}_{A_{r}}(G^{*},Q^{*},f)=_{(ii)}\mathfrak{L}_{A_{r}}(G^{*},Q^{*},f^{*})=0,
\end{align*}
where\emph{ $(i)$} is due to the fact that $\mathcal{F}\subset\mathcal{F}^{\diamond}$,
and $(ii)$ is by the definition of $\mathcal{S}_{A}$.

For $s<r$, by Assumption \ref{assu:theta}(b), there exists $(G_{s}^{*},Q_{s}^{*},f_{s}^{*})\in\mathcal{S}_{A_{s}}$
such that $f_{s}^{*}\in\mathcal{F}$ and
\begin{align*}
V(A_{s}) & =\inf_{\substack{Q\in{\cal Q}\\
G\in\mathcal{G}
}
}\sup\limits_{f\in{\cal F}}\mathfrak{L}_{A_{s}}(G,Q,f)=\sup_{f\in\mathcal{F}}\mathfrak{L}_{A_{s}}(G_{s}^{*},Q_{s}^{*},f)\\
 & \ge\mathfrak{L}_{A_{s}}(G_{s}^{*},Q_{s}^{*},f_{s}^{*})=\mathfrak{F}_{A}(G_{s}^{*},Q_{s}^{*})=\inf_{Q\in{\cal Q}^{\diamond}}\mathfrak{F}_{A}(G_{s}^{*},Q)=\inf\limits_{Q\in{\cal Q}^{\diamond}}\sup_{f\in{\cal F}^{\diamond}}\mathfrak{L}_{A}(G_{s}^{*},Q,f)\\
 & =\overline{W}_{1}(P_{X},P_{G_{s}^{*}(A_{s}Z_{0})})\ge W_{1}(P_{X},P_{G_{s}^{*}(A_{s}Z_{0})}).
\end{align*}
Now we prove that $W_{1}(P_{X},P_{G_{s}^{*}(A_{s}Z_{0})})>0$ for
any $s<r$ by contradiction. Suppose that $W_{1}(P_{X},P_{G_{s}^{*}(A_{s}Z_{0})})=0$.
Then by definition, $G_{s}^{*}(A_{s}Z_{0})$ must be supported on
$\mathcal{X}$, and $\varphi(X)$ and $\varphi(G_{s}^{*}(A_{s}Z_{0}))$
are identically distributed. Using the same argument in the proof
of Theorem \ref{thm:existence_qg}, we can show that there exists
a homeomorphism $Q:\mathbb{R}^{r}\rightarrow\mathbb{R}^{r}$ such
that $W\coloneqq Q(\varphi(X))\sim N(0,I_{r})$. Let
\[
B_{1}=\begin{pmatrix}I_{s}\\
\mathbf{0}_{(d-s)\times s}
\end{pmatrix}\in\mathbb{R}^{d\times s},\quad B_{2}=\begin{pmatrix}I_{s} & \mathbf{0}_{r-s}\end{pmatrix}\in\mathbb{R}^{s\times r},
\]
and then it is easy to find that $A_{s}Z_{0}\overset{d}{=}B_{1}B_{2}W$,
and hence
\[
W\overset{d}{=}h(W),\quad h(z)=Q(\varphi(G_{s}^{*}(B_{1}B_{2}z))),
\]
which implies that $h(W)=W$ almost surely, and the function $h:\mathbb{R}^{r}\rightarrow\mathbb{R}^{r}$
satisfies $h(z)=z$ almost everywhere. Since $h$ is continuous, we
have that in fact $h(z)=z$ holds everywhere. Now let $h_{1}(x)=Q(\varphi(G_{s}^{*}(B_{1}x)))$,
$h_{2}(x)=B_{2}x$, and then $h_{1}:\mathbb{R}^{s}\rightarrow\mathbb{R}^{r}$
and $h_{2}:\mathbb{R}^{r}\rightarrow\mathbb{R}^{s}$ are both continuous
mappings. Since $s<r$, Lemma \ref{lem:invariance_domain} shows that
$h_{2}$ cannot be injective. Therefore, there exist $y,z$, $y\neq z$,
such that $h_{2}(y)=h_{2}(z)$. As a result, $h_{1}(h_{2}(y))=h_{1}(h_{2}(z))$,
which contradicts with the previous claim that $h_{1}(h_{2}(y))=h(y)=y\neq z=h(z)=h_{1}(h_{2}(z))$.

Therefore, for some $c>0$ we have $V(A_{s})\ge c>0$ for $s<r$,
$V(A_{s})\ge0$ for $s>r$, and $V(A_{r})=0$. Let $\hat{\varrho}_{n}(s)=\min_{\theta_{G}}\hat{\rho}_{n}(\theta_{G},A_{s})=\hat{V}_{n}(A_{s})+\lambda_{n}s$.
It has been shown that $\hat{V}_{n}(A)=V(A)+O_{P}(n^{-1/2})$, so
for $s<r$,
\begin{align*}
\mathbb{P}\left(\hat{\varrho}_{n}(r)\ge\hat{\varrho}_{n}(s)\right) & \le\mathbb{P}\left(\hat{\varrho}_{n}(r)>\frac{c}{2}\right)+\mathbb{P}\left(\hat{\varrho}_{n}(s)\le\hat{\varrho}(r)\le\frac{c}{2}\right)\\
 & \le\mathbb{P}\left(\hat{\varrho}_{n}(r)>\frac{c}{2}\right)+\mathbb{P}\left(\hat{\varrho}_{n}(s)\le\frac{c}{2}\right)\\
 & =\mathbb{P}\left(\hat{V}_{n}(A_{r})+\lambda_{n}\cdot r>\frac{c}{2}\right)+\mathbb{P}\left(\hat{V}_{n}(A_{s})+\lambda_{n}\cdot s\le\frac{c}{2}\right)\rightarrow0.
\end{align*}
The first term in the last equation goes to zero because $\hat{V}_{n}(A_{r})+\lambda_{n}\cdot r=V(A_{r})+O_{P}(n^{-1/2})+\lambda_{n}\cdot r=O_{P}(n^{-1/2})+\lambda_{n}\cdot r$
and $\lambda_{n}\rightarrow0$. The second term also goes to zero,
since
\[
\hat{V}_{n}(A_{s})+\lambda_{n}\cdot s=V(A_{s})+O_{P}(n^{-1/2})+\lambda_{n}\cdot s\overset{P}{\rightarrow}V(A_{s})\ge c.
\]

For $s>r$, if $V(A_{s})=0$, then $\hat{\varrho}_{n}(r)/\hat{\varrho}_{n}(s)\overset{P}{\rightarrow}r/s<1$,
and if $V(A_{s})>0$, then $\hat{\varrho}_{n}(r)/\hat{\varrho}_{n}(s)\overset{P}{\rightarrow}0$.
In both cases, we have
\[
\mathbb{P}\left(\hat{\varrho}_{n}(r)<\hat{\varrho}_{n}(s)\right)\rightarrow1,\quad s>r.
\]
Overall, we have
\[
\mathbb{P}\left(\hat{r}\neq r\right)\le\mathbb{P}\left(\bigcup_{s\neq\hat{r}}\left\{ \hat{\varrho}_{n}(r)\ge\hat{\varrho}_{n}(s)\right\} \right)\le\sum_{s\neq r}\mathbb{P}\left(\hat{\varrho}_{n}(r)\ge\hat{\varrho}_{n}(s)\right)\rightarrow0.
\]

\section{Additional Experiment Details}

\subsection{Neural network architectures}

In this section, we present the neural network architectures for each
experiment. In what follows, $\mathrm{CONCAT}(x;v)$ means concatenating
vectors $x$ and $v$, $e_{s}\in\mathbb{R}^{d}$ is the $s$-th unit
vector, $\mathrm{FC}_{o}$ is the fully-connected layer with $o$
output units, $\mathrm{Conv}_{o,k,s,p}$ is the convolutional layer
with $o$ output channels, kernel size $k$, stride $s$, and padding
$p$, $\mathrm{ConvTrans}_{o,k,s,p,q}$ is the transposed convolutional
layer with $o$ output channels, kernel size $k$, stride $s$, padding
$p$, and output padding $q$, $\mathrm{InstanceNorm}$ is the instance
normalization layer, $\mathrm{ReLU}(x)=\max\{x,0\}$, $\mathrm{LeakyReLU}(x;\alpha)=\max\{x,0\}+\alpha\cdot\min\{x,0\}$,
$\mathrm{Sigmoid}(x)=1/(1+e^{-x})$, $\mathrm{Tanh}(x)=(e^{2x}-1)/(e^{2x}+1)$,
and $\mathrm{SiLU}(x)=x/(1+e^{-x})$. Detailed implementations can
be found in the code available at \url{https://github.com/yixuan/LWGAN}.

\paragraph*{Toy examples}

For Swiss roll, S-curve, and Hyperplane datasets, the latent dimension
$d$ is set to 5, 5, and 10, respectively.
\begin{itemize}
\item Encoder architecture:
\begin{align*}
x\in\mathbb{R}^{p} & \rightarrow\mathrm{CONCAT}(e_{s})\rightarrow\mathrm{FC}_{512}\rightarrow\mathrm{ReLU}\\
 & \rightarrow\mathrm{FC}_{256}\rightarrow\mathrm{ReLU}\rightarrow\mathrm{FC}_{128}\rightarrow\mathrm{ReLU}\\
 & \rightarrow\mathrm{FC}_{64}\rightarrow\mathrm{ReLU}\rightarrow\mathrm{FC}_{32}\rightarrow\mathrm{ReLU}\rightarrow\mathrm{FC}_{d}
\end{align*}
\item Generator architecture:
\begin{align*}
z\in\mathbb{R}^{d} & \rightarrow\mathrm{FC}_{64}\rightarrow\mathrm{SiLU}\rightarrow\mathrm{FC}_{64}\rightarrow\mathrm{SiLU}\rightarrow\mathrm{FC}_{64}\rightarrow\mathrm{SiLU}\rightarrow\mathrm{FC}_{p}
\end{align*}
\item Critic architecture:
\begin{align*}
x\in\mathbb{R}^{p} & \rightarrow\mathrm{CONCAT}(e_{s})\rightarrow\mathrm{FC}_{64}\rightarrow\mathrm{ReLU}\\
 & \rightarrow\mathrm{FC}_{64}\rightarrow\mathrm{ReLU}\rightarrow\mathrm{FC}_{64}\rightarrow\mathrm{ReLU}\rightarrow\mathrm{FC}_{1}
\end{align*}
\end{itemize}

\paragraph*{MNIST}

The latent dimension is $d=16$ for digits 1 and 2, and $d=20$ for
all digits.
\begin{itemize}
\item Encoder architecture:
\begin{align*}
x\in\mathbb{R}^{28\times28} & \rightarrow\mathrm{Conv}_{64,5,2,2}\rightarrow\mathrm{LeakyReLU}(0.1)\\
 & \rightarrow\mathrm{Conv}_{128,5,2,2}\rightarrow\mathrm{LeakyReLU}(0.1)\\
 & \rightarrow\mathrm{Conv}_{256,5,2,2}\rightarrow\mathrm{LeakyReLU}(0.1)\\
 & \rightarrow\mathrm{CONCAT}(e_{s})\rightarrow\mathrm{FC}_{2d}\rightarrow\mathrm{LeakyReLU}(0.1)\\
 & \rightarrow\mathrm{CONCAT}(e_{s})\rightarrow\mathrm{FC}_{d}
\end{align*}
\item Generator architecture:
\begin{align*}
z\in\mathbb{R}^{d} & \rightarrow\mathrm{FC}_{4d}\rightarrow\mathrm{LeakyReLU}(0.1)\rightarrow\mathrm{FC}_{4096}\rightarrow\mathrm{LeakyReLU}(0.1)\\
 & \rightarrow\mathrm{ConvTrans}_{128,5,1,0,0}\rightarrow\mathrm{LeakyReLU}(0.1)\\
 & \rightarrow\mathrm{ConvTrans}_{64,5,1,0,0}\rightarrow\mathrm{LeakyReLU}(0.1)\\
 & \rightarrow\mathrm{ConvTrans}_{1,8,2,0,0}\rightarrow\mathrm{Sigmoid}
\end{align*}
\item Critic architecture:
\begin{align*}
x\in\mathbb{R}^{28\times28} & \rightarrow\mathrm{Conv}_{64,5,2,2}\rightarrow\mathrm{LeakyReLU}(0.1)\rightarrow\mathrm{Conv}_{128,5,2,2}\rightarrow\mathrm{LeakyReLU}(0.1)\\
 & \rightarrow\mathrm{Conv}_{256,5,2,2}\rightarrow\mathrm{LeakyReLU}(0.1)\rightarrow\mathrm{CONCAT}(e_{s})\rightarrow\mathrm{FC}_{2d}\\
 & \rightarrow\mathrm{LeakyReLU}(0.1)\rightarrow\mathrm{CONCAT}(e_{s})\rightarrow\mathrm{FC}_{1}
\end{align*}
\end{itemize}

\paragraph*{CelebA}

For CelebA, the latent dimension is $d=128$.
\begin{itemize}
\item Encoder architecture:
\begin{align*}
x\in\mathbb{R}^{64\times64\times3} & \rightarrow\mathrm{Conv}_{64,5,2,2}\rightarrow\mathrm{ReLU}\rightarrow\mathrm{Conv}_{128,5,2,2}\rightarrow\mathrm{InstanceNorm}\rightarrow\mathrm{ReLU}\\
 & \rightarrow\mathrm{Conv}_{256,5,2,2}\rightarrow\mathrm{InstanceNorm}\rightarrow\mathrm{ReLU}\\
 & \rightarrow\mathrm{Conv}_{512,5,2,2}\rightarrow\mathrm{InstanceNorm}\rightarrow\mathrm{ReLU}\\
 & \rightarrow\mathrm{CONCAT}(e_{s})\rightarrow\mathrm{FC}_{2d}\rightarrow\mathrm{ReLU}\rightarrow\mathrm{CONCAT}(e_{s})\rightarrow\mathrm{FC}_{d}
\end{align*}
\item Generator architecture:
\begin{align*}
z\in\mathbb{R}^{d} & \rightarrow\mathrm{FC}_{4d}\rightarrow\mathrm{ReLU}\rightarrow\mathrm{FC}_{8192}\rightarrow\mathrm{ReLU}\\
 & \rightarrow\mathrm{ConvTrans}_{256,5,2,2,1}\rightarrow\mathrm{ReLU}\\
 & \rightarrow\mathrm{ConvTrans}_{128,5,2,2,1}\rightarrow\mathrm{ReLU}\\
 & \rightarrow\mathrm{ConvTrans}_{64,5,2,2,1}\rightarrow\mathrm{ReLU}\\
 & \rightarrow\mathrm{ConvTrans}_{3,5,2,2,1}\rightarrow\mathrm{Tanh}
\end{align*}
\item Critic architecture:
\begin{align*}
x\in\mathbb{R}^{64\times64\times3} & \rightarrow\mathrm{Conv}_{64,5,2,2}\rightarrow\mathrm{ReLU}\\
 & \rightarrow\mathrm{Conv}_{128,5,2,2}\rightarrow\mathrm{InstanceNorm}\rightarrow\mathrm{ReLU}\\
 & \rightarrow\mathrm{Conv}_{256,5,2,2}\rightarrow\mathrm{InstanceNorm}\rightarrow\mathrm{ReLU}\\
 & \rightarrow\mathrm{Conv}_{512,5,2,2}\rightarrow\mathrm{InstanceNorm}\rightarrow\mathrm{ReLU}\\
 & \rightarrow\mathrm{Conv}_{2d,4,1,0}\rightarrow\mathrm{CONCAT}(e_{s})\rightarrow\mathrm{FC}_{d}\rightarrow\mathrm{ReLU}\\
 & \rightarrow\mathrm{CONCAT}(e_{s})\rightarrow\mathrm{FC}_{1}
\end{align*}
\end{itemize}

\subsection{Comparison metrics}

Proposed by \citet{salimans2016improved}, the inception score (IS)
uses a pre-trained Inception-v3 model to predict the class probabilities
for each generated image. These predictions are then summarized into
IS by the KL divergence as follows:
\[
\mathrm{IS}=\exp\left(\mathbb{E}_{x\sim P_{G(Z^{*})}}D_{KL}\left(p(y|x)\Vert p(y)\right)\right),
\]
where $p(y|x)$ is the predicted probabilities conditioning on the
generated images, and $p(y)$ is the corresponding marginal distribution.
Higher scores of IS are better, corresponding to a larger KL divergence
between the two distributions.

The Fr\'echet inception distances (FID) is proposed by \citet{heusel2017gans}
to improve IS by directly comparing the statistics of generated samples
to real samples. It is defined as the Fr\'echet distance between
two multivariate normal distributions:
\[
\mathrm{FID}=\Vert\mu_{r}-\mu_{G}\Vert^{2}+\text{Tr}\left(\Sigma_{r}+\Sigma_{G}-2(\Sigma_{r}\Sigma_{G})^{1/2}\right),
\]
where $X_{r}\sim N(\mu_{r},\Sigma_{r})$ and $X_{G}\sim N(\mu_{G},\Sigma_{G})$
are the 2048-dimensional activations of the Inception-v3 pool-3 layer
for real and generated samples, respectively. For FID, lower is better.

The reconstruction error is defined as
\[
\mathrm{RE}=\dfrac{1}{m}\sum_{i=1}^{m}\Vert\hat{X}_{i}-X_{i}\Vert_{2},
\]
where $\hat{X}_{i}$ is the reconstructed sample for $X_{i}$. The
reconstruction error is used to evaluate whether the model generates
meaningful latent codes and has the capacity to recover the original
information. Smaller reconstruction errors indicate a more meaningful
latent space that can be decoded into the original samples.

\subsection{Monitoring the training process}

During the training process, we have saved various metrics to monitor
the state of the model. Figure \ref{fig:losses} shows three types
of losses during the training of LWGAN on the CelebA data. The pre-$GQ$
critic loss stands for the term $f(G(Q(X)))-f(G(A_{s}Z_{0}))$ before
updating $G$ and $Q$ in each outer iteration, and post-$GQ$ critic
loss is the same quantity but after updating $G$ and $Q$. The reconstruction
error stands for the term $\left\Vert X-G(Q(X))\right\Vert $. The
various spikes in the reconstruction error plot are the results of
randomly picking one rank $s$ in each iteration by the design of
Algorithm 1. In cases that $s$ is small, the reconstruction error
would be large as explained by Corollary 1. However, we can find that
the lower bound of the reconstruction error steadily decreases, indicating
that for ranks larger than the intrinsic dimension, the reconstruction
quality is indeed improving. From Figure \ref{fig:losses} we can
also find that the critic losses quickly become stable after the first
few thousands of iterations, implying that our proposed computational
method in Section 3.3 is both stable and efficient.

\begin{figure}[h]
\begin{centering}
\includegraphics[width=0.99\textwidth]{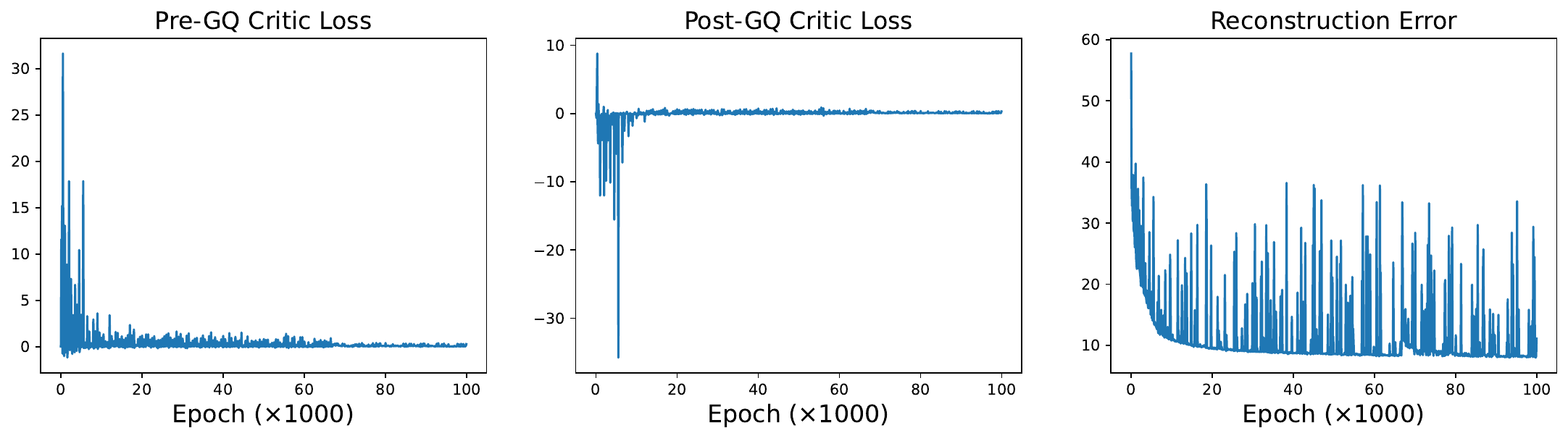}
\par\end{centering}
\caption{\protect\label{fig:losses}Monitoring the loss function values of
LWGAN during the training on the CelebA data.}
\end{figure}

\subsection{Uncertainty quantification for the estimated intrinsic dimensions}

For the toy examples in Section 5.1, we have conducted a bootstrap-type
experiment to quantify the uncertainty of the estimated intrinsic
dimensions. The experiment steps are as follows:
\begin{enumerate}
\item Given the dataset, train an LWGAN model with final neural network
parameters $\hat{\theta}$. Let $\hat{G}$ and $\hat{r}$ be the estimated
generator and intrinsic dimension, respectively.
\item Simulate new data points $\hat{X}_{i}=\hat{G}(A_{\hat{r}}Z_{0,i})$,
$i=1,\ldots,B$, where $Z_{0,i}\overset{iid}{\sim}N(0,I_{d})$.
\item Train a new LWGAN model on $\hat{X}_{1},\ldots,\hat{X}_{B}$, possibly
using $\hat{\theta}$ to warm start the training procedure. Let $\hat{r}_{1}^{\mathrm{boot}}$
be the estimated intrinsic dimension on this simulated dataset.
\item Repeat steps 2 and 3 for 100 rounds, and summarize the distribution
of $\hat{r}_{1}^{\mathrm{boot}},\ldots,\hat{r}_{100}^{\mathrm{boot}}$.
\end{enumerate}
Ideally, the distribution of the bootstrap estimates $\hat{r}_{1}^{\mathrm{boot}},\ldots,\hat{r}_{100}^{\mathrm{boot}}$
should be concentrated around the estimated intrinsic dimension $\hat{r}$.
Table \ref{tab:bootstrap} demonstrates the results on the three simulated
datasets, from which we can find that the bootstrap distribution is
indeed consistent with the estimates.

\begin{table}[h]
\caption{\protect\label{tab:bootstrap}Bootstrap distribution of the estimated
intrinsic dimensions for the toy examples.}

\bigskip{}

\centering{}%
\begin{tabular}{cccc}
\toprule 
Dataset & True $r$ & Estimated $\hat{r}$ & Bootstrap Distribution\tabularnewline
\midrule
\midrule 
Swiss roll & 1 & 1 & $\hat{r}_{i}^{\mathrm{boot}}=\begin{cases}
1, & 90\%\\
2, & 10\%
\end{cases}$\tabularnewline
\midrule 
S-curve & 2 & 2 & $\hat{r}_{i}^{\mathrm{boot}}=\begin{cases}
2, & 99\%\\
3, & 1\%
\end{cases}$\tabularnewline
\midrule 
Hyperplane & 4 & 4 & $\hat{r}_{i}^{\mathrm{boot}}=\begin{cases}
4, & 99\%\\
5, & 1\%
\end{cases}$\tabularnewline
\bottomrule
\end{tabular}
\end{table}

\end{document}